\PassOptionsToPackage{dvipsnames}{xcolor}
\documentclass{article}

\usepackage[main, final, nonatbib]{neurips_2026}

\usepackage[utf8]{inputenc} 
\usepackage[T1]{fontenc}    
\usepackage[textsize=tiny]{todonotes}
\usepackage{hyperref}       
\usepackage{url}            
\usepackage{booktabs}       
\usepackage{amsfonts}       
\usepackage{nicefrac}       
\usepackage{microtype}      
\usepackage{amsmath}
\usepackage{amssymb}
\usepackage{mathtools}
\usepackage{amsthm}
\usepackage{xspace}
\usepackage{graphicx}
\usepackage{pifont}
\usepackage{bbm}
\usepackage{multirow}
\usepackage{makecell}
\usepackage{enumitem}
\usepackage{wrapfig}
\usepackage{algpseudocode}
\usepackage{algorithm}
\usepackage[square,numbers]{natbib}

\newcommand{\bluetext}[1]{\textcolor{NavyBlue}{#1}}
\definecolor{darkpurple}{HTML}{660066}
\definecolor{figred}{HTML}{B85450}
\definecolor{figgreen}{HTML}{82B366}
\definecolor{figgray}{HTML}{808080}

\renewcommand{\d}{\mathop{}\!d}
\newcommand{\indep}{\perp \!\!\! \perp}

\newcommand*\circledgreen[1]{%
\tikz[baseline=(char.base)]{
  \node[shape=circle, draw=ForestGreen!60, fill=ForestGreen!10, thick, inner sep=1pt] (char) {\scriptsize\textsf{#1}};
}}

\newcommand{\method}{SurvB-learner\xspace}
\usepackage[capitalize,noabbrev]{cleveref}
\algrenewcommand\algorithmicforall{\textbf{for all}}
\makeatletter
\renewcommand{\ALG@step}{\addtocounter{ALG@line}{1}\arabic{ALG@line}}
\makeatother
\algnewcommand{\Input}{\item[\textbf{Input:}]}

\theoremstyle{plain}
\newtheorem{theorem}{Theorem}[section]
\newtheorem{proposition}[theorem]{Proposition}
\newtheorem{lemma}[theorem]{Lemma}
\newtheorem{corollary}[theorem]{Corollary}
\theoremstyle{definition}
\newtheorem{definition}[theorem]{Definition}
\newtheorem{assumption}[theorem]{Assumption}
\theoremstyle{remark}

\definecolor{darkpurple}{HTML}{660066}
\definecolor{figred}{HTML}{B85450}
\definecolor{figgreen}{HTML}{82B366}
\definecolor{figgray}{HTML}{808080}
\definecolor{darkgreen}{rgb}{0.0, 0.5, 0.0}
\definecolor{lightyellow}{HTML}{FFE699}
\definecolor{red_revision}{HTML}{FF0000}
\definecolor{darkblue}{HTML}{2E6EB3}
\definecolor{derkgreen}{HTML}{3E7D00}
\definecolor{darkyellow}{HTML}{D99542}
\definecolor{darkpurple}{HTML}{660066}

\usepackage{amsmath,amsfonts,bm}

\def\eqref#1{equation~\ref{#1}}

\def\1{\bm{1}}

\def\rmE{{\mathbf{E}}}

\DeclareMathAlphabet{\mathsfit}{\encodingdefault}{\sfdefault}{m}{sl}
\SetMathAlphabet{\mathsfit}{bold}{\encodingdefault}{\sfdefault}{bx}{n}

\def\gA{{\mathcal{A}}}

\def\gC{{\mathcal{C}}}
\def\gD{{\mathcal{D}}}

\def\gH{{\mathcal{H}}}

\def\gM{{\mathcal{M}}}

\def\gQ{{\mathcal{Q}}}

\def\gS{{\mathcal{S}}}
\def\gT{{\mathcal{T}}}

\def\gX{{\mathcal{X}}}

\def\sN{{\mathbb{N}}}

\def\sP{{\mathbb{P}}}

\def\sR{{\mathbb{R}}}

\newcommand{\E}{\mathbb{E}}

\title{Assessing the robustness of heterogeneous treatment effects in survival analysis under informative censoring}

\author{%
  \textbf{Yuxin Wang}\thanks{Correspondence to: \texttt{Yuxin.Wang1@lmu.de}}, \, \textbf{Dennis Frauen}, \textbf{Jonas Schweisthal}, \textbf{Maresa Schröder},
  \textbf{Stefan Feuerriegel}\\
  LMU Munich \\
  Munich Center of Machine Learning (MCML)
}

\begin{document}

\maketitle

\begin{abstract}
    Dropout is common in clinical studies, with up to half of patients leaving early due to side effects or other reasons. When dropout is informative (i.e., dependent on survival time), it introduces censoring bias, because of which treatment effect estimates are also biased. In this paper, we propose an assumption-lean framework to assess the robustness of conditional average treatment effect (CATE) estimates in survival analysis when facing censoring bias. Unlike existing works that rely on strong assumptions, such as non-informative censoring, to obtain point estimation, we use partial identification to derive informative bounds on the CATE. Thereby, our framework helps to identify patient subgroups where treatment is effective despite informative censoring. We further propose a novel model-agnostic meta-learner, called \textbf{SurvB-learner}, to estimate the bounds that can be used in combination with arbitrary machine-learning models, and that has favorable theoretical properties such as double-robustness and quasi-oracle efficiency. We finally demonstrate the effectiveness of our meta-learner across various experiments using both simulated and real-world data.
\end{abstract}

\section{Introduction}

Dropout is widespread in survival analysis, especially in oncology. A recent systematic review found that nearly all cancer studies, both prospective and retrospective, reported dropout, with rates as high as 53\% in colorectal cancer and 43\% in non–small cell lung cancer (NSCLC) \citep{shand.2024}. Patients drop out for many reasons, including severe side effects, personal circumstances, or loss to follow-up \citep{fizazi.2017}. However, due to dropout, survival analyses can be biased~\citep{gupta.2025, templeton.2020}, making therapies appear more or less effective than they truly are. Here, we present a framework to assess the robustness of heterogeneous treatment effects in survival analysis under dropout.

When dropout is related to survival time (known as \emph{informative censoring}), it introduces systematic bias into survival analysis because event times such as disease-free or overall survival are censored in a non-random way~\citep{klein.2003, vanderlaan.2003}. We refer to this problem as \textit{censoring bias}. Such bias directly undermines the reliable estimation of treatment effects and is especially problematic when studying heterogeneous treatment effects, which aim to identify patient subgroups that respond differently to treatment. Because dropout mechanisms often vary across subgroups, for example, patients with advanced disease may leave due to rapid progression, while others may withdraw due to toxicity~\citep{fizazi.2017, hui.2013, perez-cruz.2018}, censoring bias can affect subgroup-specific estimates even more severely than population-wide effects.

In this paper, we propose an assumption-lean framework to assess the robustness of CATE estimates in survival analysis under informative censoring (see overview figure in Fig.~\ref{fig:method_overview} in Appendix~\ref{app:methods_overview}). When censoring depends on survival time, unbiased point estimation of the CATE is generally impossible. Rather than assuming censoring is independent of the survival time, we explicitly account for informative censoring and therefore reframe the task using \emph{partial identification}, which shifts the goal from estimating a single point value to constructing informative lower and upper bounds on the CATE. This allows clinicians to assess whether treatment benefits remain positive for certain subgroups even in the presence of censoring. To achieve this, we build upon the idea of worst-case survival bounds~\citep{basu.1982, slud.1983, zheng.1995}, incorporating censoring strength and domain knowledge (e.g., such as expected survival time after dropout from the trial) via sensitivity functions to construct informative bounds.

We further introduce a novel two-stage meta-learner called \textbf{\method} to efficiently estimate these bounds. Our \method is carefully designed to correct for the plug-in bias of na\"ive learners. Our \method is flexible (i.e., model-agnostic) and can be used together with arbitrary machine learning models. Further, it has favorable theoretical properties such as double robustness and quasi-oracle efficiency. Double robustness means that our estimator remains consistent even if one of the nuisance models (e.g., outcome or censoring) is misspecified, while quasi-oracle efficiency guarantees asymptotic performance close to that of an estimator with knowledge of the true nuisance models. Finally, beyond discrete treatments, we extend our framework to continuous treatments (Appendix~\ref{app:continuous_setting}) and further show that it is applicable to both randomized data from clinical trials as well as observational data from real-world studies with hidden confounding (Appendix~\ref{app:hidden_confounding}). 

In this paper, we make three \textbf{contributions:}\footnote{Code available at: \url{https://github.com/Yuxin217/survival_bias}.} 
(1) We propose an assumption-lean framework to audit censoring bias in the CATE estimates from a censored dataset. Our method replaces the non-informative censoring assumption with sensitivity functions that use censoring strength and domain knowledge (e.g., expected survival after dropout) to form informative bounds.
(2) We further introduce a model-agnostic meta-learner called \textbf{SurvB-learner} to efficiently estimate bounds.
(3) We provide theoretical results for our meta-learners by showing consistency, double robustness, and quasi-oracle efficiency properties. Finally, we confirm the effectiveness of our meta-learners by performing various experiments using both synthetic and real-world data.

\section{Related work}
\label{sec:related_work}

We review the main related literature streams on causal inference with survival outcomes. We refer to Appendix~\ref{app:extended_related_work} for details.

\textbf{CATE estimation under non-informative censoring assumption:} Standard methods for survival analysis (e.g., the Kaplan-Meier estimator~\citep{kaplan.1958}, the Cox proportional hazards model~\citep{breslow.1975, cox.1972}, and the accelerated failure time model~\citep{cox.1972, wei.1992}) rely on the assumption of \textit{non-informative censoring}, meaning that dropout time is independent of survival time. Existing methods for estimating heterogeneous treatment effects from survival data (e.g., including Cox-based models \citep{gao.2022, bo.2025}, tree-based models \citep[e.g.,][]{cui.2023, henderson.2020, tabib.2020, zhang.2017}, neural network-based methods \citep[e.g.,][]{katzman.2018, schrod.2022}. and orthogonal meta-learners tailored to censored outcomes~\citep{curth.2021, frauen.2025, xu.2024, xu.2022}) all rely on the assumption of non-informative censoring. $\Rightarrow$ \emph{When the non-informative censoring assumption is violated, estimates obtained from the previous methods will be biased.}

\textbf{Survival methods for informative censoring:} Several methods in survival analysis address the problem of \emph{dependent} censoring. Early works~\citep{basu.1982, slud.1983, zheng.1995} introduced worst-case bounds, which estimate survival functions by assuming that each censored individual experiences the event immediately after their censoring time. While these studies focused on bounding the survival function itself, more recent work has shifted to partially-identify covariate effects on survival time under dependent censoring~\citep{sakaguchi.2024, willems.2025}. $\Rightarrow$ \emph{These studies remain strictly within the survival analysis framework, without any causal inference task.}

\textbf{Causal survival methods for informative censoring:} While the above works address dependent censoring purely within a survival analysis framework, a parallel line of research tackles informative censoring in causal inference settings. Several methods have been proposed to address censoring bias when estimating the average treatment effect (ATE) from censored data \citep{bai.2025, mao.2018, rubinstein.2025, schaubel.2011, voinot.2025a}. In contrast, clinicians and researchers currently lack practical tools for assessing the robustness of heterogeneous treatment effect estimates when informative censoring is present. $\Rightarrow$ \emph{None of these are directly applicable to CATE estimation.}

\textbf{Research gap:} So far, existing heterogeneous treatment effect learners fail to account for informative censoring. To the best of our knowledge, we are thus the first to provide a partial identification tool that includes the censoring bias to ensure a robust CATE estimation.

\section{Problem setup}
\label{sec:problem_setup}

\textbf{Data:} We consider the standard setting for estimating CATEs based on censored, time-to-event data\footnote{We deal with the problem of right censoring, which is common in survival analysis settings. We thus assume that events have not happened before time $t=0$.}~\citep{cui.2023, curth.2021b, frauen.2025, vanderlaan.2003, zhang.2017}. That is, we consider the population $(X, A, T, C) \sim \sP$, where $X \in \gX \subseteq \sR^p$ are observed covariates, $A \in \gA$ is the treatment\footnote{Our method is applicable to both discrete and continuous treatment settings. For simplicity, we focus on discrete treatments (i.e., $A \in \gA \subseteq \mathbb{N}$) in the main paper. We provide an extension to continuous treatments in Appendix~\ref{app:continuous_setting}.}, $T \in \gT \subseteq [0, t_{\mathrm{max}}] \subseteq \mathbb{R^+}$ is the event time of interest (e.g., the time of overall survival (OS) or the time with disease-free survival (DFS)), $C \in \gT$ is the censoring time (e.g., the time at which a patient drops out of the study). Importantly, $t_{\mathrm{max}}$ is typically known in medical applications: it is naturally provided by the theoretical maximum human lifespan in general survival analysis, while, in oncology, the maximum survival time is often limited by disease biology. For instance, patients with certain forms of advanced lung cancer rarely survive beyond a few months \citep{imai.2015}.

\begin{wrapfigure}[12]{r}{0.40\textwidth}
    \centering
    \vspace{-0.8cm}
    \includegraphics[width=\linewidth]{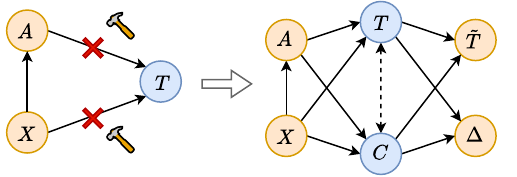}
    \vspace{-0.4cm}
    \caption{\textbf{Causal graph for survival data under informative censoring}. Variables in \textcolor{orange}{yellow} are observed, while variables in \bluetext{blue} are unobserved. Solid edges denote causal relationships, and dashed edges reflect dependence. Note that the absence of arrows encodes causal assumptions, not their presence.}
    \label{fig:causal_graph}
\end{wrapfigure}

The challenge with censored datasets is typically that $T$ is not fully observed, but only $\tilde{T} = \min \{ T, C \}$~\citep{breslow.1975, cox.1972, kaplan.1958}. Thus, one can not observe data from the full population $(X, A, T, C)$. Instead, we only observe a dataset $\gD = \left\{ \left( x_i, a_i, \tilde{t}_i, \delta_i \right)\right\}_{i=1}^n$ of size $n \in \sN$ sampled i.i.d. from the population $\mathbb{Z} = (X, A, \tilde{T}, \Delta )$, where $\Delta = \mathbbm{1}(C \leq T)$ is a censoring indicator for the event and the observed time is given by $\tilde{T} = \min \{ T, C \}$. The corresponding causal graph is shown in Fig.~\ref{fig:causal_graph}.

\textbf{Causal estimand:} We make use of the potential outcome framework \citep{rubin.1974} to formalize our causal inference task. Let $T(a) \in \gT$ denote the potential event time under treatment intervention $A = a$. We are interested in the CATE on survival time, i.e.,
\begin{equation}
\tau_{a_1, a_2}(x) = \tau_{a_1}(x) - \tau_{a_2}(x) = \mathbb{E}_{\mathbb{P}}\left[ T(a_1) - T(a_2) \mid X = x \right] .
\end{equation}
Hence, our estimand $\tau_{a_1, a_2}(x)$ measures the difference of conditional means between two treatments for a patient with covariates $X = x$. For completeness, we define the conditional average potential outcome (CAPO) for survival time via $\tau_{a}(x) = \mathbb{E}\left[ T(a) \mid X = x \right]$.

\textbf{Key definitions:} We define the \emph{propensity score} for treatment $a$ as $\pi_a(x) = \sP(A = a \mid X = x)$, for $a \in \gA$, which captures the treatment assignment mechanism. Further, we denote the \emph{censoring strength} as $\xi(x, a) = \sP (\Delta = 1 \mid X = x, A = a)$, which gives the probability that the event is censored, meaning that a patient does not complete the study based on covariates $X = x$ and treatment $A = a$. We further denote the \emph{expected survival time function} by $\mu(x, a) = \E [T \mid X = x, A = a]$ and the \emph{expected conditional survival time function} by $\nu(\delta, x, a) = \E [\tilde{T} \mid \Delta = \delta, X = x, A = a]$. Here, $\mu(x, a)$ is the expected survival time given covariates $X = x$ and treatment $A = a$, while $\nu(\delta, x, a)$ is the expected survival time further conditioned on the censoring indicator $\Delta = \delta$. Finally, we define the \emph{post-dropout survival function} $\gamma(x, a)$, which denotes the expected remaining survival time after dropout, given covariates $X = x$ and treatment $A = a$. Note that the post-dropout survival function serves as a sensitivity function in our method; we treat it as \emph{known} rather than estimating it from the observed data.

\textbf{Identifiability:} We make use of the following standard Assumption~\ref{ass:causal_survival}.

\begin{assumption} \label{ass:causal_survival}
For all $x\in \gX$, $a\in \gA$, it holds:
(i)~Consistency: $A = a \Rightarrow \tilde{T} = \min\{ T, C \} = \min \{ T(a), C(a) \}$;
(ii)~Treatment overlap: $0<\pi_a(x)<1$, $\forall x \in \mathcal{X}$, and $\forall a \in \mathcal{A}$;
(iii)~Unconfoundedness: $T(a) \indep A \,\mid\ X$;
(iv) Censoring overlap: $0\leq \xi(x, a) < 1$, $\forall x \in \mathcal{X}$, and $\forall a \in \mathcal{A}$.
\end{assumption}

Assumptions (i)--(iii) are standard in the causal inference literature and are widely used for estimating CATEs \citep{ candes.2023, curth.2021b, imbens.2004, rubin.1974, shalit.2017}. (i)~Consistency ensures that an individual's observed outcome under treatment $a$ equals their potential outcome $T(a)$; it also implies no interference between individuals. (ii)~Treatment overlap means that every patient has a positive probability of receiving each treatment, which thereby ensures sufficient support in the data. (iii)~Unconfoundedness assumes that there are no unobserved confounders, which implies that, in the uncensored subgroup, the expected conditional survival time function $\E[T(a) \mid X=x]$ is identifiable. In the following, we focus on clinical trials where unconfoundedness is ensured by design and real-world data without hidden confounding. We later also provide an extension to real-world data with hidden confounding in Appendix~\ref{app:hidden_confounding}. (iv)~Censoring overlap is additionally common in survival analysis~\citep{cai.2019, westling.2024} and ensures a positive probability of being uncensored (i.e., experiencing the event before censoring) every covariate value.

\subsection{Challenges from informative censoring}

Estimating heterogeneous treatment effects in survival analysis is challenging because standard assumptions required for valid causal inference are often violated. In particular, most of the causal survival literature~\citep{curth.2021b, frauen.2025, gao.2022,vanderlaan.2003, xu.2024, xu.2022} assumes non-informative censoring. While this allows for identification of the CATE, it rarely holds in medical practice. In such cases, censoring carries prognostic information, which, if ignored, leads to biased estimates. Below, we formalize the challenges from informative censoring and thereby motivate the need for our framework later. 

Why are the standard assumptions from the survival literature~\citep{ curth.2021b, frauen.2025, vanderlaan.2003} violated? These works typically make the assumption of non-informative censoring to identify the CATE $\tau_{a_1, a_2}(x)$ from censored data, which can be stated as 
\begin{equation}
\label{eq:non_informative_censoring}
T \indep C \mid X=x, A=a, \qquad \forall x\in \gX, a\in \gA
\end{equation}
This assumption requires that the censoring time is independent of the patient's survival time, conditional on covariates and treatment. Under this assumption, the expected conditional survival time function among the censored patients, i.e., $\mathbb{E} [ T \mid X = x, A = a, \Delta = 1]$, is identifiable. However, in practice, this assumption is often violated. For example, patients with a more severe disease (and thus shorter survival times $T$) may drop out earlier, for example, when transferred to palliative care or other facilities. Similarly, patients may be more likely to be lost to follow-up as death approaches, for instance, by withdrawing from the clinical trial \citep{hernan.2004, templeton.2020}. In these cases, the expected conditional potential survival function under censoring, $\mathbb{E} [ T(a)\mid X = x, \Delta = 1]$, is \emph{not} identified. Therefore, the causal estimand, the CATE, is also \emph{not} identified as stated in the following lemma.

\begin{lemma}\label{lem:total_prob_causal}
Let the \emph{expected conditional survival function} $\nu(\delta, x, a)$ and the \emph{censoring strength} $\xi(x, a)$ be defined as above. Under informative censoring, the true survival time is unobserved, and, therefore, the CATE is not identified; i.e.,
{\small{
\begin{align}
\label{eq:total_prob_causal}
\tau_{a_1, a_2}(x)
&= \nu(0, x, a_1)[1-\xi(x, a_1)] + \bluetext{\mathbb{E} [ T \mid X = x, A = a_1, \Delta = 1}]\xi(x, a_1) \nonumber\\
&- \nu(0, x, a_2)[1-\xi(x, a_2)] - \bluetext{\mathbb{E} [ T \mid X = x, A = a_2,  \Delta = 1]}\xi(x, a_2).
\end{align}
}}
\end{lemma}
\vspace{-0.5cm}
\begin{proof}
See Appendix~\ref{app:proof_lemma}.
\end{proof}
\vspace{-0.4cm}

Because these outcomes for the equations in \bluetext{blue color} are never observed by definition, they cannot be recovered from the data without further assumptions. Hence, even with an infinitely large sample, the true values of these expectations, and therefore of the CATE, cannot be determined from the observed data distribution alone. As a result, unbiased point estimation of the CATE is impossible under informative censoring. 

Therefore, we relax the assumption of non-informative censoring from Eq.~\ref{eq:non_informative_censoring} and adopt the approach of \emph{partial identification}: instead of targeting a single point estimate, we aim to construct informative lower and upper bounds on the CATE. To achieve this, we build on the idea of worst-case survival bounds~\citep{basu.1982, slud.1983, zheng.1995}, but where we incorporate both censoring strength and domain knowledge (e.g., expected survival time after dropout) via sensitivity functions. This approach is beneficial in practice because we explicitly account for censoring bias while still providing clinically meaningful insights. For instance, the resulting bounds allow clinicians to determine whether the effect of treatment is positive for certain subgroups even in the presence of censoring, which implies a \emph{robust benefit} from treatment.  

\vspace{-0.2cm}
\section{Our approach: partial identification of the CATE under informative censoring}
\label{sec:propose_bounds}
\vspace{-0.2cm}

To address the above challenges, we now move away from point estimation to partial identification of the CATE, which allows us to obtain informative bounds in the presence of informative censoring. To do so, we first notice that, as stated in Lemma~\ref{lem:total_prob_causal}, the key challenge for estimating the CATE lies in the expected conditional potential survival function under censoring, i.e., $\bluetext{\mathbb{E} [ T(a) \mid X = x, \Delta = 1]}$. As a remedy, our approach is therefore to first derive partial identification bounds for this quantity, which then yield bounds for CAPO $\tau_a(x)$. We can then naturally obtain bounds for CATE $\tau_{a_1, a_2}(x)$, as the CATE is defined as the difference between CAPOs, i.e., $\tau_{a_1}(x)$ and $\tau_{a_2}(x)$. Our bounds, which we derive below, build on Manski bounds~\citep{manski.1990}, but we carefully adapt to our survival setting. In that sense, our bounds follow the idea of worst-case bounds~\citep{basu.1982, slud.1983, zheng.1995} in survival analysis. Importantly, unlike standard Manski bounds~\citep{manski.1990}, which assume the only source of missing data is treatment assignment, our survival setting introduces an additional source of partial observability through informative censoring, which makes partial identification in our setting non-trivial. 

Next, we formally define the lower and upper bounds for the CAPO, which we denote by $\mu^-(x, a)$ and $\mu^+(x, a)$, respectively. The bounds thus capture the range of plausible values under our partial identification framework while allowing for informative censoring.

\textbf{Lower bound}: To construct a lower bound for the CAPO, we leverage the definition of $\tilde{T} = \min \{C, T\}$. By design, we have $T \geq\tilde{T}$ and $\bluetext{\mathbb{E} [ T(a) \mid X = x, \Delta = 1]} \geq \mathbb{E} [ \tilde{T} \mid X = x, A=a, \Delta = 1]$ when conditioned on the censored events (i.e., $\Delta = 1$). We then replace $T$ with $\tilde{T}$ in Eq.~(\ref{eq:total_prob_causal}) and obtain a lower bound  $\mu^{-}(x, a)$ as follows:
\begin{align}
\label{eq:lower_bound}
    \mu(x,a) &\geq \nu(0, x, a)[1-\xi(x, a)] 
     + \mathbb{E} [ \tilde{T} \mid \Delta = 1, X = x, A=a]\xi(x, a) \nonumber\\
     &= \mathbb{E} [ \tilde{T} \mid X = x, A=a] :=  \mu^{-}(x, a).
\end{align}

\textbf{Upper bound}: To construct an upper bound for the CAPO, we introduce $\gamma(x, a)$ as a sensitivity function, stemming from domain knowledge on the maximum survival time after dropout to construct informative bounds. Unlike many other applications of partial identification, our sensitivity function has a natural interpretation: it simply captures the maximum possible expected survival time a patient may have after censoring (we discuss how to specify this in clinical practice below).

Formally, we consider sensitivity function $\gamma(x, a)$ which satisfies
\begin{align}
\label{eq:range_of_gamma}
&\E[T - \tilde{T}\mid X = x, A=a, \Delta = 1] \leq \gamma(x, a) \leq t_{\mathrm{max}} - \underbrace{\E [\tilde{T} \mid X = x, A=a, \Delta = 1]}_{:= \nu(1, x, a)},
\end{align}
\vspace{-0.4cm}

for all $x \in \mathcal{X}$ and $a \in \mathcal{A}$. The left-hand side $\E[T - \tilde{T}\mid X = x, A=a, \Delta = 1]$ reflects the expected remaining survival time beyond the observed follow-up, while the right-hand site $ t_{\mathrm{max}} - \nu(1, x, a)$ reflects the difference between the theoretical maximum lifespan and the observed follow-up. Depending on the setting, $\gamma(x, a)$ can incorporate domain knowledge (when specific information about post-dropout survival is available) or be set to $t_\mathrm{max} - \nu(1, x, a)$ in the absence of such information. Importantly, the latter ensures that we can still obtain informative bounds even when domain knowledge is not available.

This setting is realistic in many clinical and oncology applications, where prior information about post-dropout survival can be obtained from historical trials~\citep{pecci.2025}, disease staging systems~\citep{detterbeck.2017}, and established clinical benchmarks~\citep{yin.2022}, such as 3-year DFS or 5-year OS. For example, in advanced lung cancer or pancreatic cancer, survival beyond a few months or years is exceedingly rare~\citep{hu.2023,inoue.2025,pecci.2025}. In this sense, $\gamma(x,a)$ serves as a sensitivity function: a user-specified input that encodes clinically reasonable assumptions about post-dropout survival and allows researchers to assess how conclusions change under those assumptions. Importantly, $\gamma(x,a)$ may vary across individuals with different covariates, such as baseline age, lifestyle factors, or comorbidities, thereby allowing patient-specific heterogeneity to be incorporated.

Depending on the availability of domain knowledge, we consider two ways of specifying upper bounds on the CAPOs. \textbf{Case}~\circledgreen{1}: when clinical knowledge about post-dropout survival is available for different covariates and treatments, one can set $\gamma(x, a)$ to a clinically plausible maximum post-dropout survival time and use it to construct an upper bound on the CAPOs. \textbf{Case}~\circledgreen{2}: when such domain knowledge is unavailable, one can instead set $\gamma(x,a)=t_{\mathrm{max}}-\nu(1,x,a)$ to obtain a conservative upper bound on the CAPOs. This choice makes no strong clinical assumptions about post-dropout survival while ensuring that the resulting bounds remain informative.

\noindent\textbf{Upper bounds for post-dropout survival.} We consider two cases:

$\bullet$ \textbf{Case} \circledgreen{1} \textbf{(domain-knowledge upper bounds):} In the first case, when domain knowledge about $\gamma(x, a)$ is available, this information can be directly used as a sensitivity function to construct informative upper bounds for the CAPO. For example, in some lung cancer trials~\citep{imai.2015}, patients typically survive no longer than two to five months after dropping out, which suggests setting $\gamma(x,a)=3$. The resulting upper bound is
\begin{align}
\label{eq:upper_gamma_bound}
\mu(x,a) 
&\leq \nu(0,x,a)[1-\xi(x,a)] + \nu(1,x,a)\xi(x,a) + \gamma(x,a)\xi(x,a) := \mu^{+}(x,a),
\end{align}
which is expressed as a weighted combination of the expected conditional survival time without censoring $\nu(0,x,a)$, the expected conditional survival time with censoring $\nu(1,x,a)$, and the post-dropout survival captured by $\gamma(x,a)$.

$\bullet$ \textbf{Case} \circledgreen{2} \textbf{(conservative upper bounds):} 
In the second case, domain knowledge about post-dropout survival is unavailable—for example, due to relocation, withdrawal of consent, or complete loss to follow-up. In such situations, a conservative sensitivity function is used. Specifically, we take the upper bound of $\gamma(x,a)$ and set
\vspace{-0.4cm}
\begin{align}
\gamma(x,a) = t_{\mathrm{max}} - \nu(1,x,a)
\end{align}
\vspace{-0.6cm}

where $t_{\mathrm{max}}$ is the maximum support of the distribution of $\tilde{T}$. Substituting this into Eq.~(\ref{eq:total_prob_causal}) yields the conservative upper bound
\vspace{-0.2cm}
\begin{align}
\label{eq:upper_non_informative_bound}
\mu(x,a)
&\leq \nu(0,x,a)[1-\xi(x,a)] + t_{\mathrm{max}}\xi(x,a) := \mu^{+}(x,a).
\end{align}
\vspace{-0.6cm}

\section{Property of the partial identification bounds for CATE}

We now present our main result: our derivation of the partial identification bounds for CATE $\tau_{a_1, a_2}^{-}(x)$ and $\tau_{a_1, a_2}^{+}(x)$, which characterize the range of the CATE under informative censoring.

\begin{theorem}\label{thm:bounds}
Under assumptions~\ref{ass:causal_survival}, the CATE is bounded via 
\begin{equation}
    \tau_{a_1, a_2}^{-}(x) \leq \tau_{a_1, a_2}(x) \leq \tau_{a_1, a_2}^{+}(x),
\end{equation}
where $\tau_{a_1, a_2}^{+}(x) = \mu^{+}(x, a_1) - \mu^{-}(x, a_2)$ and $\tau_{a_1, a_2}^{-}(x) = \mu^{-}(x, a_1) - \mu^{+}(x, a_2)$. Here, $\mu^{-}(x, a_1)$ and $\mu^{-}(x, a_2)$ are given by Eq.~(\ref{eq:lower_bound}), and $\mu^{+}(x, a_1)$ and $\mu^{+}(x, a_2)$ are given by Eq.~(\ref{eq:upper_gamma_bound}) or Eq.~(\ref{eq:upper_non_informative_bound}) depending on the choice of $\gamma(x, a)$. 
\end{theorem}

The width of the bounds is an important property, as tighter bounds convey more information. Since the bounds on the CATE are derived from those on the CAPO, we study the theoretical properties of the CAPO bounds. In the following proposition, we show that the tightness of the bounds is determined by the censoring strength in the data and the post-dropout survival time function.

\begin{proposition}
\label{prop:width_of_bounds}
Let $\mu^{-}(x, a)$ and $\mu^{+}(x, a)$ be the lower bound and upper bound of the CAPO. The width of the bound for the CAPO is then 
\begin{equation}
\mu^{+}(x, a) - \mu^{-}(x, a) = \gamma(x, a)\xi(x, a),
\end{equation}
where $\gamma(x, a)$ is the known sensitivity function, and where the censoring strength $\xi(x, a)$ is the probability of patients being censored given covariates and treatments. The bounds are informative under either of the following conditions: (a) when strong domain knowledge removes uncertainty about post-dropout survival; or (b)~when $\xi(x, a) = 0$, which corresponds to the absence of censoring.
\end{proposition}

\vspace{-0.4cm}
\begin{proof}
See Appendix~\ref{app:proof_width}.
\end{proof}
\vspace{-0.4cm}

The above proposition characterizes when our partial identification bounds become informative. Hence, our bounds are particularly well-suited for settings with low or moderate censoring: when censoring is low, the width of bounds remains tight even if $\gamma(x, a)$ is misspecified; when censoring is high, but domain knowledge is available to construct $\gamma(x, a)$, the bounds remain informative. Hence, this ensures that our framework provides valid and interpretable results under informative censoring, thereby helping clinicians identify patient subgroups that reliably benefit from treatment without relying on potentially biased survival models.

\section{\method: A meta-learner for estimating the bounds}
\vspace{-0.2cm}

We now develop a two-stage meta-learner, called \method, for estimating bounds from Theorem~\ref{thm:bounds}. For simplicity, we first derive the two-stage meta-learner for the CAPOs; the corresponding bounds for the CATE can be obtained directly by taking the difference between two CAPOs (see Appendix~\ref{app:survb-learner_cate} for details). Importantly, our two-stage meta-learner is flexible and can be instantiated with arbitrary machine learning methods (including, e.g., linear models or neural networks).

\textbf{Why the simple ``plug-in'' learner is problematic:} A straightforward approach to estimating the bounds is the \emph{plug-in learner}. In this approach, one first estimates the nuisance functions $\hat{\nu}(\delta, x, a)$, $\hat{\xi}(x, a)$, and $\hat{\pi}_a(x)$, which correspond to ${\nu}(\delta, x, a)$, ${\xi}(x, a)$, and $\pi_a(x)$ for all $\delta \in \{0, 1\}$, $a \in \mathcal{A}$, $x \in \mathcal{X}$. Then, one can directly ``plug in'' the estimated nuisance functions into Eq.~(\ref{eq:lower_bound}), Eq.~(\ref{eq:upper_gamma_bound}), and Eq.~(\ref{eq:upper_non_informative_bound}) to obtain the bounds. However, while straightforward, this approach is known to suffer from plug-in bias \citep{kennedy.2023a}, which limits statistical efficiency and suboptimal estimation performance. 

To address the drawbacks of the plug-in learner, we propose a \emph{two-stage meta-learner}, which, as we show later, is doubly robust and quasi-oracle efficient. To achieve this, we borrow the general principles of two-stage learners~\citep{kunzel.2019, nie.2020}: first, to construct pseudo-outcomes that are, in expectation, equal to the target estimand; and, second, to use the pseudo-outcomes in a second-stage regression to estimate the target quantity, which, in our case, are $\mu^{\pm}(x, a)$. We now adapt it to our setting. 

\vspace{-0.2cm}
\subsection{Overview of the \method}
\vspace{-0.2cm}

At a high level, our \method proceeds in two stages (see the pseudocode in Algorithm~\ref{alg:pseudocode}). (1)~We estimate the nuisance functions with cross-fitting, where any suitable machine learning model can be applied, and construct pseudo-outcomes $\hat{\phi}^{\pm}(x, a)$. (2)~We regress the estimated pseudo-outcomes on the covariates $ X = x$, which yields a function to estimate the CATE bound. In the following, we show how the pseudo-outcomes are constructed (additional derivations are provided in Appendix~\ref{app:derivation_estimator}).

\textbf{Lower bound}: The pseudo-outcome for the lower bound is defined by
\vspace{-0.2cm}
{\footnotesize{
\begin{align}
\label{eq:lower_bound_learner}
    &\hat{\phi}^{-}(x, a) \\
    &= \frac{1(A = a, \Delta = 0)}{\hat{\pi}_a(x)}\{\tilde{T} - \hat{\nu}(0, x, a)\}+ \hat{\nu}(0, x, a)[1-\hat{\xi}(x, a)] + \frac{\hat{\nu}(0, x, a)1(A = a)}{\hat{\pi}_a(x)} \{1(\Delta = 0)- [1-\hat{\xi}(x, a)]\} \nonumber \\
    &+ \frac{1(A = a, \Delta = 1)}{\hat{\pi}_a(x)}\{\tilde{T} - \hat{\nu}(1, x, a)\} + \hat{\nu}(1, x, a)\hat{\xi}(x, a) + \frac{\hat{\nu}(1, x, a)1(A = a)}{\hat{\pi}_a(x)} \{1(\Delta = 1)- \hat{\xi}(x, a)\}. \nonumber
\end{align}
}}
\vspace{-0.5cm}

\textbf{Upper bound}: For the pseudo-outcome of the upper bound, we again distinguish two cases as above. These correspond to having a sensitivity function $\gamma(x, a)$ with domain knowledge in \textbf{Case}~\circledgreen{1}, and a conservative upper bound in \textbf{Case}~\circledgreen{2}.

$\bullet$ \textbf{Case} \circledgreen{1} (\emph{domain-knowledge upper bounds).} In this case, the upper bound for a given $\gamma(x, a)$ is stated in Eq.~(\ref{eq:upper_gamma_bound}). The corresponding pseudo-outcome is then given by
{\footnotesize{
\begin{align}
\label{eq:upper_gamma_bound_learner}
    &\hat{\phi}^{+}(x, a)\\
    &= \frac{1(A = a, \Delta = 0)}{\hat{\pi}_a(x)}\{\tilde{T} - \hat{\nu}(0, x, a)\}+ \hat{\nu}(0, x, a)[1-\hat{\xi}(x, a)]
    + \frac{\hat{\nu}(0, x, a)1(A = a)}{\hat{\pi}_a(x)} \{1(\Delta = 0)- [1-\hat{\xi}(x, a)]\} \nonumber\\
    &+ \frac{1(A = a, \Delta = 1)}{\hat{\pi}_a(x)}\{\tilde{T} - \hat{\nu}(1, x, a)\}+ \hat{\nu}(1, x, a)\hat{\xi}(x, a)
    + \frac{\hat{\nu}(1, x, a)1(A = a)}{\hat{\pi}_a(x)} \{1(\Delta = 1)- \hat{\xi}(x, a)\} \nonumber\\
    &+ \gamma(x, a) \frac{1(A = a)}{\hat{\pi}_a(x)} \left\{ 1(\Delta = 1) - \hat{\xi}(x, a)\right\}
    + \gamma(x, a) \hat{\xi}(x, a). \nonumber
\end{align}}}
\vspace{-0.5cm}

$\bullet$ \textbf{Case} \circledgreen{2} (\emph{conservative upper bounds).} Recall that the conservative upper bound is a special case of $\gamma(x, a) = t_{\mathrm{max}} - \nu(1, x, a)$, as defined in Eq.~(\ref{eq:upper_non_informative_bound}). The corresponding pseudo-outcome is given by
{\footnotesize{
\begin{align}
\label{eq:upper_non_informative_bound_learner}
    &\hat{\phi}^{+}(x, a)\\
    &= \frac{1(A = a, \Delta = 0)}{\hat{\pi}_a(x)}\{\tilde{T} - \hat{\nu}(0, x, a)\} + \hat{\nu}(0, x, a)[1-\hat{\xi}(x, a)]
    + \frac{\hat{\nu}(0, x, a)1(A = a)}{\hat{\pi}_a(x)} \{1(\Delta = 0)- [1-\hat{\xi}(x, a)]\} \nonumber\\
    &+ t_{\mathrm{max}} \frac{1(A = a)}{\hat{\pi}_a(x)} \left\{ 1(\Delta =1) - \hat{\xi}(x, a)\right\} + t_{\mathrm{max}}\hat{\xi}(x, a). \nonumber
\end{align}}}
\vspace{-0.4cm}
The full estimation procedure is stated in Algorithm~\ref{alg:pseudocode}.

\subsection{Theoretical properties}
\vspace{-0.2cm}

Below, we derive several favorable theoretical properties for our \method, including consistency, double robustness, and quasi-oracle efficiency. 

\begin{theorem}[Consistency and double robustness]\label{thm:DR-property}
    Let $\hat{\nu}(\delta, x, a)$, $\hat{\xi}(x, a)$, and $\hat{\pi}_a(x)$, denote the estimators of $\nu(\delta, x, a)$, $\xi(x, a)$, and $\pi_a(x)$, respectively. Then, for all $x \in \gX$, we have that $\E\left[\hat{\phi}^{+}(x, a) \mid X = x\right] = \mu^{+}(x, a)$ and $\E\left[\hat{\phi}^{-}(x, a) \mid X = x\right] = \mu^{-}(x, a)$ if at least one of the following conditions holds: 
    (1) $\hat{\pi}_a(x) = \pi_a(x)$, or
    (2) $\hat{\nu}(\delta, x, a) = \nu(\delta, x, a)$ and $\hat{\xi}(x, a) = \xi(x, a)$.
\end{theorem}
\vspace{-0.4cm}
\begin{proof}
    See Appendix~\ref{app:proof_DR_property}.
\end{proof}
\vspace{-0.4cm}

Both consistency and double robustness are general properties of our \method, regardless of the underlying machine learning model. Consistency means that, under correct specification of the nuisance components, our estimator converges to the true target quantity as the sample size grows. Double robustness means that the bounds remain consistent if \underline{either} the propensity score function $\pi_a(x)$ is correctly specified, \underline{or} if the combination of the expected conditional survival time function $\nu(\delta, x, a)$ and the censoring strength $\xi(x, a)$ are correctly specified. Hence, our \method is robust against misspecification in one set of nuisance components. Importantly, the estimated CATE bounds also inherit double robustness, since they are constructed from the CAPO bounds. This property parallels the double robustness of the DR-learner in the standard CATE setting \citep{kennedy.2023}, but here we extend it to partial identification under informative censoring.

We further derive an asymptotic bound on the convergence rate of \method under standard smoothness assumptions. For this, we consider $s$-smooth functions from the Hölder class $\gH(s)$, which achieve Stone's minimax rate~\citep{stone.1980} of $n^{\frac{-2s}{2s+p}}$, where $p$ is the dimension of $\gX$.

\begin{assumption}[Smoothness]
\label{assumption:smoothness}
We assume that (1) the nuisance component $\nu(\delta,x, a)$ is $\alpha$-smooth, $\xi(x, a)$ is $\beta$-smooth, and $\pi_a(x)$ is $\zeta$-smooth; (2) all nuisance components are estimated with their respective minimax rate of $n^{\frac{-2k}{2k+p}}$, where $k\in \{\alpha, \beta, \zeta\}$; and (3) the oracle CAPO $\tau_a(\cdot)$ and the oracle CATE $\tau(\cdot)$ are $\eta$-smooth and that the initial CATE estimator $\hat{\tau}(x)$ converges with rate $r_{\tau}(n)$.
\end{assumption}

\begin{assumption}[Boundedness]
\label{assumption:boundedness}
We assume that there exist constants $C$, $\varepsilon$, $L > 0$ such that, for all $x\in \gX$, it holds that: (1) $\mid \nu(\delta, x, a)\mid \leq C$, (2) $\varepsilon< \hat{\pi}_a(x) < 1-\varepsilon$; and (3) $\mid \hat{\mu}^{\pm} (x, a) \mid \leq L$.
\end{assumption}

Assumptions~\ref{assumption:smoothness} and \ref{assumption:boundedness} are standard in the literature and in line with previous works on the theoretical analysis of CATE point estimators \citep{curth.2021, kennedy.2023} and estimators for partial identification bounds \citep{oprescu.2023}. Assumption~\ref{assumption:smoothness} provides a strategy to quantify the difficulty of the underlying nonparametric regression problems through smoothness conditions, while Assumption~\ref{assumption:boundedness} ensures that both the oracle bounds for the CATE and estimators are bounded. Based on this, we state the quasi-oracle efficiency of our \method.

\begin{theorem}[Quasi-oracle efficiency]
\label{thm:oracle-efficiency}
Let $\mathcal{D}_l$ for $l \in \{1, 2, 3\}$ be independent samples of size $n$. Let $\hat{\nu}(\delta, x, a)$, $\hat{\xi}(x, a)$ be trained on $\mathcal{D}_1$, and $\hat{\pi}_a(x)$ be trained on $\mathcal{D}_2$. We denote $\hat{\phi}^{\pm}(x, a)$ as the pseudo-outcome and $\hat{\mu}^{\pm}(x, a)=\hat{\E}_n[\hat{\phi}^{\pm}(x, a)\mid X=x]$ as the pseudo-outcome regression on $\mathcal{D}_3$ for some generic estimator $\hat{\E}_n[\cdot\mid X=x]$ of ${\E}_n[\cdot\mid X=x]$. Suppose that the second-stage estimator $\hat{\E}_n$ achieves the minimax rate $n^{\frac{-2\eta}{2\eta+p}}$ and satisfies the stability assumption from~\citet{kennedy.2023}. Under the assumptions of smoothness, the oracle risk has the following upper bound:
\begin{align}
\label{eq:oracle_risk}
&\E \left[(\hat{\mu}^{\pm}(x, a) - {\mu}^{\pm}_{a}(x))^2\right] \lesssim n^{\frac{-2\eta}{2\eta+p}} + n^{\left(\frac{-2\zeta}{2\zeta+p}-\frac{2\beta}{2\beta+p}\right)}+n^{\left(\frac{-2\zeta}{2\zeta+p} - \frac{2\alpha}{2\alpha+p}\right)}.
\end{align}
\end{theorem}
\vspace{-0.6cm}
\begin{proof}
    See Appendix~\ref{app:proof_quasi-orcale_efficiency}. In the proof, we even provide a more general bound, which depends on the pointwise mean-squared errors of the nuisance parameters estimators.
\end{proof}
\vspace{-0.2cm}

The above theorem gives an upper bound on the risk of the \method in comparison to the theoretical optimal convergence rate that would be achieved if knowing the oracle nuisance function. Hence, quasi-oracle efficiency means that the \method asymptotically achieves the same convergence rate as if the ground-truth nuisance functions were known. Eq.~(\ref{eq:oracle_risk}) illustrates this property by shown that, even if the propensity score is highly complex and its estimator converges slowly, our \method still converges quickly as long as the other nuisance estimators (i.e. the combination of expected conditional survival time function $\nu(\delta, x, a)$ and censoring strength $\xi(x, a)$) converge sufficiently fast. To derive the above bound, we leverage the cross-fitting approach from~\citet{kennedy.2023}, which was initially used to analyze the DR-learner for CATE estimation, where it allowed for the derivation of robust convergence rates. This technique has since been adapted to other meta-learners (e.g., \citep{curth.2021}), but, notably, not to partial identification in the causal survival setting.

\vspace{-0.4cm}
\section{Experiments}
\vspace{-0.4cm}

Further, we evaluate the effectiveness of the bounds generated by our SurvB-learner through experiments on synthetic datasets (we provide the results of real-world dataset in Appendix~\ref{app:adjuvant_results}). Synthetic datasets are commonly used to evaluate causal inference methods~\citep{vanderlaan.2003, frauen.2025, curth.2021b}, since they have the advantage of providing access to the ground-truth CATE and thus allow for direct comparison against oracle estimates. We additionally benchmark against standard \emph{point estimator} that assumes non-informative censoring in Appendix~\ref{app:drcut_comparison}.

\begin{wraptable}[19]{r}{0.45\textwidth}
    \centering
    \resizebox{\linewidth}{!}{%
    \begin{tabular}{lc cccc}
    \toprule
    \multirow{2}{*}{\centering Bound Type} & \multirow{2}{*}{\centering \makecell[c]{Censoring \\ strength}}& \multicolumn{2}{c}{\circledgreen{1} Domain knowledge} & 
    \multicolumn{2}{c}{\circledgreen{2} Conservative} \\
    \cmidrule(lr){3-4} \cmidrule(lr){5-6} 
    & & Plug-in learner & SurvB-learner & Plug-in learner & SurvB-learner  \\
    \midrule
    \multirow{3}{*}{Exponential function} 
    & 0.2
    & $3.244 \pm 0.152$ & $\mathbf{2.541 \pm 0.138}$ & $2.120 \pm 0.994$  & $\mathbf{1.477 \pm 1.006}$ \\
    & 0.4 
    & $4.265 \pm 0.129$ & $\mathbf{3.429 \pm 0.071}$ & $2.572 \pm 1.571$  & $\mathbf{2.006 \pm 1.396}$ \\
    & 0.6
    & $5.133 \pm 0.120$ & $\mathbf{3.995 \pm 0.066}$ & $2.969 \pm 2.008$  & $\mathbf{2.305 \pm 1.610}$ \\
    \midrule
    \multirow{3}{*}{Sin function} 
    & 0.2
    & $1.400 \pm 0.067$ & $\mathbf{0.799 \pm 0.020}$  & $1.104 \pm 0.140$ & $\mathbf{0.576 \pm 0.114}$ \\
    & 0.4 
    & $1.602 \pm 0.095$ & $\mathbf{1.064 \pm 0.010}$  & $1.127 \pm 0.240$ & $\mathbf{0.758 \pm 0.158}$ \\
    & 0.6
    & $1.782 \pm 0.082$ & $\mathbf{1.243 \pm 0.016}$  & $1.199 \pm 0.299$ & $\mathbf{0.868 \pm 0.177}$ \\
    \midrule
    \multirow{3}{*}{Logistic-sin function} 
    & 0.2
    & $1.646 \pm 0.052$ & $\mathbf{1.112 \pm 0.047}$ & $1.273 \pm 0.285$  & $\mathbf{0.747 \pm 0.306}$ \\
    & 0.4 
    & $2.060 \pm 0.092$ & $\mathbf{1.507 \pm 0.057}$ & $1.435 \pm 0.478$  & $\mathbf{1.003 \pm 0.399}$ \\
    & 0.6
    & $2.362 \pm 0.051$ & $\mathbf{1.739 \pm 0.025}$ & $1.552 \pm 0.635$  & $\mathbf{1.145 \pm 0.465}$ \\
    \bottomrule
    \end{tabular}
    }
    \vspace{-0.2cm}
    {\scriptsize \raggedright \textsuperscript{*} Smaller is better. Best value in bold.\par}
    \caption{\textbf{Results for experiments with synthetic data}. The experiments serve two purposes: (i) to show that our \method obtains values close to the oracle bounds, and (ii) to demonstrate that the plug-in learner is not robust and therefore inferior. Experimental details are in Appendix~\ref{app:implementation_detail}. We report the mean $\pm$ standard deviation of the RMSE compared to the oracle bound over 5 random runs for different synthetic datasets. Overall, we generate over 60 different datasets under varying scenarios. Our \method outperforms the plug-in learner by a clear margin.}
    \label{tab:syn_exp_results}
    \vspace{-0.5cm}
\end{wraptable}
\textbf{Data:} We simulate samples from several scenarios with different underlying functions and censoring probabilities. Here, we consider several scenarios where we vary the censoring probability across datasets (e.g., $\xi = 0.2$, $0.4$, or $0.6$). Details of the data-generating mechanisms are provided in Appendix~\ref{app:data_generation}. Altogether, we generate over 60 different datasets under varying scenarios. We then compare our \method against the oracle CATE and oracle bounds calculated using the ground-truth nuisance estimators. We evaluate our \method for both \textbf{Case} \circledgreen{1} with domain-knowledge bounds and \textbf{Case} \circledgreen{2} with conservative bounds. 

In our experiments, we employ the random forest~\citep{breiman.2001} for nuisance estimation in both the plug-in learner and our \method, which ensures that performance differences arise solely from the meta-learners rather than from the underlying model choice. We apply three-fold cross-fitting for our \method and use the random forest regressors with default hyperparameter settings in the second-stage regression. For more implementation details, we refer to Appendix~\ref{app:implementation_detail}.

\textbf{Result:} We report the results in terms of the root mean squared error (RMSE) with respect to the oracle bounds (Table~\ref{tab:syn_exp_results}). Our \method achieves the smallest average RMSE and the lowest variance, indicating that it captures the underlying patterns in the data well. On average, the RMSE from our \method is $1.43$ times smaller than that of the plug-in learner. This confirms our theoretical motivation: the plug-in learner is inferior due to plug-in bias, while our \method, which is specifically designed to debias the plug-in estimator, performs best.

\begin{wrapfigure}[27]{l}{0.45\textwidth}
    \centering
    \vspace{-0.4cm}
    \includegraphics[width=1\linewidth]{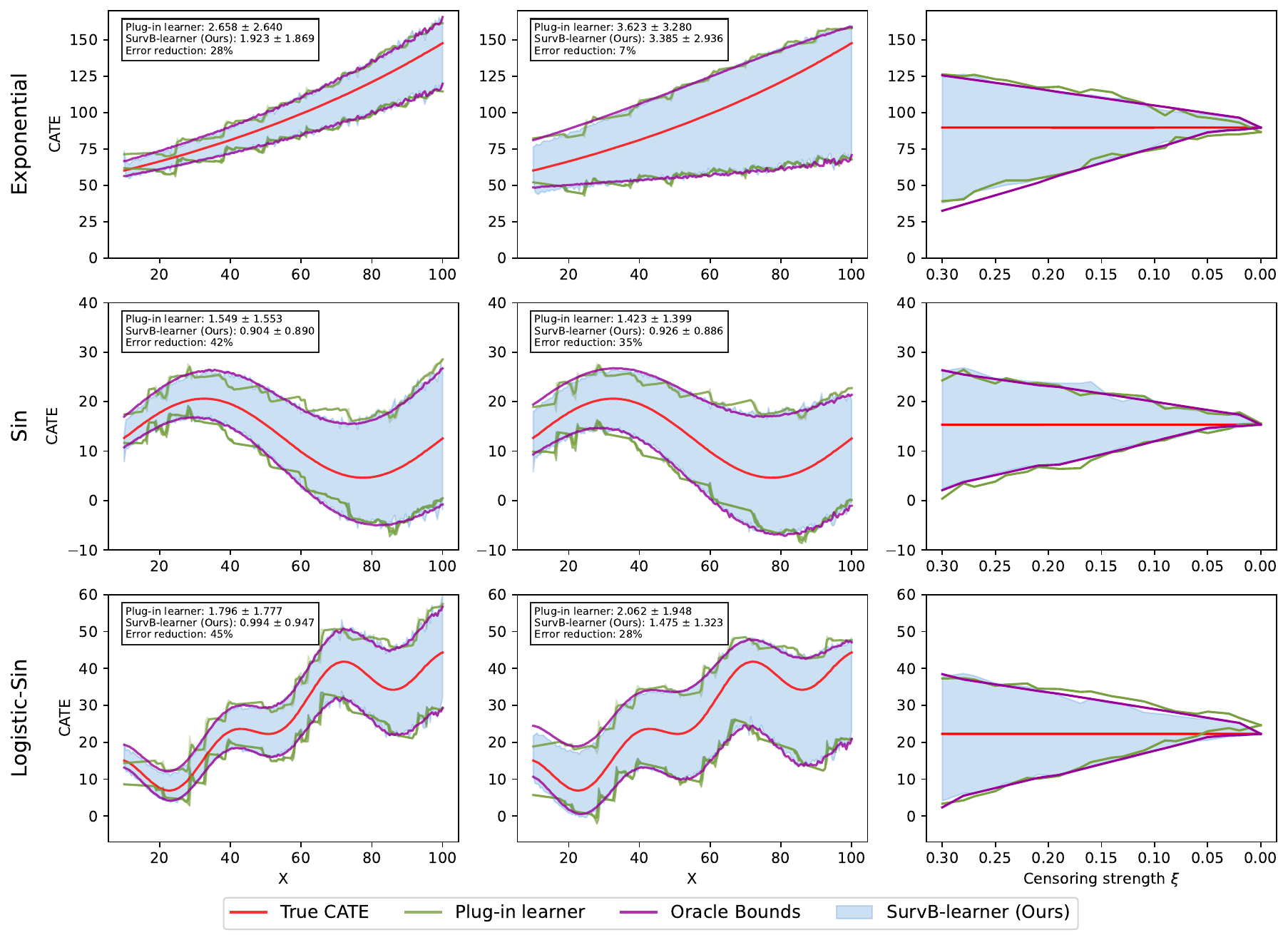}
    \vspace{-0.5cm}
    \caption{\textbf{Results for experiments with synthetic data}. Comparison of estimation methods for bounds based on synthetic datasets. Shown are the results \emph{domain-knowledge bounds} (left) and for \emph{conservative bounds} (center). \textbf{Right:} Estimated bound width across different censoring strengths for the conservative bounds. The width should go to zero as the censoring strength goes to zero. Our \method (shown by the blue shaded area in the background) is close to the oracle bounds and often overlaps with them. In contrast, the plug-in learner is unstable (``wiggly''), because of which the plug-in learner is thus not robust, consistent with theory.}
    \label{fig:syn_exp_results_fig}
    \vspace{-0.8cm}
\end{wrapfigure}

Fig.~\ref{fig:syn_exp_results_fig} provides additional insights by visualizing the oracle bounds in comparison to both the plug-in learner and our \method. We make two key observations: (i)~Our \method reliably learns both valid domain-knowledge bounds and valid conservative bounds (see the left and center columns in Fig.~\ref{fig:syn_exp_results_fig}). Our \method aligns more closely with the oracle bounds than the plug-in learner across different datasets and censoring strengths. Oftentimes, the plug-in learner even reports bounds that are too narrow and thus not faithful. (ii)~We report the widths of the conservative bound (see right column in Fig.~\ref{fig:syn_exp_results_fig}). Therein, the widths of the conservative bounds shrink as censoring strength decreases toward zero. This aligns with our expectation that censoring strength primarily determines the bound width.

\vspace{-0.3cm}
\subsection{Empirical study: ADJUVANT dataset}
\vspace{-0.3cm}

We demonstrate our \method on the ADJUVANT dataset~\citep{liu.2021, zhong.2018}, a real-world clinical trial in non-small cell lung cancer. Here, we apply conservative censoring bounds to partially-identify the CATE. Details about the results, data and implementation are provided in Appendix~\ref{app:adjuvant_results}. $\Rightarrow$ Our \method can reliably identify patient subgroups that benefit from treatment even in small-sample settings, which demonstrates the clinical utility of our \method.

\textbf{Conclusion:} We propose an assumption-lean framework that estimates CATE bounds via our \method to identify subgroups benefiting from treatment under informative censoring, with \textbf{future work} on learning sensitivity functions from external data and extending the framework with stronger inference tools.

\section*{Acknowledgement}
Our research is supported by the DAAD programme Konrad Zuse Schools of Excellence in Artificial Intelligence, sponsored by the Federal Ministry of Research, Technology and Space. This work has also been supported by the German Federal Ministry of Research, Technology and Space (Grant: 03LWH0181B).

\bibliographystyle{abbrvnat}
\bibliography{literature}

\begin{thebibliography}{99}
\providecommand{\natexlab}[1]{#1}
\providecommand{\url}[1]{\texttt{#1}}
\expandafter\ifx\csname urlstyle\endcsname\relax
  \providecommand{\doi}[1]{doi: #1}\else
  \providecommand{\doi}{doi: \begingroup \urlstyle{rm}\Url}\fi

\bibitem[Bai and Cui(2025)]{bai.2025}
Y.~Bai and Y.~Cui.
\newblock Partial causal identification for right censored data with noncompliance.
\newblock \emph{Journal of Nonparametric Statistics}, 2025.

\bibitem[Barnard(1949)]{barnard.1949}
G.~A. Barnard.
\newblock Statistical inference.
\newblock \emph{Journal of the Royal Statistical Society. Series B (Methodological)}, 11\penalty0 (2):\penalty0 115--149, 1949.

\bibitem[Basu and Klein(1982)]{basu.1982}
A.~P. Basu and J.~P. Klein.
\newblock Some recent results in competing risks theory.
\newblock In \emph{Institute of Mathematical Statistics Lecture Notes - Monograph Series}, pages 216--229. Institute of Mathematical Statistics, Hayward, CA, 1982.
\newblock ISBN 978-0-940600-02-7.

\bibitem[Bo and Ding(2025)]{bo.2025}
N.~Bo and Y.~Ding.
\newblock Estimating interpretable heterogeneous treatment effect with causal subgroup discovery in survival outcomes.
\newblock \emph{arXiv preprint}, arXiv:2409.19241, 2025.

\bibitem[Bonvini et~al.(2022)Bonvini, Kennedy, Ventura, and Wasserman]{bonvini.2022}
M.~Bonvini, E.~Kennedy, V.~Ventura, and L.~Wasserman.
\newblock Sensitivity analysis for marginal structural models.
\newblock \emph{arXiv preprint}, arXiv:2210.04681, 2022.

\bibitem[Breiman(2001)]{breiman.2001}
L.~Breiman.
\newblock Random forests.
\newblock \emph{Machine Learning}, 45\penalty0 (1):\penalty0 5--32, 2001.

\bibitem[Breslow(1975)]{breslow.1975}
N.~E. Breslow.
\newblock Analysis of survival data under the proportional hazards model.
\newblock \emph{International Statistical Review / Revue Internationale de Statistique}, 43\penalty0 (1):\penalty0 45--57, 1975.

\bibitem[Cai and van~der Laan(2019)]{cai.2019}
W.~Cai and M.~J. van~der Laan.
\newblock One-step targeted maximum likelihood for time-to-event outcomes.
\newblock \emph{arXiv preprint}, arXiv:1802.09479, 2019.

\bibitem[Cand{\`e}s et~al.(2023)Cand{\`e}s, Lei, and Ren]{candes.2023}
E.~Cand{\`e}s, L.~Lei, and Z.~Ren.
\newblock Conformalized survival analysis.
\newblock \emph{Journal of the Royal Statistical Society Series B: Statistical Methodology}, 85\penalty0 (1):\penalty0 24--45, 2023.

\bibitem[Cheng et~al.(2022)Cheng, Li, Thomas, and Li]{cheng.2022}
C.~Cheng, F.~Li, L.~E. Thomas, and F.~F. Li.
\newblock Addressing extreme propensity scores in estimating counterfactual survival functions via the overlap weights.
\newblock \emph{American Journal of Epidemiology}, 191\penalty0 (6):\penalty0 1140--1151, 2022.

\bibitem[Colangelo and Lee(2023)]{colangelo.2023}
K.~Colangelo and Y.-Y. Lee.
\newblock Double debiased machine learning nonparametric inference with continuous treatments.
\newblock \emph{arXiv preprint}, arXiv:2004.03036, 2023.

\bibitem[Cox(1972)]{cox.1972}
D.~R. Cox.
\newblock Regression models and life-tables.
\newblock \emph{Journal of the Royal Statistical Society: Series B (Methodological)}, 34\penalty0 (2):\penalty0 187--202, 1972.

\bibitem[Cui et~al.(2023)Cui, Kosorok, Sverdrup, Wager, and Zhu]{cui.2023}
Y.~Cui, M.~R. Kosorok, E.~Sverdrup, S.~Wager, and R.~Zhu.
\newblock Estimating heterogeneous treatment effects with right-censored data via causal survival forests.
\newblock \emph{Journal of the Royal Statistical Society Series B: Statistical Methodology}, 85\penalty0 (2):\penalty0 179--211, 2023.

\bibitem[Curth and van~der Schaar(2021)]{curth.2021}
A.~Curth and M.~van~der Schaar.
\newblock Nonparametric estimation of heterogeneous treatment effects: From theory to learning algorithms.
\newblock In \emph{AISTATS}, 2021.

\bibitem[Curth et~al.(2021)Curth, Lee, and {van der Schaar}]{curth.2021b}
A.~Curth, C.~Lee, and M.~{van der Schaar}.
\newblock Survite: Learning heterogeneous treatment effects from time-to-event data.
\newblock In \emph{NeurIPS}, 2021.

\bibitem[Davidov et~al.(2025)Davidov, Feldman, Shamai, Kimmel, and Romano]{davidov.2025}
H.~Davidov, S.~Feldman, G.~Shamai, R.~Kimmel, and Y.~Romano.
\newblock Conformalized survival analysis for general right-censored data.
\newblock In \emph{ICLR}, 2025.

\bibitem[Detterbeck et~al.(2017)Detterbeck, Boffa, Kim, and Tanoue]{detterbeck.2017}
F.~C. Detterbeck, D.~J. Boffa, A.~W. Kim, and L.~T. Tanoue.
\newblock The eighth edition lung cancer stage classification.
\newblock \emph{Chest}, 151\penalty0 (1):\penalty0 193--203, 2017.

\bibitem[Dorn and Guo(2023)]{dorn.2023}
J.~Dorn and K.~Guo.
\newblock Sharp sensitivity analysis for inverse propensity weighting via quantile balancing.
\newblock \emph{Journal of the American Statistical Association}, 118\penalty0 (544):\penalty0 2645--2657, 2023.

\bibitem[Dorn et~al.(2025)Dorn, Kevin, and Kallus]{dorn.2025}
J.~Dorn, G.~Kevin, and N.~Kallus.
\newblock Doubly-valid/doubly-sharp sensitivity analysis for causal inference with unmeasured confounding.
\newblock \emph{Journal of the American Statistical Association}, 120\penalty0 (549):\penalty0 331--342, 2025.

\bibitem[Duarte et~al.(2024)Duarte, Finkelstein, Knox, Mummolo, and Shpitser]{duarte.2024}
G.~Duarte, N.~Finkelstein, D.~Knox, J.~Mummolo, and I.~Shpitser.
\newblock An automated approach to causal inference in discrete settings.
\newblock \emph{Journal of the American Statistical Association}, 119\penalty0 (547):\penalty0 1778--1793, 2024.

\bibitem[Fizazi et~al.(2017)Fizazi, Tran, Fein, Matsubara, {Rodriguez-Antolin}, Alekseev, {\"O}zg{\"u}ro{\u g}lu, Ye, Feyerabend, Protheroe, De~Porre, Kheoh, Park, Todd, Chi, and {LATITUDE Investigators}]{fizazi.2017}
K.~Fizazi, N.~Tran, L.~Fein, N.~Matsubara, A.~{Rodriguez-Antolin}, B.~Y. Alekseev, M.~{\"O}zg{\"u}ro{\u g}lu, D.~Ye, S.~Feyerabend, A.~Protheroe, P.~De~Porre, T.~Kheoh, Y.~C. Park, M.~B. Todd, K.~N. Chi, and {LATITUDE Investigators}.
\newblock Abiraterone plus prednisone in metastatic, castration-sensitive prostate cancer.
\newblock \emph{The New England Journal of Medicine}, 377\penalty0 (4):\penalty0 352--360, 2017.

\bibitem[Frauen et~al.(2023)Frauen, Melnychuk, and Feuerriegel]{frauen.2023a}
D.~Frauen, V.~Melnychuk, and S.~Feuerriegel.
\newblock Sharp bounds for generalized causal sensitivity analysis.
\newblock In \emph{NeurIPS}, 2023.

\bibitem[Frauen et~al.(2024)Frauen, Imrie, Curth, Melnychuk, Feuerriegel, and {van der Schaar}]{frauen.2024a}
D.~Frauen, F.~Imrie, A.~Curth, V.~Melnychuk, S.~Feuerriegel, and M.~{van der Schaar}.
\newblock A neural framework for generalized causal sensitivity analysis.
\newblock In \emph{ICLR}, 2024.

\bibitem[Frauen et~al.(2025)Frauen, Schr{\"o}der, Hess, and Feuerriegel]{frauen.2025}
D.~Frauen, M.~Schr{\"o}der, K.~Hess, and S.~Feuerriegel.
\newblock Orthogonal survival learners for estimating heterogeneous treatment effects from time-to-event data.
\newblock In \emph{NeurIPS}, 2025.

\bibitem[Gao and Hastie(2025)]{gao.2022}
Z.~Gao and T.~Hastie.
\newblock Estimating heterogeneous treatment effects for general responses.
\newblock \emph{Biometrics}, 81\penalty0 (4):\penalty0 ujaf162, 2025.

\bibitem[Gui et~al.(2024)Gui, Hore, Ren, and Barber]{gui.2024}
Y.~Gui, R.~Hore, Z.~Ren, and R.~F. Barber.
\newblock Conformalized survival analysis with adaptive cut-offs.
\newblock \emph{Biometrika}, 111\penalty0 (2):\penalty0 459--477, 2024.

\bibitem[Gunsilius(2020)]{gunsilius.2020}
F.~Gunsilius.
\newblock A path-sampling method to partially identify causal effects in instrumental variable models.
\newblock \emph{arXiv preprint}, arXiv:1910.09502, 2020.

\bibitem[Gupta et~al.(2025)Gupta, Schwartz, Saha, Wen, Rahman, and Trippa]{gupta.2025}
T.~R. Gupta, D.~Schwartz, R.~Saha, P.~Wen, R.~Rahman, and L.~Trippa.
\newblock Informative censoring in externally controlled clinical trials: a potential source of bias.
\newblock \emph{ESMO Open}, 10\penalty0 (1):\penalty0 104094, 2025.

\bibitem[Harrell(2015)]{harrell.2015}
F.~E. Harrell.
\newblock \emph{Regression Modeling Strategies: With Applications to Linear Models, Logistic and Ordinal Regression, and Survival Analysis}.
\newblock Springer Series in Statistics. Springer International Publishing, 2015.

\bibitem[Henderson et~al.(2020)Henderson, Louis, Rosner, and Varadhan]{henderson.2020}
N.~C. Henderson, T.~A. Louis, G.~L. Rosner, and R.~Varadhan.
\newblock Individualized treatment effects with censored data via fully nonparametric {B}ayesian accelerated failure time models.
\newblock \emph{Biostatistics}, 21\penalty0 (1):\penalty0 50--68, 2020.

\bibitem[Heng and Small(2021)]{heng.2021}
S.~Heng and D.~S. Small.
\newblock Sharpening the rosenbaum sensitivity bounds to address concerns about interactions between observed and unobserved covariates.
\newblock \emph{Statistica Sinica}, 31:\penalty0 2331--2353, 2021.

\bibitem[Hern{\'a}n et~al.(2004)Hern{\'a}n, {Hern{\'a}ndez-D{\'i}az}, and Robins]{hernan.2004}
M.~A. Hern{\'a}n, S.~{Hern{\'a}ndez-D{\'i}az}, and J.~M. Robins.
\newblock A structural approach to selection bias.
\newblock \emph{Epidemiology}, 15\penalty0 (5):\penalty0 615--625, 2004.

\bibitem[Hu et~al.(2021)Hu, Ji, and Li]{hu.2021}
L.~Hu, J.~Ji, and F.~Li.
\newblock Estimating heterogeneous survival treatment effect in observational data using machine learning.
\newblock \emph{Statistics in Medicine}, 40\penalty0 (21):\penalty0 4691--4713, 2021.

\bibitem[Hu et~al.(2023)Hu, Liu, Wang, Liu, Zhang, and Zhao]{hu.2023}
Y.~Hu, S.~Liu, L.~Wang, Y.~Liu, D.~Zhang, and Y.~Zhao.
\newblock Treatment-free survival after discontinuation of immune checkpoint inhibitors in m{NSCLC}: a systematic review and meta-analysis.
\newblock \emph{Frontiers in Immunology}, 14, 2023.

\bibitem[Hui et~al.(2013)Hui, Glitza, Chisholm, Yennu, and Bruera]{hui.2013}
D.~Hui, I.~Glitza, G.~Chisholm, S.~Yennu, and E.~Bruera.
\newblock Attrition rates, reasons, and predictive factors in supportive care and palliative oncology clinical trials.
\newblock \emph{Cancer}, 119\penalty0 (5):\penalty0 1098--1105, 2013.

\bibitem[Ichimura and Newey(2021)]{ichimura.2021}
H.~Ichimura and W.~K. Newey.
\newblock The influence function of semiparametric estimators.
\newblock \emph{arXiv preprint}, arXiv:1508.01378, 2021.

\bibitem[Imai et~al.(2015)Imai, Mori, Wakuda, Ono, Akamatsu, Shukuya, Taira, Kenmotsu, Naito, Kaira, Murakami, Endo, Nakajima, Yamamoto, and Takahashi]{imai.2015}
H.~Imai, K.~Mori, K.~Wakuda, A.~Ono, H.~Akamatsu, T.~Shukuya, T.~Taira, H.~Kenmotsu, T.~Naito, K.~Kaira, H.~Murakami, M.~Endo, T.~Nakajima, N.~Yamamoto, and T.~Takahashi.
\newblock Progression-free survival, post-progression survival, and tumor response as surrogate markers for overall survival in patients with extensive small cell lung cancer.
\newblock \emph{Annals of Thoracic Medicine}, 10\penalty0 (1):\penalty0 61--66, 2015.

\bibitem[Imbens(2004)]{imbens.2004}
G.~W. Imbens.
\newblock Nonparametric estimation of average treatment effects under exogeneity: A review.
\newblock \emph{The Review of Economics and Statistics}, 86\penalty0 (1):\penalty0 4--29, 2004.

\bibitem[Inoue et~al.(2025)Inoue, Kitahara, Karayama, Asada, Nishimoto, Matsuura, Hashimoto, Fujii, Matsui, Inami, Toyoshima, Matsuda, Ikeda, Niwa, Kaida, Sato, Ito, Yasui, Suzuki, Hozumi, Furuhashi, Enomoto, Fujisawa, Inui, and Suda]{inoue.2025}
Y.~Inoue, Y.~Kitahara, M.~Karayama, K.~Asada, K.~Nishimoto, S.~Matsuura, D.~Hashimoto, M.~Fujii, T.~Matsui, N.~Inami, M.~Toyoshima, H.~Matsuda, M.~Ikeda, M.~Niwa, Y.~Kaida, M.~Sato, Y.~Ito, H.~Yasui, Y.~Suzuki, H.~Hozumi, K.~Furuhashi, N.~Enomoto, T.~Fujisawa, N.~Inui, and T.~Suda.
\newblock Post-discontinuation survival in patients with advanced {NSCLC} receiving immune checkpoint inhibitors: A pooled analysis of prospective cohort studies.
\newblock \emph{JTO Clinical and Research Reports}, 6\penalty0 (8):\penalty0 100847, 2025.

\bibitem[Jesson et~al.(2021)Jesson, Mindermann, Gal, and Shalit]{jesson.2021}
A.~Jesson, S.~Mindermann, Y.~Gal, and U.~Shalit.
\newblock Quantifying ignorance in individual-level causal-effect estimates under hidden confounding.
\newblock In \emph{ICML}, 2021.

\bibitem[Jesson et~al.(2022)Jesson, Douglas, Manshausen, Solal, Meinshausen, Stier, Gal, and Shalit]{jesson.2022}
A.~Jesson, A.~R. Douglas, P.~Manshausen, M.~Solal, N.~Meinshausen, P.~Stier, Y.~Gal, and U.~Shalit.
\newblock Scalable sensitivity and uncertainty analyses for causal-effect estimates of continuous-valued interventions.
\newblock In \emph{NeurIPS}, 2022.

\bibitem[Jin et~al.(2022)Jin, Ren, and Zhou]{jin.2022}
Y.~Jin, Z.~Ren, and Z.~Zhou.
\newblock Sensitivity analysis under the f-sensitivity models: a distributional robustness perspective.
\newblock \emph{arXiv preprint}, arXiv:2203.04373, 2022.

\bibitem[Jin et~al.(2023)Jin, Ren, and Cand{\`e}s]{jin.2023}
Y.~Jin, Z.~Ren, and E.~J. Cand{\`e}s.
\newblock Sensitivity analysis of individual treatment effects: A robust conformal inference approach.
\newblock \emph{Proceedings of the National Academy of Sciences}, 120\penalty0 (6):\penalty0 e2214889120, 2023.

\bibitem[Kallus et~al.(2019)Kallus, Mao, and Zhou]{kallus.2019}
N.~Kallus, X.~Mao, and A.~Zhou.
\newblock Interval estimation of individual-level causal effects under unobserved confounding.
\newblock In \emph{AISTATS}, 2019.

\bibitem[Kaplan and Meier(1958)]{kaplan.1958}
E.~L. Kaplan and P.~Meier.
\newblock Nonparametric estimation from incomplete observations.
\newblock \emph{Journal of the American Statistical Association}, 53\penalty0 (282):\penalty0 457--481, 1958.

\bibitem[Katzman et~al.(2018)Katzman, Shaham, Cloninger, Bates, Jiang, and Kluger]{katzman.2018}
J.~L. Katzman, U.~Shaham, A.~Cloninger, J.~Bates, T.~Jiang, and Y.~Kluger.
\newblock Deepsurv: personalized treatment recommender system using a cox proportional hazards deep neural network.
\newblock \emph{BMC Medical Research Methodology}, 18\penalty0 (1):\penalty0 24, 2018.

\bibitem[Kennedy(2023{\natexlab{a}})]{kennedy.2023}
E.~H. Kennedy.
\newblock Towards optimal doubly robust estimation of heterogeneous causal effects.
\newblock \emph{Electronic Journal of Statistics}, 17\penalty0 (2), 2023{\natexlab{a}}.

\bibitem[Kennedy(2023{\natexlab{b}})]{kennedy.2023a}
E.~H. Kennedy.
\newblock Semiparametric doubly robust targeted double machine learning: a review.
\newblock \emph{arXiv preprint}, arXiv:2203.06469, 2023{\natexlab{b}}.

\bibitem[Kilbertus et~al.(2020)Kilbertus, Kusner, and Silva]{kilbertus.2020}
N.~Kilbertus, M.~J. Kusner, and R.~Silva.
\newblock A class of algorithms for general instrumental variable models.
\newblock In \emph{Advances in Neural Information Processing Systems}, volume~33, pages 20108--20119, 2020.

\bibitem[Klein and Moeschberger(2003)]{klein.2003}
J.~P. Klein and M.~L. Moeschberger.
\newblock \emph{Survival Analysis: Techniques for Censored and Truncated Data}.
\newblock Statistics for Biology and Health. Springer, New York, 2003.

\bibitem[K{\"u}nzel et~al.(2019)K{\"u}nzel, Sekhon, Bickel, and Yu]{kunzel.2019}
S.~R. K{\"u}nzel, J.~S. Sekhon, P.~J. Bickel, and B.~Yu.
\newblock Meta-learners for estimating heterogeneous treatment effects using machine learning.
\newblock \emph{Proceedings of the National Academy of Sciences}, 116\penalty0 (10):\penalty0 4156--4165, 2019.

\bibitem[Leung et~al.(1997)Leung, Elashoff, and Afifi]{leung.1997}
K.~M. Leung, R.~M. Elashoff, and A.~A. Afifi.
\newblock Censoring issues in survival analysis.
\newblock \emph{Annual Review of Public Health}, 18:\penalty0 83--104, 1997.

\bibitem[Liu et~al.(2021)Liu, Bao, Wang, Mao, Chen, Tong, Xu, Wu, Wei, Liu, Chen, Cheng, Yin, Yang, Ren, Li, Li, Huang, Liu, Xu, Chen, Xu, Liu, Yu, Wang, Ma, Yan, Dong, Zhang, Su, Yang, Yang, Zhou, Wu, Shao, Zhong, and Wu]{liu.2021}
S.-Y. Liu, H.~Bao, Q.~Wang, W.-M. Mao, Y.~Chen, X.~Tong, S.-T. Xu, L.~Wu, Y.-C. Wei, Y.-Y. Liu, C.~Chen, Y.~Cheng, R.~Yin, F.~Yang, S.-X. Ren, X.-F. Li, J.~Li, C.~Huang, Z.-D. Liu, S.~Xu, K.-N. Chen, S.-D. Xu, L.-X. Liu, P.~Yu, B.-H. Wang, H.-T. Ma, H.-H. Yan, S.~Dong, X.-C. Zhang, J.~Su, J.-J. Yang, X.-N. Yang, Q.~Zhou, X.~Wu, Y.~Shao, W.-Z. Zhong, and Y.-L. Wu.
\newblock Genomic signatures define three subtypes of {EGFR}-mutant stage {II}–{III} non-small-cell lung cancer with distinct adjuvant therapy outcomes.
\newblock \emph{Nature Communications}, 12\penalty0 (1):\penalty0 6450, 2021.

\bibitem[Manski(1990)]{manski.1990}
C.~F. Manski.
\newblock Nonparametric bounds on treatment effects.
\newblock \emph{The American Economic Review}, 80\penalty0 (2):\penalty0 319--323, 1990.

\bibitem[Mao et~al.(2018)Mao, Li, Yang, and Shen]{mao.2018}
H.~Mao, L.~Li, W.~Yang, and Y.~Shen.
\newblock On the propensity score weighting analysis with survival outcome: Estimands, estimation, and inference.
\newblock \emph{Statistics in Medicine}, 37\penalty0 (26):\penalty0 3745--3763, 2018.

\bibitem[Matsuyama and Yamaguchi(2008)]{matsuyama.2008}
Y.~Matsuyama and T.~Yamaguchi.
\newblock Estimation of the marginal survival time in the presence of dependent competing risks using inverse probability of censoring weighted ({IPCW}) methods.
\newblock \emph{Pharmaceutical Statistics}, 7\penalty0 (3):\penalty0 202--214, 2008.

\bibitem[Murphy et~al.(1997)Murphy, Rossini, and {van der Vaart}]{murphy.1997}
S.~A. Murphy, A.~J. Rossini, and A.~W. {van der Vaart}.
\newblock Maximum likelihood estimation in the proportional odds model.
\newblock \emph{Journal of the American Statistical Association}, 92\penalty0 (439):\penalty0 968--976, 1997.

\bibitem[Nie and Wager(2021)]{nie.2020}
X.~Nie and S.~Wager.
\newblock Quasi-oracle estimation of heterogeneous treatment effects.
\newblock \emph{Biometrika}, 108\penalty0 (2):\penalty0 299--319, 2021.

\bibitem[Oprescu et~al.(2023)Oprescu, Dorn, Ghoummaid, Jesson, Kallus, and Shalit]{oprescu.2023}
M.~Oprescu, J.~Dorn, M.~Ghoummaid, A.~Jesson, N.~Kallus, and U.~Shalit.
\newblock B-learner: Quasi-oracle bounds on heterogeneous causal effects under hidden confounding.
\newblock In \emph{ICML}, 2023.

\bibitem[Padh et~al.(2023)Padh, Zeitler, Watson, Kusner, Silva, and Kilbertus]{padh.2023}
K.~Padh, J.~Zeitler, D.~Watson, M.~Kusner, R.~Silva, and N.~Kilbertus.
\newblock Stochastic causal programming for bounding treatment effects.
\newblock In \emph{CLeaR}, 2023.

\bibitem[Pecci et~al.(2025)Pecci, Thummalapalli, Alden, Ricciuti, Alessi, Elkrief, Rizvi, Wang, Jeng, Egger, Vaz, Barrichello, Lamberti, Di~Federico, Santo, {Rossato de Almeida}, Gandhi, Clark, Nishino, Johnson, Hellmann, Schoenfeld, and Awad]{pecci.2025}
F.~Pecci, R.~Thummalapalli, S.~L. Alden, B.~Ricciuti, J.~V. Alessi, A.~Elkrief, H.~Rizvi, X.~Wang, M.~Jeng, J.~V. Egger, V.~R. Vaz, A.~Barrichello, G.~Lamberti, A.~Di~Federico, V.~Santo, G.~{Rossato de Almeida}, M.~Gandhi, P.~Clark, M.~Nishino, B.~E. Johnson, M.~Hellmann, A.~J. Schoenfeld, and M.~M. Awad.
\newblock Factors associated with disease progression after discontinuation of immune checkpoint inhibitors for immune-related toxicity in patients with advanced non–small cell lung cancer.
\newblock \emph{Clinical Cancer Research}, 31\penalty0 (12):\penalty0 2413--2425, 2025.

\bibitem[{Perez-Cruz} et~al.(2018){Perez-Cruz}, Shamieh, Paiva, Kwon, Muckaden, Bruera, and Hui]{perez-cruz.2018}
P.~E. {Perez-Cruz}, O.~Shamieh, C.~E. Paiva, J.~H. Kwon, M.~A. Muckaden, E.~Bruera, and D.~Hui.
\newblock Factors associated with attrition in a multicenter longitudinal observational study of patients with advanced cancer.
\newblock \emph{Journal of Pain and Symptom Management}, 55\penalty0 (3):\penalty0 938--945, 2018.

\bibitem[Robins and Finkelstein(2000)]{robins.2000}
J.~M. Robins and D.~M. Finkelstein.
\newblock Correcting for noncompliance and dependent censoring in an {AIDS} clinical trial with inverse probability of censoring weighted ({IPCW}) log-rank tests.
\newblock \emph{Biometrics}, 56\penalty0 (3):\penalty0 779--788, 2000.

\bibitem[Robins and Rotnitzky(1992)]{robins.1992}
J.~M. Robins and A.~Rotnitzky.
\newblock Recovery of information and adjustment for dependent censoring using surrogate markers.
\newblock \emph{AIDS epidemiology}, pages 297--331, 1992.

\bibitem[Rosenbaum(1987)]{rosenbaum.1987}
P.~R. Rosenbaum.
\newblock Sensitivity analysis for certain permutation inferences in matched observational studies.
\newblock \emph{Biometrika}, 74\penalty0 (1):\penalty0 13--26, 1987.

\bibitem[Rosenbaum and Rubin(1983)]{rosenbaum.1983}
P.~R. Rosenbaum and D.~B. Rubin.
\newblock Assessing sensitivity to an unobserved binary covariate in an observational study with binary outcome.
\newblock \emph{Journal of the Royal Statistical Society. Series B (Methodological)}, 45\penalty0 (2):\penalty0 212--218, 1983.

\bibitem[Rubin and {van der Laan}(2007)]{rubin.2007}
D.~Rubin and M.~J. {van der Laan}.
\newblock A doubly robust censoring unbiased transformation.
\newblock \emph{The International Journal of Biostatistics}, 3\penalty0 (1):\penalty0 Article 4, 2007.

\bibitem[Rubin(1974)]{rubin.1974}
D.~B. Rubin.
\newblock Estimating causal effects of treatments in randomized and nonrandomized studies.
\newblock \emph{Journal of Educational Psychology}, 66\penalty0 (5):\penalty0 688, 1974.

\bibitem[Rubinstein et~al.(2024)Rubinstein, Agniel, Han, {Horvitz-Lennon}, and Normand]{rubinstein.2025}
M.~Rubinstein, D.~Agniel, L.~Han, M.~{Horvitz-Lennon}, and S.-L. Normand.
\newblock Bounding causal effects with an unknown mixture of informative and non-informative censoring.
\newblock \emph{arXiv preprint}, arXiv:2411.16902, 2024.

\bibitem[Sakaguchi(2024)]{sakaguchi.2024}
S.~Sakaguchi.
\newblock Partial identification and inference in duration models with endogenous censoring.
\newblock \emph{Journal of Applied Econometrics}, 39\penalty0 (2):\penalty0 308--326, 2024.

\bibitem[Schaubel and Wei(2011)]{schaubel.2011}
D.~E. Schaubel and G.~Wei.
\newblock Double inverse-weighted estimation of cumulative treatment effects under nonproportional hazards and dependent censoring.
\newblock \emph{Biometrics}, 67\penalty0 (1):\penalty0 29--38, 2011.

\bibitem[Schrod et~al.(2022)Schrod, Sch{\"a}fer, Solbrig, Lohmayer, Gronwald, Oefner, Bei{\ss}barth, Spang, Zacharias, and Altenbuchinger]{schrod.2022}
S.~Schrod, A.~Sch{\"a}fer, S.~Solbrig, R.~Lohmayer, W.~Gronwald, P.~J. Oefner, T.~Bei{\ss}barth, R.~Spang, H.~U. Zacharias, and M.~Altenbuchinger.
\newblock {BITES}: balanced individual treatment effect for survival data.
\newblock \emph{Bioinformatics}, 38\penalty0 (Supplement\_1):\penalty0 i60--i67, 2022.

\bibitem[Schweisthal et~al.(2024)Schweisthal, Frauen, Van Der~Schaar, and Feuerriegel]{schweisthal.2024}
J.~Schweisthal, D.~Frauen, M.~Van Der~Schaar, and S.~Feuerriegel.
\newblock Meta-learners for partially-identified treatment effects across multiple environments.
\newblock In \emph{ICML}, 2024.

\bibitem[Schweisthal et~al.(2025)Schweisthal, Frauen, Schr{\"o}der, Hess, Kilbertus, and Feuerriegel]{schweisthal.2025}
J.~Schweisthal, D.~Frauen, M.~Schr{\"o}der, K.~Hess, N.~Kilbertus, and S.~Feuerriegel.
\newblock Learning representations of instruments for partial identification of treatment effects.
\newblock In \emph{ICML}, 2025.

\bibitem[Shalit et~al.(2017)Shalit, Johansson, and Sontag]{shalit.2017}
U.~Shalit, F.~D. Johansson, and D.~Sontag.
\newblock Estimating individual treatment effect: generalization bounds and algorithms.
\newblock In \emph{ICML}, 2017.

\bibitem[Shand et~al.(2024)Shand, Stovold, Goulding, and Cheema]{shand.2024}
J.~Shand, E.~Stovold, L.~Goulding, and K.~Cheema.
\newblock Cancer care treatment attrition in adults: Measurement approaches and inequities in patient dropout rates: a rapid review.
\newblock \emph{BMC Cancer}, 24\penalty0 (1):\penalty0 1345, 2024.

\bibitem[Slud and Rubinstein(1983)]{slud.1983}
E.~V. Slud and L.~V. Rubinstein.
\newblock Dependent competing risks and summary survival curves.
\newblock \emph{Biometrika}, 70\penalty0 (3):\penalty0 643--649, 1983.

\bibitem[Soriano et~al.(2023)Soriano, {Ben-Michael}, Bickel, Feller, and Pimentel]{soriano.2023}
D.~Soriano, E.~{Ben-Michael}, P.~J. Bickel, A.~Feller, and S.~D. Pimentel.
\newblock Interpretable sensitivity analysis for balancing weights.
\newblock \emph{Journal of the Royal Statistical Society Series A: Statistics in Society}, 186\penalty0 (4):\penalty0 707--721, 2023.

\bibitem[Stone(1980)]{stone.1980}
C.~J. Stone.
\newblock Optimal rates of convergence for nonparametric estimators.
\newblock \emph{The Annals of Statistics}, 8\penalty0 (6):\penalty0 1348--1360, 1980.

\bibitem[Tabib and Larocque(2020)]{tabib.2020}
S.~Tabib and D.~Larocque.
\newblock Non-parametric individual treatment effect estimation for survival data with random forests.
\newblock \emph{Bioinformatics (Oxford, England)}, 36\penalty0 (2):\penalty0 629--636, 2020.

\bibitem[Tan(2006)]{tan.2006}
Z.~Tan.
\newblock A distributional approach for causal inference using propensity scores.
\newblock \emph{Journal of the American Statistical Association}, 101\penalty0 (476):\penalty0 1619--1637, 2006.

\bibitem[Templeton et~al.(2020)Templeton, Amir, and Tannock]{templeton.2020}
A.~J. Templeton, E.~Amir, and I.~F. Tannock.
\newblock Informative censoring — a neglected cause of bias in oncology trials.
\newblock \emph{Nature Reviews Clinical Oncology}, 17\penalty0 (6):\penalty0 327--328, 2020.

\bibitem[Van Der~Laan and Robins(2003)]{vanderlaan.2003}
M.~J. Van Der~Laan and J.~M. Robins.
\newblock \emph{Unified Methods for Censored Longitudinal Data and Causality}.
\newblock Springer Series in Statistics. Springer, New York, NY, 2003.
\newblock ISBN 978-0-387-21700-0.

\bibitem[Voinot et~al.(2025)Voinot, Berenfeld, Mayer, Sebastien, and Josse]{voinot.2025a}
C.~Voinot, C.~Berenfeld, I.~Mayer, B.~Sebastien, and J.~Josse.
\newblock Treatment effect estimation in causal survival analysis: Practical recommendations.
\newblock \emph{arXiv preprint}, arXiv:2501.05836, 2025.

\bibitem[Wei(1992)]{wei.1992}
L.~J. Wei.
\newblock The accelerated failure time model: a useful alternative to the cox regression model in survival analysis.
\newblock \emph{Statistics in Medicine}, 11\penalty0 (14-15):\penalty0 1871--1879, 1992.

\bibitem[Westling et~al.(2024)Westling, Luedtke, Gilbert, and Carone]{westling.2024}
T.~Westling, A.~Luedtke, P.~B. Gilbert, and M.~Carone.
\newblock Inference for treatment-specific survival curves using machine learning.
\newblock \emph{Journal of the American Statistical Association}, 119\penalty0 (546):\penalty0 1541--1553, 2024.

\bibitem[Wiegrebe et~al.(2024)Wiegrebe, Kopper, Sonabend, Bischl, and Bender]{wiegrebe.2024}
S.~Wiegrebe, P.~Kopper, R.~Sonabend, B.~Bischl, and A.~Bender.
\newblock Deep learning for survival analysis: a review.
\newblock \emph{Artificial Intelligence Review}, 57\penalty0 (3):\penalty0 65, 2024.

\bibitem[Willems et~al.(2025)Willems, Beyhum, and Keilegom]{willems.2025}
I.~Willems, J.~Beyhum, and I.~V. Keilegom.
\newblock Bounds for the regression parameters in dependently censored survival models.
\newblock \emph{arXiv preprint}, arXiv:2503.11210, 2025.

\bibitem[Xia et~al.(2021)Xia, Lee, Bengio, and Bareinboim]{xia.2021}
K.~Xia, K.-Z. Lee, Y.~Bengio, and E.~Bareinboim.
\newblock The causal-neural connection: Expressiveness, learnability, and inference.
\newblock In \emph{NeurIPS}, 2021.

\bibitem[Xia et~al.(2023)Xia, Pan, and Bareinboim]{xia.2023}
K.~Xia, Y.~Pan, and E.~Bareinboim.
\newblock Neural causal models for counterfactual identification and estimation.
\newblock In \emph{ICLR}, 2023.

\bibitem[Xu et~al.(2024)Xu, Cobzaru, Finkelstein, Welsch, Ng, and Shahn]{xu.2024}
S.~Xu, R.~Cobzaru, S.~N. Finkelstein, R.~E. Welsch, K.~Ng, and Z.~Shahn.
\newblock Estimating heterogeneous treatment effects on survival outcomes using counterfactual censoring unbiased transformations.
\newblock \emph{arXiv preprint}, arXiv:2401.11263, 2024.

\bibitem[Xu et~al.(2022)Xu, Ignatiadis, Sverdrup, Fleming, Wager, and Shah]{xu.2022}
Y.~Xu, N.~Ignatiadis, E.~Sverdrup, S.~Fleming, S.~Wager, and N.~Shah.
\newblock Treatment heterogeneity for survival outcomes.
\newblock \emph{arXiv preprint}, arXiv:2207.07758, 2022.

\bibitem[Yin et~al.(2022)Yin, Salem, Dixon, Jin, Cohen, DeGramont, Van~Cutsem, Taieb, Alberts, Wolmark, Schmoll, Saltz, George, Goldberg, Kerr, Lonardi, Yoshino, Yothers, Grothey, Andre, and Shi]{yin.2022}
J.~Yin, M.~E. Salem, J.~G. Dixon, Z.~Jin, R.~Cohen, A.~DeGramont, E.~Van~Cutsem, J.~Taieb, S.~R. Alberts, N.~Wolmark, H.-J. Schmoll, L.~B. Saltz, T.~J. George, R.~R.~M. Goldberg, R.~Kerr, S.~Lonardi, T.~Yoshino, G.~Yothers, A.~Grothey, T.~Andre, and Q.~Shi.
\newblock Reevaluating disease-free survival as an endpoint vs overall survival in stage {III} adjuvant colon cancer trials.
\newblock \emph{JNCI: Journal of the National Cancer Institute}, 114\penalty0 (1):\penalty0 60--67, 2022.

\bibitem[Ying(2024)]{ying.2024}
A.~Ying.
\newblock Proximal survival analysis to handle dependent right censoring.
\newblock \emph{Journal of the Royal Statistical Society Series B: Statistical Methodology}, 86\penalty0 (5):\penalty0 1414--1434, 2024.

\bibitem[Zhang and Singer(2010)]{zhang.2010}
H.~Zhang and B.~H. Singer.
\newblock \emph{Recursive Partitioning and Applications}.
\newblock Springer Science \& Business Media, 2010.
\newblock ISBN 978-1-4419-6824-1.

\bibitem[Zhang et~al.(2017)Zhang, Le, Liu, Zhou, and Li]{zhang.2017}
W.~Zhang, T.~D. Le, L.~Liu, Z.-H. Zhou, and J.~Li.
\newblock Mining heterogeneous causal effects for personalized cancer treatment.
\newblock \emph{Bioinformatics}, 33\penalty0 (15):\penalty0 2372--2378, 2017.

\bibitem[Zhao et~al.(2019)Zhao, Small, and Bhattacharya]{zhao.2019}
Q.~Zhao, D.~S. Small, and B.~B. Bhattacharya.
\newblock Sensitivity analysis for inverse probability weighting estimators via the percentile bootstrap.
\newblock \emph{Journal of the Royal Statistical Society Series B: Statistical Methodology}, 81\penalty0 (4):\penalty0 735--761, 2019.

\bibitem[Zheng and Klein(1995)]{zheng.1995}
M.~Zheng and J.~P. Klein.
\newblock Estimates of marginal survival for dependent competing risks based on an assumed copula.
\newblock \emph{Biometrika}, 82\penalty0 (1):\penalty0 127--138, 1995.

\bibitem[Zhong et~al.(2018)Zhong, Wang, Mao, Xu, Wu, Shen, Liu, Chen, Cheng, Xu, Wang, Fei, Li, Li, Huang, Liu, Xu, Chen, Xu, Liu, Yu, Wang, Ma, Yan, Yang, Zhou, Wu, Wang, Hu, Liu, Li, and Shi]{zhong.2018}
W.-Z. Zhong, Q.~Wang, W.-M. Mao, S.-T. Xu, L.~Wu, Y.~Shen, Y.-Y. Liu, C.~Chen, Y.~Cheng, L.~Xu, J.~Wang, K.~Fei, X.-F. Li, J.~Li, C.~Huang, Z.-D. Liu, S.~Xu, K.-N. Chen, S.-D. Xu, L.-X. Liu, P.~Yu, B.-H. Wang, H.-T. Ma, H.-H. Yan, X.-N. Yang, Q.~Zhou, Y.-L. Wu, S.-Y. Wang, J.~Hu, W.~Liu, W.~Li, and J.-H. Shi.
\newblock Gefitinib versus vinorelbine plus cisplatin as adjuvant treatment for stage {II}–{IIIA} ({N1}–{N2}) {EGFR}-mutant {NSCLC} ({ADJUVANT}/ctong1104): a randomised, open-label, phase 3 study.
\newblock \emph{The Lancet Oncology}, 19\penalty0 (1):\penalty0 139--148, 2018.

\end{thebibliography}
\appendix
\newpage
\section{Framework of our work}
\label{app:methods_overview}

\begin{figure}[htbp]
\centering
    \includegraphics[width=0.9\linewidth]{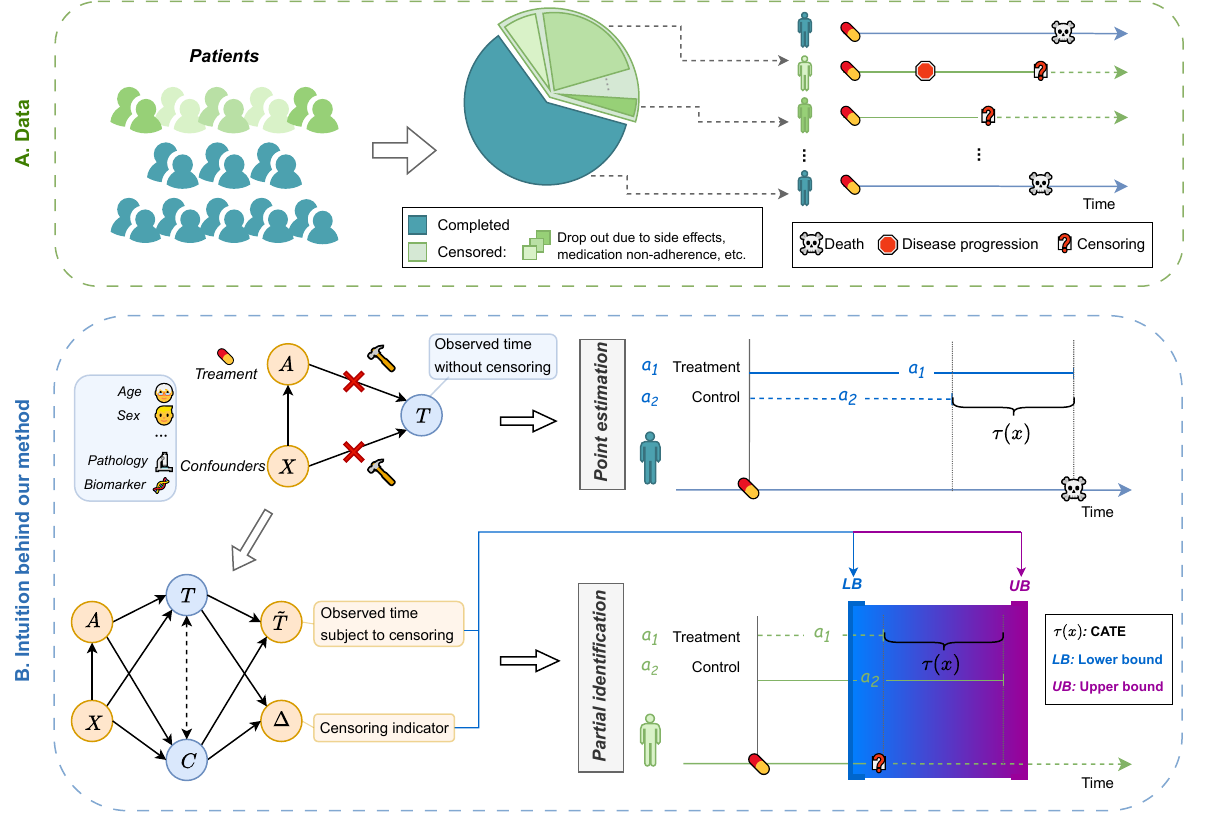}
    \vspace*{0.2 cm}
    \caption{\textbf{Overview of our framework.}
    \textcolor{darkgreen}{\textbf{(A)}~Data:} A key challenge in survival analysis is \emph{censoring}, where follow-up information about patient outcomes is incomplete. In practice, this can be due to various reasons such as patients dropping out of clinical studies for side effects~\citep{gupta.2025, wiegrebe.2024}. As a result, the exact event time (e.g., of death or disease progression) for censored patients is unknown. 
    \textcolor{darkblue}{\textbf{(B)}~Intuition behind our framework:} Under informative censoring (i.e., when dropout depends on survival time), unbiased point estimation of the CATE $\tau_{a_1, a_2}(x)$ is generally not possible. We therefore reframe the estimation task using partial identification, which allows us to construct informative bounds (i.e., {\emph{LB}} and {\emph{UB}}) for the CATE, while explicitly accounting for censoring bias. 
    \textcolor{darkyellow}{\textbf{(C)}}~\textbf{Output example:} We propose a two-stage meta-learner called \textbf{\method} to efficiently estimate these bounds, which provide intuitive insights into how unknown censoring may affect CATE estimates across patient subgroups.}
    \vspace{-0.6cm}
    \label{fig:method_overview}
\end{figure}

\newpage
\section{Extended related work}
\label{app:extended_related_work}

In the following, we provide an overview of existing works that are broadly relevant to our work on partial identification of CATEs with censored datasets. To this end, we review three streams of work: (1)~partial identification for CATE estimation, (2)~methods dealing with informative censoring in survival analysis, and (3)~survival analysis in a causal inference context.

\subsection{Partial identification for CATE estimation}

There is a rich body of literature on the partial identification of causal quantities from uncensored data. Existing works can be broadly categorized into two steps, beginning with methods for general partial identification and followed by those employing specific sensitivity models.

The aim of partial identification is to compute bounds for the causal quantity of interest whenever point identification is not possible~\citep{manski.1990}. Several works for partial identification have developed methods under different settings and thus different structural assumptions, such as discrete variables~\citep{duarte.2024}, instrumental variables~\citep {gunsilius.2020, kilbertus.2020, schweisthal.2024, schweisthal.2025}. More recently, neural network–based approaches have been applied to partial identification~\citep{padh.2023, xia.2021, xia.2023}. However, works on partial identification for causal quantities in survival settings are comparatively scarce. 

Alternatively, a large class of partial identification methods employs sensitivity models to impose assumptions on the strength of hidden confounding. These are typically classified by the underlying sensitivity model. Popular sensitivity models include Rosenbaum’s sensitivity model~\citep{rosenbaum.1983, rosenbaum.1987, heng.2021}, the marginal sensitivity model (MSM)~\citep{oprescu.2023, tan.2006, kallus.2019, zhao.2019, jesson.2021, dorn.2023, frauen.2023a, soriano.2023, dorn.2025}, and the $f$-sensitivity model~\citep{jin.2022}. There are further sensitivity models suitable for continuous treatment settings, which are often based on the MSM~\citep{frauen.2023a, bonvini.2022, jesson.2022} and the $f$-sensitivity model~\citep{jin.2023}. Frauen et al. \citep{frauen.2024a} provide a unified framework that is compatible with multiple sensitivity models, including Rosenbaum’s sensitivity model, the MSM, and the $f$-sensitivity model. However, all of them are dealing with unobserved confoundedness in the datasets and, therefore, are \emph{not} directly applicable to censored datasets, where bounds do \emph{not} come from confounding bias but censoring bias. 

\subsection{Classical survival methods}
In medical practice, survival times are not always observed due to censoring~\citep{leung.1997}. To account for this, several tools have been developed, including the Kaplan-Meier estimator~\citep{kaplan.1958} and the Cox proportional hazards model~\citep {cox.1972, breslow.1975}. Other common approaches include accelerated failure time (AFT) models~\citep{wei.1992} and proportional odds models~\citep{murphy.1997, harrell.2015}, which combine parametric and semi-parametric specifications. Together, these methods form the foundation of standard survival analysis, but they typically assume that censoring is \emph{non-informative} (while our focus is on \emph{informative} censoring).

\subsection{Survival analysis in causal inference contexts}

\textbf{Efficient average treatment effect estimation for time-to-event data:} Most work on semiparametric inference for time-to-event data has focused on average treatment effects (ATEs). For example, \citet{rubin.2007} proposed so-called doubly robust censoring unbiased transformations for semiparametric efficient inference on the ATE under censoring. Furthermore, various works proposed one-step Targeted Maximum Likelihood Estimators (TMLE) for estimating average causal quantities in survival settings~\citep{vanderlaan.2003, cai.2019, westling.2024}. Complementary to these, \citet{cheng.2022} proposed inverse probability of censoring weighted (IPCW) estimators for weighted average survival functions, though these rely on correctly specified parametric nuisance models.

Early developments in causal survival analysis have focused on addressing the bias induced by \emph{informative} censoring. Methods such as those in \citep{robins.1992, robins.2000, matsuyama.2008} developed inverse probability of censoring weighting (IPCW) techniques to obtain point estimates under specific model, e.g., Kaplan-Meier esitmator, proportional hazard model and AFT models. However, these approaches are not framed around current potential outcomes frameworks and therefore do not target causal estimands such as the CATE. Moreover, none of them are model-agnostic sensitivity tools: their bounds are tied to particular survival models.

Building upon this gap, recent work extends survival analysis to causal inference, where the goal is to estimate treatment effects, such as the CATE, under censoring. Here, methods fall broadly into two categories: model-based learners and meta-learners. The former, model-based learners, include tree-based methods~\citep{zhang.2017, henderson.2020, tabib.2020, cui.2023, hu.2021} and neural-network-based methods~\citep{schrod.2022, katzman.2018, curth.2021}. In contrast, few works have proposed (orthogonal) meta-learners specifically designed for censored time-to-event data~\citep{gao.2022, xu.2024, frauen.2025, xu.2022}. However, these methods focus on \emph{point} estimation and typically require \emph{censoring-independent assumptions}, which limit their use in real-world applications. Another line of research develops inferential methods for censored data based on conformal inference~\citep{candes.2023, gui.2024, davidov.2025}, but these approaches aim to produce prediction intervals for individual treatment effects rather than bounds on CATEs due to censoring.

Some recent studies have aimed to relax the standard assumption of non-informative censoring. For example, \citet{bai.2025} developed a partial identification method for the ATE using instrumental variables. \citet{rubinstein.2025} proposed a method for bounding the ATE under a mixture of informative and non-informative censoring, which relies on the ratio of informative censoring in the dataset to bound the ATE rather than focusing solely on informative censoring. Our method develops bounds for CATE under informative censoring. \citet{ying.2024} proposed a proxy-based causal inference method. However, these approaches focus on the ATE and \emph{not} the CATE under informative censoring.

\emph{Research gap:} To the best of our knowledge, no existing work provides partial identification bounds for the CATE in survival analysis under informative censoring.

\newpage
\section{\method for CATE}
\label{app:survb-learner_cate}

In our main paper, we have stated our \method for the CAPOs. In this section, we now state our \method for the CATE. For this, we follow Theorem~\ref{thm:bounds} and take the difference between the lower bounds and upper bounds of CAPOs. As in the main paper, we distinguish two cases: \textbf{Case}~\circledgreen{1} with domain knowledge, and \textbf{Case}~\circledgreen{2} with a conservative upper bound.

$\bullet$ \textbf{Case} \circledgreen{1} (\emph{domain-knowledge upper bounds).} Here, we construct upper bounds for a given $\gamma(x, a)$. Then, the pseudo-outcome of the upper bound is given by
{\footnotesize {\begin{equation}\label{eq:upper_gamma_bound_cate_learner}
\begin{aligned}
    &\hat{\phi}^{+}_{a_1, a_2}(x)\\
    =& \hat{\phi}^{+}(x, a_1) - \hat{\phi}^{-}(x, a_2) \\
    =& \frac{1(A = a_1, \Delta = 0)}{\hat{\pi}_{a_1}(x)}\{\tilde{T} - \hat{\nu}(0, x, a_1)\} + \frac{\hat{\nu}(0, x, a_1)1(A = a_1)}{\hat{\pi}_{a_1}(x)} \{1(\Delta = 0)- [1-\hat{\xi}(x, a_1)]\} + \hat{\nu}(0, x, a_1)[1-\hat{\xi}(x, a_1)]\\
    &+ \frac{1(A = a_1, \Delta = 1)}{\hat{\pi}_{a_1}(x)}\{\tilde{T} - \hat{\nu}(1, x, a_1)\} + \frac{\hat{\nu}(1, x, a_1)1(A = a_1)}{\hat{\pi}_{a_1}(x)} \{1(\Delta = 1)- \hat{\xi}(x, a_1)\} + \hat{\nu}(1, x, a_1)\hat{\xi}(x, a_1)\\
    &+ \gamma(x, a_1) \frac{1(A = a_1)}{\hat{\pi}_{a_1}(x)} \left\{ 1(\Delta = 1) - \hat{\xi}(x, a_1)\right\} + \gamma(x, a_1) \hat{\xi}(x, a_1)\\
    & - \frac{1(A = a_2, \Delta = 0)}{ \hat{\pi}_{a_2}(x)}\{\tilde{T} - \hat{\nu}(0, x, a_2)\} - \frac{\hat{\nu}(0, x, a_2)1(A = a_2)}{ \hat{\pi}_{a_2}(x)} \{1(\Delta = 0)- [1-\hat{\xi}(x, a_2)]\} - \hat{\nu}(0, x, a_2)[1-\hat{\xi}(x, a_2)]\\
    &- \frac{1(A = a_2, \Delta = 1)}{ \hat{\pi}_{a_2}(x)}\{\tilde{T} - \hat{\nu}(1, x, a_2)\} - \frac{\hat{\nu}(1, x, a_2)1(A = a_2)}{ \hat{\pi}_{a_2}(x)} \{1(\Delta = 1)- \hat{\xi}(x, a_2)\} - \hat{\nu}(1, x, a_2)\hat{\xi}(x, a_2).
\end{aligned}
\end{equation}}}
The pseudo-outcome of the lower bound is given by
{\footnotesize{\begin{equation}\label{eq:lower_gamma_bound_cate_learner}
\begin{aligned}
    &\hat{\phi}^{-}_{a_1, a_2}(x)\\
    =& \hat{\phi}^{-}(x, a_1) - \hat{\phi}^{+}(x, a_2) \\
    =& \frac{1(A = a_1, \Delta = 0)}{ \hat{\pi}_{a_1}(x)}\{\tilde{T} - \hat{\nu}(0, x, a_1)\} + \frac{\hat{\nu}(0, x, a_1)1(A = a_1)}{ \hat{\pi}_{a_1}(x)} \{1(\Delta = 0)- [1-\hat{\xi}(x, a_1)]\} + \hat{\nu}(0, x, a_1)[1-\hat{\xi}(x, a_1)]\\
    &+ \frac{1(A = a_1, \Delta = 1)}{ \hat{\pi}_{a_1}(x)}\{\tilde{T} - \hat{\nu}(1, x, a_1)\} + \frac{\hat{\nu}(1, x, a_1)1(A = a_1)}{ \hat{\pi}_{a_1}(x)} \{1(\Delta = 1)- \hat{\xi}(x, a_1)\} + \hat{\nu}(1, x, a_1)\hat{\xi}(x, a_1)\\
    &- \frac{1(A = a_2, \Delta = 0)}{\hat{\pi}_{a_2}(x)}\{\tilde{T} - \hat{\nu}(0, x, a_2)\} 
    - \frac{\hat{\nu}(0, x, a_2)1(A = a_2)}{\hat{\pi}_{a_2}(x)} \{1(\Delta = 0)- [1-\hat{\xi}(x, a_2)]\} 
    - \hat{\nu}(0, x, a_2)[1-\hat{\xi}(x, a_2)]\\
    &- \frac{1(A = a_2, \Delta = 1)}{\hat{\pi}_{a_2}(x)}\{\tilde{T} - \hat{\nu}(1, x, a_2)\} 
    - \frac{\hat{\nu}(1, x, a_2)1(A = a_2)}{\hat{\pi}_{a_2}(x)} \{1(\Delta = 1)- \hat{\xi}(x, a_2)\} 
    - \hat{\nu}(1, x, a_2)\hat{\xi}(x, a_2)\\
    &- \gamma(x, a_2) \frac{1(A = a_2)}{\hat{\pi}_{a_2}(x)} \left\{ 1(\Delta = 1) - \hat{\xi}(x, a_2)\right\} 
    - \gamma(x, a_2) \hat{\xi}(x, a_2).\\
\end{aligned}
\end{equation}}}

$\bullet$ \textbf{Case} \circledgreen{2} (\emph{conservative upper bounds).} The conservative upper bound is a special case of $\gamma(x, a) = t_{\mathrm{max}} - \mathbb{E} [\tilde{T} \mid \Delta = 1, X = x, A=a]$. Here, the corresponding pseudo-outcome is given by
{\footnotesize {\begin{equation}\label{eq:upper_non_informative_bound_cate_learner}
\begin{aligned}
    &\hat{\phi}^{+}_{a_1, a_2}(x)\\
    =& \hat{\phi}^{+}(x, a_1) - \hat{\phi}^{-}(x, a_2) \\
    =& \frac{1(A = a_1, \Delta = 0)}{\hat{\pi}_{a_1}(x)}\{\tilde{T} - \hat{\nu}(0, x, a_1)\} + \frac{\hat{\nu}(0, x, a_1)1(A = a_1)}{\hat{\pi}_{a_1}(x)} \{1(\Delta = 0)- [1-\hat{\xi}(x, a_1)]\} + \hat{\nu}(0, x, a_1)[1-\hat{\xi}(x, a_1)]\\
    &+ t_{\mathrm{max}} \frac{1(A = a_1)}{\hat{\pi}_{a_1}(x)} \left\{ 1(\Delta = 1) - \hat{\xi}(x, a_1)\right\} + t_{\mathrm{max}} \hat{\xi}(x, a_1)\\
    & - \frac{1(A = a_2, \Delta = 0)}{ \hat{\pi}_{a_2}(x)}\{\tilde{T} - \hat{\nu}(0, x, a_2)\} - \frac{\hat{\nu}(0, x, a_2)1(A = a_2)}{ \hat{\pi}_{a_2}(x)} \{1(\Delta = 0)- [1-\hat{\xi}(x, a_2)]\} - \hat{\nu}(0, x, a_2)[1-\hat{\xi}(x, a_2)]\\
    &- \frac{1(A = a_2, \Delta = 1)}{ \hat{\pi}_{a_2}(x)}\{\tilde{T} - \hat{\nu}(1, x, a_2)\} - \frac{\hat{\nu}(1, x, a_2)1(A = a_2)}{ \hat{\pi}_{a_2}(x)} \{1(\Delta = 1)- \hat{\xi}(x, a_2)\} - \hat{\nu}(1, x, a_2)\hat{\xi}(x, a_2).
\end{aligned}
\end{equation}}}
The pseudo-outcome of the lower bound is given by
{\footnotesize{\begin{equation}\label{eq:lower_non_informative_bound_cate_learner}
\begin{aligned}
    &\hat{\phi}^{-}_{a_1, a_2}(x)\\
    =& \hat{\phi}^{-}(x, a_1) - \hat{\phi}^{+}(x, a_2) \\
    =& \frac{1(A = a_1, \Delta = 0)}{ \hat{\pi}_{a_1}(x)}\{\tilde{T} - \hat{\nu}(0, x, a_1)\} + \frac{\hat{\nu}(0, x, a_1)1(A = a_1)}{ \hat{\pi}_{a_1}(x)} \{1(\Delta = 0)- [1-\hat{\xi}(x, a_1)]\} + \hat{\nu}(0, x, a_1)[1-\hat{\xi}(x, a_1)]\\
    &+ \frac{1(A = a_1, \Delta = 1)}{ \hat{\pi}_{a_1}(x)}\{\tilde{T} - \hat{\nu}(1, x, a_1)\} + \frac{\hat{\nu}(1, x, a_1)1(A = a_1)}{ \hat{\pi}_{a_1}(x)} \{1(\Delta = 1)- \hat{\xi}(x, a_1)\} + \hat{\nu}(1, x, a)\hat{\xi}(x, a_1)\\
    &- \frac{1(A = a_2, \Delta = 0)}{\hat{\pi}_{a_2}(x)}\{\tilde{T} - \hat{\nu}(0, x, a_2)\} 
    - \frac{\hat{\nu}(0, x, a_2)1(A = a_2)}{\hat{\pi}_{a_2}(x)} \{1(\Delta = 0)- [1-\hat{\xi}(x, a_2)]\} 
    - \hat{\nu}(0, x, a_2)[1-\hat{\xi}(x, a_2)]\\
    &- t_{\mathrm{max}} \frac{1(A = a_2)}{\hat{\pi}_{a_2}(x)} \left\{ 1(\Delta = 1) - \hat{\xi}(x, a_2)\right\} 
    - t_{\mathrm{max}} \hat{\xi}(x, a_2).\\
\end{aligned}
\end{equation}}}

\begin{algorithm}[htbp]
\caption{\method}
\label{alg:pseudocode}
\begin{algorithmic}[1]
\Input Observational data $\mathcal{D} = \{(x_i, a_i, \tilde{t}_i, \delta_i)_{i=1}^n\}$, domain knowledge $\gamma(x, a)$
\Ensure Estimated bounds $\hat{\mu}^{-}(x, a)$, $\hat{\mu}^{+}(x, a)$
\State \textbf{/* Stage 1: Nuisance estimation with cross-fitting */}
\State $\hat{\nu}(\delta, x, a) \gets \hat{\mathbb{E}}[\tilde{T} \mid \Delta = \delta, X = x, A = a]$
\State $\hat{\xi}(x, a) \gets \hat{\mathbb{P}}[\Delta = 1 \mid X = x, A = a]$
\State $\hat{\pi}_a(x) \gets \hat{\mathbb{P}}[A = a \mid X = x]$
\State $\hat{\phi}^{\pm}(x, a) \gets$ Eq.~(\ref{eq:lower_bound_learner}), (\ref{eq:upper_gamma_bound_learner}), (\ref{eq:upper_non_informative_bound_learner}) based on $\{\hat{\pi}_a, \hat{\xi}, \hat{\nu}\}$
\State \textbf{/* Stage 2: Pseudo-outcome regression */}
\State $\hat{\mu}^{-}(x, a) \gets \hat{\mathbb{E}}[\hat{\phi}^{-} \mid X = x]$
\State $\hat{\mu}^{+}(x, a) \gets \hat{\mathbb{E}}[\hat{\phi}^{+} \mid X = x]$
\State \Return $\hat{\mu}^{-}(x, a),\ \hat{\mu}^{+}(x, a)$
\end{algorithmic}
\end{algorithm}

\newpage
\section{Proofs}
\subsection{Proof of Lemma~\ref{lem:total_prob_causal}}
\label{app:proof_lemma}
\begin{proof}
The result follows from the law of total expectation~\citep{barnard.1949} and the identifiability assumptions, i.e., 
\begin{equation}
\begin{aligned}
&\E[ T(a)\mid X = x]\\
=& \E[ T \mid X = x, A = a] \\
=& \E[ T \mid X = x, A = a, \Delta = 0] P(\Delta = 0 \mid X = x, A = a)\\
&+ \E[ T \mid X = x, A = a, \Delta = 1] P(\Delta = 1 \mid X = x, A = a).
\end{aligned}
\end{equation}
Then, we have
\begin{equation}
\begin{aligned}
&\E [ T(a_1) - T(a_2) \mid X = x]\\
=& \E[ T(a_1)\mid X = x] - \E [ T(a_2)\mid X = x]\\
=& \E[ T\mid X = x, A = a_1] - \E [ T\mid X = x, A = a_2]\\
=& \nu(0, x, a_1)[1-\xi(x, a_1)] + \bluetext{\E[ T \mid X = x, A = a_1, \Delta = 1]}\xi(x, a_1)\\
&- \nu(0, x, a_2)[1-\xi(x, a_2)] - \bluetext{\E[ T \mid X = x, A = a_2, \Delta = 1]}\xi(x, a_2).
\end{aligned}
\end{equation}

\end{proof}

\newpage
\subsection{Proof of Theorem~\ref{thm:bounds}}
\label{app:proof_bounds}
\begin{proof}
Note that  
\begin{equation}
\tau_{a_1, a_2}(x) = \mu(x, a_1) - \mu(x, a_2).
\end{equation}
Further, we have that 
\begin{equation}
\mu(x, a) = \nu(0, x, a)[1-\xi(x, a)] + \E[T \mid X = x, A =a, \Delta= 1]\xi(x, a).
\end{equation}
while $ \E[\tilde{T} \mid \Delta = 1, X = x, A = a] \leq \E[T \mid \Delta = 1, X = x, A = a] \leq \E[\tilde{T} \mid \Delta = 1, X = x, A = a] + \gamma(x, a)$ by the definition of $\gamma(x, a)$.
Therefore, by the definitions of $\mu^{-}(x, a)$ and $\mu^{+}(x, a)$, we have 
\begin{equation}
\mu^{-}(x, a) \leq \mu(x, a) \leq \mu^{+}(x, a).
\end{equation}
Hence, taking the minimum and maximum of $\mu(x, a_1)$ and $\mu(x, a_2)$ yields the result:
\begin{equation}
\mu^{-}(x, a_1) - \mu^{+}(x, a_2) \leq \tau_{a_1, a_2}(x) \leq \mu^{+}(x, a_1) - \mu^{-}(x, a_2).
\end{equation}

\end{proof}

\newpage
\subsection{Proof of Proposition~\ref{prop:width_of_bounds}}
\label{app:proof_width}

\begin{proof}
First, we recall the definition of the domain knowledge upper bounds for two CAPOs with different treatments $a_1$ and $a_2$: 
\begin{equation}
\begin{aligned}
\mu^{-}(x, a_1) =& \nu(0,x, a_1)(1-\xi(x, a_1))+\nu(1,x, a_1)\xi(x, a_1),\\
\mu^{+}(x, a_1) =& \nu(0,x, a_1)(1-\xi(x, a_1))+\nu(1,x, a_1)\xi(x, a_1) + \gamma(x, a_1)\xi(x, a_1),\\
\mu^{-}(x, a_2) =& \nu(0,x, a_2)(1-\xi(x, a_2))+\nu(1,x, a_2)\xi(x, a_2),\\
\mu^{+}(x, a_2) =& \nu(0,x, a_2)(1-\xi(x, a_2))+\nu(1,x, a_2)\xi(x, a_2) + \gamma(x, a_2)\xi(x, a_2).\\
\end{aligned}
\end{equation}
Then, we could see that for treatment $a \in\gA$, the width of CAPO bounds is
\begin{equation}
\begin{aligned}
&\mu^{+}(x, a) - \mu^{-}(x, a) =\gamma(x, a)\xi(x, a) \\
\end{aligned}
\end{equation}

Then, we follow Theorem~\ref{thm:bounds} to derive \circledgreen{1} domain knowledge bounds of CATE $\tau_{a_1, a_2}(x)$ via
\begin{equation}
\begin{aligned}
&\tau_{a_1, a_2}^{-}(x)\\
=&\mu^{-}(x, a_1) - \mu^{+}(x, a_2)\\
=& \nu(0,x, a_1)(1-\xi(x, a_1))+\nu(1,x, a_1)\xi(x, a_1) - \nu(0,x, a_2)(1-\xi(x, a_2)) -\nu(1,x, a_2)\xi(x, a_2)\\
&- \gamma(x, a_2)\xi(x, a_2),\\
&\tau_{a_1, a_2}^{+}(x)\\
=&\mu^{+}(x, a_1) - \mu^{-}(x, a_2)\\
=& \nu(0,x, a_1)(1-\xi(x, a_1))+\nu(1,x, a_1)\xi(x, a_1) - \nu(0,x, a_2)(1-\xi(x, a_2)) - \nu(1,x, a_2)\xi(x, a_2)\\
&+ \gamma(x, a_1)\xi(x, a_1).\\
\end{aligned}
\end{equation}

By the definition of $\gamma(x, a)$ in Eq.~(\ref{eq:range_of_gamma}) and Theorem~\ref{thm:bounds}, the true CATE is guaranteed to lie within these bounds, i.e., $\tau_{a_1,a_2}^{-}(x) \leq \tau_{a_1,a_2}(x) \leq \tau_{a_1,a_2}^{+}(x)$. The width of the bounds is then
\begin{equation}
\tau_{a_1, a_2}^{+}(x) - \tau_{a_1, a_2}^{-}(x) = \gamma(x, a_1)\xi(x, a_1) + \gamma(x, a_2)\xi(x, a_2).
\end{equation}

Here, we yield for \circledgreen{2} conservative bounds of CATE $\tau_{a_1, a_2}(x)$ via
\begin{equation}
\begin{aligned}
\mu^{-}(x, a_1) =& \nu(0,x, a_1)(1-\xi(x, a_1))+\nu(1,x, a_1)\xi(x, a_1),\\
\mu^{+}(x, a_1) =& \nu(0,x, a_1)(1-\xi(x, a_1))+ t_{\mathrm{max}}\xi(x, a_1)\\
\mu^{-}(x, a_2) =& \nu(0,x, a_2)(1-\xi(x, a_2))+\nu(1,x, a_2)\xi(x, a_2),\\
\mu^{+}(x, a_2) =& \nu(0,x, a_2)(1-\xi(x, a_2))+ t_{\mathrm{max}}\xi(x, a_2) . \\
\end{aligned}
\end{equation}
Next, we follow Theorem~\ref{thm:bounds} to derive the domain knowledge bounds of the CATE $\tau_{a_1, a_2}(x)$. We obtain
\begin{equation}
\begin{aligned}
&\tau_{a_1, a_2}^{-}(x)\\
=&\mu^{-}(x, a_1) - \mu^{+}(x, a_2)\\
=& \nu(0,x, a_1)(1-\xi(x, a_1))+\nu(1,x, a_1)\xi(x, a_1) - \nu(0,x, a_2)(1-\xi(x, a_2)) - t_{\mathrm{max}}\xi(x, a_2),\\
&\tau_{a_1, a_2}^{+}(x)\\
=&\mu^{+}(x, a_1) - \mu^{-}(x, a_2)\\
=& \nu(0,x, a_1)(1-\xi(x, a_1)) + t_{\mathrm{max}}\xi(x, a_1) - \nu(0,x, a_2)(1-\xi(x, a_2)) - \nu(1,x, a_2)\xi(x, a_2) 
\end{aligned}
\end{equation}
By the definition of $\gamma(x, a) = t_{\mathrm{max}} - \nu(1, x, a)$ in Eq.~(\ref{eq:range_of_gamma}) and Theorem~\ref{thm:bounds}, the true CATE is guaranteed to lie within these bounds, i.e., $\tau_{a_1,a_2}^{-}(x) \leq \tau_{a_1,a_2}(x) \leq \tau_{a_1,a_2}^{+}(x)$. The width of the bounds is then
\begin{equation}
\tau_{a_1, a_2}^{+}(x) - \tau_{a_1, a_2}^{-}(x) = [t_{\mathrm{max}} - \nu(1,x,a_1)]\xi(x,a_1) + [t_{\mathrm{max}} - \nu(1,x,a_2)]\xi(x,a_2).
\end{equation}
\end{proof}

\newpage
\subsection{Derivation of meta-learners}
\label{app:derivation_estimator}
First, we derive the efficient influence function (EIF) score for the \textbf{lower bound}:
{\small{
\begin{equation}
\begin{aligned}
& \mathbb{EIF} (\psi(\mathbb{P}))\\
=&  \mathbb{EIF} \left( \int_x \E \left[\nu(0, x, a)\left(1-\xi(x, a)\right) + \nu(1, x, a) \xi(x, a)\right] p(x) \d x \right)\\
=& \frac{1(A = a, \Delta = 0)}{\hat{\pi}_a(x)}\{\tilde{T} - \hat{\nu}(0, x, a)\} + \frac{\hat{\nu}(0, x, a)1(A = a)}{\hat{\pi}_a(x)} \{1(\Delta = 0)- [1-\hat{\xi}(x, a)]\} + \hat{\nu}(0, x, a)[1-\hat{\xi}(x, a)]\\
&+ \frac{1(A = a, \Delta = 1)}{\hat{\pi}_a(x)}\{\tilde{T} - \hat{\nu}(1, x, a)\} + \frac{\hat{\nu}(1, x, a)1(A = a)}{\hat{\pi}_a(x)} \{1(\Delta = 1)- \hat{\xi}(x, a)\} + \hat{\nu}(1, x, a)\hat{\xi}(x, a) -\psi(a).
\end{aligned}
\end{equation}
}}
\textbf{Upper bound:} For the upper bound, we distinguish the two cases from the main paper as follows.

\textbf{Case} \circledgreen{1}:( \emph{domain-knowledge upper bounds}). We treat the domain function $\gamma(x, a)$ as a constant function and thus yield
\begin{equation}
\begin{aligned}
&  \mathbb{EIF} (\psi(\mathbb{P})) \\
=&  \mathbb{EIF} \left( \int_x \rmE \left[\nu(0, x, a)\left(1-\xi(x, a)\right) + \nu(1, x, a) \xi(x, a) + \gamma(x,a) \xi(x, a)\right] p(x) \d x \right) \\
=&\frac{1(A = a, \Delta = 0)}{\hat{\pi}_a(x)}\{\tilde{T} - \hat{\nu}(0, x, a)\} + \frac{\hat{\nu}(0, x, a)1(A = a)}{\hat{\pi}_a(x)} \{1(\Delta = 0)- [1-\hat{\xi}(x, a)]\} + \hat{\nu}(0, x, a)[1-\hat{\xi}(x, a)]\\
&+ \frac{1(A = a, \Delta = 1)}{\hat{\pi}_a(x)}\{\tilde{T} - \hat{\nu}(1, x, a)\} + \frac{\hat{\nu}(1, x, a)1(A = a)}{\hat{\pi}_a(x)} \{1(\Delta = 1)- \hat{\xi}(x, a)\} + \hat{\nu}(1, x, a)\hat{\xi}(x, a)\\
&+ \gamma(x, a) \frac{1(A = a)}{\hat{\pi}_a(x)} \left\{ 1(\Delta = 1) - \xi(X, a)\right\} + \gamma(x, a) \xi(X, a) -\psi(a).
\end{aligned}
\end{equation}

\textbf{Case} \circledgreen{2}:( \emph{conservative upper bounds}). We set $\gamma(x, a)$ as our known sensitivity function. We yield

\begin{equation}
\begin{aligned}
& \mathbb{EIF} (\psi(\mathbb{P}))\\
=&  \mathbb{EIF} \left( \int_x \nu(0, x, a)\left(1-\xi(x, a)\right) + t_{\mathrm{max}} \xi(x, a) p(x) \d x \right) \\
=& \int_x  \mathbb{EIF} \left[ \nu(0, x, a)\left(1-\xi(x, a)\right) + t_{\mathrm{max}} \xi(x, a) p(x) \d x \right]\\
=& \int_x   \mathbb{EIF} \left[\nu(0, x, a)\left(1-\xi(x, a)\right) p(x)\right]+  \mathbb{EIF} \left[t_{\mathrm{max}} \xi(x, a) p(x) \right] \d x \\
=& \int_x   \mathbb{EIF} \left[\nu(0, x, a)\left(1-\xi(x, a)\right) p(x)\right]+  \mathbb{EIF} \left[t_{\mathrm{max}} \xi(x, a) p(x) \right] \d x\\
=& \underbrace{\int_x   \mathbb{EIF} \left[\nu(0, x, a)\left(1-\xi(x, a)\right) p(x)\right] \d x}_{\ast} + \underbrace{\int_x  \mathbb{EIF} \left[t_{\mathrm{max}} \xi(x, a) p(x) \right] \d x}_{\ast\ast}.\\
\end{aligned}
\end{equation}
We proceed with the derivation by separating the above expression into two parts: 
\begin{equation}
\begin{aligned}
&\mathrm{(\ast)}\\
=& \int_x { \mathbb{EIF} \left[\nu(0, x, a)\right] }\left[1-\xi(x, a)\right] p(x) + \nu(0, x, a)  { \mathbb{EIF} \left[1-\xi(x, a)\right]} p(x) + \nu(0, x, a) \left[1-\xi(x, a)\right]  { \mathbb{EIF} \left[p(x)\right]} \d x \\
=& \int_x   {\frac{1\{X = x, A = a, \Delta = 0\}}{p(x, a) [1-\xi(x, a)]}[\tilde{T} -\nu(0, x, a)]}  \left[1-\xi(x, a)\right] p(x)\\
&+ \nu(0, x, a) {\frac{1\{X = x, A = a\}}{p(x, a)} [1(\Delta = 0) - 1 + \xi(x, a)]} p(x)\\
&+ \nu(0, x, a) \left[1-\xi(x, a)\right] {\left[1\{X = x\} - p(x)\right]} \d x\\
=& \int_x \frac{1(X = x, A = a, \Delta = 0)}{p(a \mid x)} \left[\tilde{T} - \nu(0, x, a)\right]
+ \frac{1(X=x, A=a)}{p(a \mid x)} \nu(0, x, a) \left[1(\Delta = 0)- 1 + \xi(x, a)\right] \\
&+ \nu(0, x, a)[1-\xi(x, a)]\left[1\{X = x\} - p(x)\right] \d x \\
=& \frac{1(A = a, \Delta = 0)}{p(a \mid X)} \left[\tilde{T} - \nu(0, x, a)\right]
+ \frac{1(A = a)}{p(a\mid X)} \nu(0, X, a)\{1(\Delta = 0)- [1-\xi(X, a)]\}\\
&+ \nu(0, X, a)[1-\xi(X, a)]
- \psi_{\ast}(a),\\
\end{aligned}
\end{equation}
and 
\begin{equation}
\begin{aligned}
&(\ast\ast)\\
=& \int_x  \mathbb{EIF} \left[t_{\mathrm{max}} \xi(x, a) p(x)\right] \d x\\
=& t_{\mathrm{max}} \int_x  \mathbb{EIF} \left[\xi(x, a) p(x)\right] \d x\\
=& t_{\mathrm{max}} \int_x  { \mathbb{EIF} \left[\xi(x, a) \right]} p(x) + \xi(x, a)  { \mathbb{EIF} \left[ p(x)\right]} \d x \\
=& t_{\mathrm{max}} \int_x \frac{1(X = x, A = a)}{p(X = x, A = a)} \left\{ 1(\Delta = 1) - \xi(x, a)\right\} p(x) + \xi(x, a) \left[1\{X = x\} - p(x)\right] \d x\\
=& t_{\mathrm{max}} \frac{1(A = a)}{p(a\mid X )} \left\{ 1(\Delta = 1) - \xi(X, a)\right\} + t_{\mathrm{max}} \xi(X, a) - \psi_{\ast\ast}(a).
\end{aligned}
\end{equation}

Finally, our estimator for the conservative upper bound is 
\begin{equation}
\begin{aligned}
&\hat{\phi}^{+}(x, a)\\
=& \frac{1(A = a, \Delta = 0)}{\hat{\pi}_a(x)}\{\tilde{T} - \hat{\nu}(0,  x, a)\} + \frac{\hat{\nu}(0,  x, a)1(A = a)}{\hat{\pi}_a(x)} \{1(\Delta = 0)- [1-\hat{\xi}( x, a)]\}\\
&+ \hat{\nu}(0,  x, a)[1-\hat{\xi}( x, a)] + t_{\mathrm{max}} \frac{1(A = a)}{\hat{\pi}_a(x)} \left\{ 1(\Delta = 1) - \hat{\xi}( x, a)\right\} + t_{\mathrm{max}} \hat{\xi}(x, a).
\end{aligned}
\end{equation}

\newpage
\subsection{Proof of Theorem~\ref{thm:DR-property}}
\label{app:proof_DR_property}
\begin{proof}
We proceed by calculating $\E\left[\hat{\phi}^{+} \mid X = x\right]$ and $\E\left[\hat{\phi}^{-} \mid X = x\right]$ for each pseudo-outcome $\hat{\phi}^{+}$ and $\hat{\phi}^{-}$,  which corresponds to an oracle second stage regression. 

We start with \textbf{the lower bounds}, which use pseudo-outcomes defined in Eq.~(\ref{eq:lower_bound_learner}). We can then write the equation as
\begin{equation}
\begin{aligned}
    &\hat{\phi}^{-}(x, a)\\
    =& \frac{1(A = a, \Delta = 0)}{\hat{\pi}_a(x)}\{\tilde{T} - \hat{\nu}(0, x, a)\} + \frac{\hat{\nu}(0, x, a)1(A = a)}{\hat{\pi}_a(x)} \{1(\Delta = 0)- [1-\hat{\xi}(x, a)]\} \\
    &+ \frac{1(A = a, \Delta = 1)}{\hat{\pi}_a(x)}\{\tilde{T} - \hat{\nu}(1, x, a)\} + \frac{\hat{\nu}(1, x, a)1(A = a)}{\hat{\pi}_a(x)} \{1(\Delta = 1)- \hat{\xi}(x, a)\}\\
    &+ \hat{\nu}(0, x, a)[1-\hat{\xi}(x, a)] + \hat{\nu}(1, x, a)\hat{\xi}(x, a)\\
    =& \frac{1(A = a, \Delta = 0)}{\hat{\pi}_a(x)}\tilde{T} - \frac{1(A = a)}{\hat{\pi}_a(x)}\hat{\nu}(0, x, a)[1-\hat{\xi}(x, a)]+ \hat{\nu}(0, x, a)[1-\hat{\xi}(x, a)]\\
    &+ \frac{1(A = a, \Delta = 1)}{\hat{\pi}_a(x)}\tilde{T} - \frac{1(A = a)}{\hat{\pi}_a(x)}\hat{\nu}(1, x, a)\hat{\xi}(x, a) + \hat{\nu}(1, x, a)\hat{\xi}(x, a).
\end{aligned}
\end{equation}

Hence, by calculating the conditional expectation, we obtain

{\small{\begin{equation}
\begin{aligned}
    &\E \left[ \hat{\phi}^{-}(x, a) \mid X =x\right]\\
    =& \E \left[\frac{1(A = a, \Delta = 0)}{\hat{\pi}_a(x)}\tilde{T} \mid X =x\right]
    - \E \left[\frac{1(A = a)}{\hat{\pi}_a(x)}\hat{\nu}(0, x, a)[1-\hat{\xi}(x, a)]\mid X =x\right]\\
    &+\E \left[\hat{\nu}(0, x, a)[1-\hat{\xi}(x, a)] \mid X =x\right]
    + \E \left[\frac{1(A = a, \Delta = 1)}{\hat{\pi}_a(x)}\tilde{T} \mid X =x\right]\\
    &- \E \left[\frac{1(A = a)}{\hat{\pi}_a(x)}\hat{\nu}(1, x, a)\hat{\xi}(x, a)\mid X =x\right]
    +\E \left[\hat{\nu}(1, x, a)\hat{\xi}(x, a) \mid X =x\right]\\
    =&\frac{\pi_a(x)}{\hat{\pi}_a(x)} [1-\xi(x, a)] \nu(0, x, a) - \frac{\pi_a(x)}{\hat{\pi}_a(x)} [1-\hat{\xi}(x, a)]\hat{\nu}(0, x, a) +\hat{\nu}(0, x, a)[1-\hat{\xi}(x, a)] \\
    &+\frac{\pi_a(x)}{\hat{\pi}_a(x)} \xi(x, a) \nu(1, x, a) - \frac{\pi_a(x)}{\hat{\pi}_a(x)} \hat{\xi}(x, a)\hat{\nu}(1, x, a) +\hat{\nu}(1, x, a)\hat{\xi}(x, a) \\
\end{aligned}
\end{equation}}}

Using $\hat{\nu}(\delta, x, a) = \nu(\delta, x, a)$, $\hat{\xi}(x, a) = \xi(x, a)$ and $\hat{\pi}_a(x) = \pi_a(x)$ implies
\begin{equation}
\begin{aligned}
&\E \left[ \hat{\phi}^{-}(x, a) \mid X =x\right]\\
=&\frac{\pi_a(x)}{\hat{\pi}_a(x)} [1-\xi(x, a)] \nu(0, x, a) - \frac{\pi_a(x)}{\hat{\pi}_a(x)} [1-\hat{\xi}(x, a)]\hat{\nu}(0, x, a) +\hat{\nu}(0, x, a)[1-\hat{\xi}(x, a)] \\
&+\frac{\pi_a(x)}{\hat{\pi}_a(x)} \xi(x, a) \nu(1, x, a) - \frac{\pi_a(x)}{\hat{\pi}_a(x)} \hat{\xi}(x, a)\hat{\nu}(1, x, a) +\hat{\nu}(1, x, a)\hat{\xi}(x, a) \\
=&  [1-\xi(x, a)] \nu(0, x, a) -  [1-\xi(x, a)] \nu(0, x, a) +  [1-\xi(x, a)] \nu(0, x, a) \\
&+ \xi(x, a) \nu(1, x, a) - \xi(x, a) \nu(1, x, a) + \xi(x, a) \nu(1, x, a)\\
=& [1-\xi(x, a)] \nu(0, x, a) + \xi(x, a) \nu(1, x, a)\\
=& \mu^{-}(x, a),
\end{aligned}
\end{equation}
which proves consistency.

We distinguish two cases. (1) Under $\hat{\nu}(\delta, x, a) = \nu(\delta, x, a)$ and $\hat{\xi}(x, a) = \xi(x, a)$, this reduces to 
\begin{equation}
\begin{aligned}
&\E \left[ \hat{\phi}^{-}(x, a) \mid X =x\right]\\
=&\frac{\pi_a(x)}{\hat{\pi}_a(x)} [1-\xi(x, a)] \nu(0, x, a) 
- \frac{\pi_a(x)}{\hat{\pi}_a(x)} [1-{\xi}(x, a)]{\nu}(0, x, a) 
+{\nu}(0, x, a)[1-{\xi}(x, a)] \\
&+\frac{\pi_a(x)}{\hat{\pi}_a(x)} \xi(x, a) \nu(1, x, a) 
- \frac{\pi_a(x)}{\hat{\pi}_a(x)} {\xi}(x, a){\nu}(0, x, a) 
+{\nu}(1, x, a){\xi}(x, a) \\
=& [1-\xi(x, a)] \nu(0, x, a) + \xi(x, a) \nu(1, x, a)\\
=& \mu^{-}(x, a).
\end{aligned}
\end{equation}

(2) Under $\hat{\pi}_a(x) = \pi_a(x)$, this reduces to
\begin{equation}
\begin{aligned}
&\E \left[ \hat{\phi}^{-}(x, a) \mid X =x\right]\\
=&\frac{\pi_a(x)}{{\pi}_a(x)} [1-\xi(x, a)] \nu(0, x, a) 
- \frac{\pi_a(x)}{{\pi}_a(x)} [1-\hat{\xi}(x, a)]\hat{\nu}(0, x, a)
+\hat{\nu}(0, x, a)[1-\hat{\xi}(x, a)] \\
&+\frac{\pi_a(x)}{{\pi}_a(x)} \xi(x, a) \nu(1, x, a) 
- \frac{\pi_a(x)}{{\pi}_a(x)} \hat{\xi}(x, a)\hat{\nu}(0, x, a) 
+\hat{\nu}(1, x, a)\hat{\xi}(x, a) \\
=& [1-\xi(x, a)] \nu(0, x, a) + \xi(x, a) \nu(1, x, a)\\
=& \mu^{-}(x, a).
\end{aligned}
\end{equation}

Together, this proves the double robustness.

Next, we move on to the upper bounds and prove them again in two cases. The pseudo-outcome of \textbf{Case} \circledgreen{1}(\emph{domain knowledge upper bounds}). defined in Eq.~(\ref{eq:upper_gamma_bound_learner}) can be simplified to 
\begin{equation}
\begin{aligned}
    &\hat{\phi}^{+}(x, a)\\
    =& \frac{1(A = a, \Delta = 0)}{\hat{\pi}_a(x)}\{\tilde{T} - \hat{\nu}(0, x, a)\} + \frac{\hat{\nu}(0, x, a)1(A = a)}{\hat{\pi}_a(x)} \{1(\Delta = 0)- [1-\hat{\xi}(x, a)]\} \\
    &+ \frac{1(A = a, \Delta = 1)}{\hat{\pi}_a(x)}\{\tilde{T} - \hat{\nu}(1, x, a)\} + \frac{\hat{\nu}(1, x, a)1(A = a)}{\hat{\pi}_a(x)} \{1(\Delta = 1)- \hat{\xi}(x, a)\}\\
    &+ \hat{\nu}(0, x, a)[1-\hat{\xi}(x, a)] + \hat{\nu}(1, x, a)\hat{\xi}(x, a)\\
    &+ \gamma(x, a) \frac{1(A = a)}{\hat{\pi}_a(x)} \left\{ 1(\Delta = 1) - \hat{\xi}(x, a)\right\} + \gamma(x, a) \hat{\xi}(x, a)\\
    =& \frac{1(A = a, \Delta = 0)}{\hat{\pi}_a(x)}\tilde{T} - \frac{1(A = a)}{\hat{\pi}_a(x)}\hat{\nu}(0, x, a)[1-\hat{\xi}(x, a)]+ \hat{\nu}(0, x, a)[1-\hat{\xi}(x, a)]\\
    &+ \frac{1(A = a, \Delta = 1)}{\hat{\pi}_a(x)}\tilde{T} - \frac{1(A = a)}{\hat{\pi}_a(x)}\hat{\nu}(1, x, a)\hat{\xi}(x, a) + \hat{\nu}(1, x, a)\hat{\xi}(x, a)\\
    &+ \gamma(x, a) \frac{1(A = a, \Delta = 1)}{\hat{\pi}_a(x)} - \gamma(x, a) \frac{1(A = a)}{\hat{\pi}_a(x)}\hat{\xi}(x, a) + \gamma(x, a) \hat{\xi}(x, a).\\
\end{aligned}
\end{equation}

The proof works analogously to the lower bound. That is, we calculate the conditional expectation over the pseudo outcome:
{\small{\begin{equation}
\begin{aligned}
    &\E \left[ \hat{\phi}^{+}(x, a) \mid X =x\right]\\
    =& \E \left[\frac{1(A = a, \Delta = 0)}{\hat{\pi}_a(x)}\tilde{T} \mid X =x\right]
    - \E \left[\frac{1(A = a)}{\hat{\pi}_a(x)}\hat{\nu}(0, x, a)[1-\hat{\xi}(x, a)]\mid X =x\right]\\
    &+\E \left[\hat{\nu}(0, x, a)[1-\hat{\xi}(x, a)] \mid X =x\right]
    + \E \left[\frac{1(A = a, \Delta = 1)}{\hat{\pi}_a(x)}\tilde{T} \mid X =x\right]\\
    &- \E \left[\frac{1(A = a)}{\hat{\pi}_a(x)}\hat{\nu}(1, x, a)\hat{\xi}(x, a)\mid X =x\right]
    +\E \left[\hat{\nu}(1, x, a)\hat{\xi}(x, a) \mid X =x\right]\\
    &+ \E \left[ \gamma(x, a) \frac{1(A = a, \Delta = 1)}{\hat{\pi}_a(x)} \mid X =x\right]
    - \E \left[ \gamma(x, a) \frac{1(A = a)}{\hat{\pi}_a(x)}\hat{\xi}(x, a) \mid X =x\right]\\
    &+ \E \left[ \gamma(x, a) \hat{\xi}(x, a) \mid X =x\right]\\
    =&\frac{\pi_a(x)}{\hat{\pi}_a(x)} [1-\xi(x, a)] \nu(0, x, a) - \frac{\pi_a(x)}{\hat{\pi}_a(x)} [1-\hat{\xi}(x, a)]\hat{\nu}(0, x, a) +\hat{\nu}(0, x, a)[1-\hat{\xi}(x, a)] \\
    &+\frac{\pi_a(x)}{\hat{\pi}_a(x)} \xi(x, a) \nu(1, x, a) - \frac{\pi_a(x)}{\hat{\pi}_a(x)} \hat{\xi}(x, a)\hat{\nu}(1, x, a) +\hat{\nu}(1, x, a)\hat{\xi}(x, a) \\
    &+ \gamma(x, a) \frac{\pi_a(x)}{\hat{\pi}_a(x)} \xi(x,a)
    - \gamma(x, a)  \frac{\pi_a(x)}{\hat{\pi}_a(x)} \hat{\xi}(x, a) 
    + \gamma(x, a) \hat{\xi}(x, a).
\end{aligned}
\end{equation}}}

Hence, $\hat{\nu}(\delta, x, a) = \nu(\delta, x, a)$, $\hat{\xi}(x, a) = \xi(x, a)$ and $\hat{\pi}_a(x) = \pi_a(x)$ implies
\begin{equation}
\begin{aligned}
    &\E \left[ \hat{\phi}^{+}(x, a) \mid X =x\right]\\
    =&\frac{\pi_a(x)}{\hat{\pi}_a(x)} [1-\xi(x, a)] \nu(0, x, a) - \frac{\pi_a(x)}{\hat{\pi}_a(x)} [1-\hat{\xi}(x, a)]\hat{\nu}(0, x, a) +\hat{\nu}(0, x, a)[1-\hat{\xi}(x, a)] \\
    &+\frac{\pi_a(x)}{\hat{\pi}_a(x)} \xi(x, a) \nu(1, x, a) - \frac{\pi_a(x)}{\hat{\pi}_a(x)} \hat{\xi}(x, a)\hat{\nu}(1, x, a) +\hat{\nu}(1, x, a)\hat{\xi}(x, a) \\
    &+ \gamma(x, a) \frac{\pi_a(x)}{\hat{\pi}_a(x)} \xi(x,a)
    - \gamma(x, a)  \frac{\pi_a(x)}{\hat{\pi}_a(x)} \hat{\xi}(x, a) 
    + \gamma(x, a) \hat{\xi}(x, a)\\
    =&  [1-\xi(x, a)] \nu(0, x, a) + \xi(x, a) \nu(1, x, a) + \gamma(x, a) \xi(x,a)\\
    =& {\mu}^{+}(x, a).
\end{aligned}
\end{equation}
which proves the consistency.

Also, (1) under $\hat{\nu}(\delta, x, a) = \nu(\delta, x, a)$ and $\hat{\xi}(x, a) = \xi(x, a)$, this reduces to 
\begin{equation}
\begin{aligned}
&\E \left[ \hat{\phi}^{+}(x, a) \mid X =x\right]\\
=&\frac{\pi_a(x)}{\hat{\pi}_a(x)} [1-\xi(x, a)] \nu(0, x, a) - \frac{\pi_a(x)}{\hat{\pi}_a(x)} [1-\hat{\xi}(x, a)]\hat{\nu}(0, x, a) +\hat{\nu}(0, x, a)[1-\hat{\xi}(x, a)] \\
&+\frac{\pi_a(x)}{\hat{\pi}_a(x)} \xi(x, a) \nu(1, x, a) - \frac{\pi_a(x)}{\hat{\pi}_a(x)} \hat{\xi}(x, a)\hat{\nu}(0, x, a) +\hat{\nu}(1, x, a)\hat{\xi}(x, a) \\
&+ \gamma(x, a) \frac{\pi_a(x)}{\hat{\pi}_a(x)} \xi(x,a)
- \gamma(x, a)  \frac{\pi_a(x)}{\hat{\pi}_a(x)} \hat{\xi}(x, a) 
+ \gamma(x, a) \hat{\xi}(x, a)\\
=&\frac{\pi_a(x)}{\hat{\pi}_a(x)} [1-\xi(x, a)] \nu(0, x, a) 
- \frac{\pi_a(x)}{\hat{\pi}_a(x)} [1-{\xi}(x, a)]{\nu}(0, x, a) 
+{\nu}(0, x, a)[1-{\xi}(x, a)] \\
&+\frac{\pi_a(x)}{\hat{\pi}_a(x)} \xi(x, a) \nu(1, x, a) 
- \frac{\pi_a(x)}{\hat{\pi}_a(x)} {\xi}(x, a){\nu}(0, x, a) 
+{\nu}(1, x, a){\xi}(x, a) \\
&+ \gamma(x, a) \frac{\pi_a(x)}{\hat{\pi}_a(x)} \xi(x,a)
- \gamma(x, a)  \frac{\pi_a(x)}{\hat{\pi}_a(x)} {\xi}(x, a) 
+ \gamma(x, a) {\xi}(x, a)\\
=&  [1-\xi(x, a)] \nu(0, x, a) + \xi(x, a) \nu(1, x, a) + \gamma(x, a) \xi(x,a)\\
=& {\mu}^{+}(x, a).
\end{aligned}
\end{equation}

(2) Under $\hat{\pi}_a(x) = \pi_a(x)$, this reduces to
\begin{equation}
\begin{aligned}
&\E \left[ \hat{\phi}^{+}(x, a) \mid X =x\right]\\
=&\frac{\pi_a(x)}{\hat{\pi}_a(x)} [1-\xi(x, a)] \nu(0, x, a) - \frac{\pi_a(x)}{\hat{\pi}_a(x)} [1-\hat{\xi}(x, a)]\hat{\nu}(0, x, a) +\hat{\nu}(0, x, a)[1-\hat{\xi}(x, a)] \\
&+\frac{\pi_a(x)}{\hat{\pi}_a(x)} \xi(x, a) \nu(1, x, a) - \frac{\pi_a(x)}{\hat{\pi}_a(x)} \hat{\xi}(x, a)\hat{\nu}(0, x, a) +\hat{\nu}(1, x, a)\hat{\xi}(x, a) \\
&+ \gamma(x, a) \frac{\pi_a(x)}{\hat{\pi}_a(x)} \xi(x,a)
- \gamma(x, a)  \frac{\pi_a(x)}{\hat{\pi}_a(x)} \hat{\xi}(x, a) 
+ \gamma(x, a) \hat{\xi}(x, a)\\
=&\frac{\pi_a(x)}{{\pi}_a(x)} [1-\xi(x, a)] \nu(0, x, a) 
- \frac{\pi_a(x)}{{\pi}_a(x)} [1-\hat{\xi}(x, a)]\hat{\nu}(0, x, a) 
+\hat{\nu}(0, x, a)[1-\hat{\xi}(x, a)] \\
&+\frac{\pi_a(x)}{{\pi}_a(x)} \xi(x, a) \nu(1, x, a) 
- \frac{\pi_a(x)}{{\pi}_a(x)} \hat{\xi}(x, a)\hat{\nu}(1, x, a) 
+\hat{\nu}(1, x, a)\hat{\xi}(x, a) \\
&+ \gamma(x, a) \frac{\pi_a(x)}{{\pi}_a(x)} \xi(x,a)
- \gamma(x, a)  \frac{\pi_a(x)}{{\pi}_a(x)} \hat{\xi}(x, a) 
+ \gamma(x, a) \hat{\xi}(x, a)\\
=&  [1-\xi(x, a)] \nu(0, x, a) + \xi(x, a) \nu(1, x, a) + \gamma(x, a) \xi(x,a)\\
=& {\mu}^{+}(x, a).
\end{aligned}
\end{equation}

Finally, the pseudo-outcomes of \textbf{Case} \circledgreen{2}(\emph{conservative upper bounds}). is defined in Eq.~(\ref{eq:upper_non_informative_bound_learner}). We simplify the equation to

\begin{equation}
\begin{aligned}
&\hat{\phi}^{+}(x, a)\\
=& \frac{1(A = a, \Delta = 0)}{\hat{\pi}_a(x)}\{\tilde{T} - \hat{\nu}(0, x, a)\} + \frac{\hat{\nu}(0, x, a)1(A = a)}{\hat{\pi}_a(x)} \{1(\Delta = 0)- [1-\hat{\xi}(x, a)]\}\\
&+ \hat{\nu}(0, x, a)[1-\hat{\xi}(x, a)] + t_{\mathrm{max}} \frac{1(A = a)}{\hat{\pi}_a(x)} \left\{ 1(\Delta =1) - \hat{\xi}(x, a)\right\} + t_{\mathrm{max}}\hat{\xi}(x, a)\\
=& \frac{1(A = a, \Delta = 0)}{\hat{\pi}_a(x)}\tilde{T} - \frac{1(A = a)}{\hat{\pi}_a(x)}\hat{\nu}(0, x, a)[1-\hat{\xi}(x, a)]+ \hat{\nu}(0, x, a)[1-\hat{\xi}(x, a)]\\
&+  t_{\mathrm{max}} \frac{1(A = a, \Delta =1)}{\hat{\pi}_a(x)} -  t_{\mathrm{max}} \frac{1(A = a)}{\hat{\pi}_a(x)} \hat{\xi}(x, a) + t_{\mathrm{max}}\hat{\xi}(x, a).
\end{aligned}
\end{equation}

Again, by taking expectation conditional on $X = x$, $A = a$, we obtain

\begin{equation}\label{eq:pseudo-outcome}
\begin{aligned}
&\E[\hat{\phi}^{+}(x, a)\mid X = x]\\
=& \frac{\pi_a(x)}{\hat{\pi}_a(x)} [1-\xi(x, a)] \nu(0, x, a) 
- \frac{\pi_a(x)}{\hat{\pi}_a(x)} [1-\hat{\xi}(x, a)]\hat{\nu}(0, x, a) 
+ \hat{\nu}(0, x, a)[1-\hat{\xi}(x, a)]\\
& + t_{\mathrm{max}} \frac{\pi_a(x)}{\hat{\pi}_a(x)} \xi(x,a)
- t_{\mathrm{max}}  \frac{\pi_a(x)}{\hat{\pi}_a(x)} \hat{\xi}(x, a) 
+ t_{\mathrm{max}} \hat{\xi}(x, a).\\
\end{aligned}
\end{equation}

Hence, using the fact that $\hat{\nu}(\delta, x, a) = \nu(\delta, x, a)$, $\hat{\xi}(x, a) = \xi(x, a)$ and $\hat{\pi}_a(x) = \pi_a(x)$ implies
\begin{equation}
\begin{aligned}
&\E\left[\hat{\phi}^{+}(x, a) \mid X = x\right]\\
=& [1-\xi(x, a)] \nu(0, x, a) -  [1-{\xi}(x, a)]{\nu}(0, x, a) +{\nu}(0, x, a)[1-{\xi}(x, a)]\\
& + t_{\mathrm{max}} \xi(x,a)
- t_{\mathrm{max}} {\xi}(x, a) 
+ t_{\mathrm{max}} {\xi}(x, a)\\
=& [1-\xi(x, a)] \nu(0, x, a) + t_{\mathrm{max}} {\xi}(x, a)\\
=&{\mu}^{+}(x, a),
\end{aligned}
\end{equation}
which proves consistency.

Again, (1) under $\hat{\nu}(\delta, x, a) = \nu(\delta, x, a)$ and $\hat{\xi}(x, a) = \xi(x, a)$, this reduces to 
\begin{equation}
\begin{aligned}
&\E[\hat{\phi}^{+}(x, a)\mid X = x]\\
=& \frac{\pi_a(x)}{\hat{\pi}_a(x)} [1-\xi(x, a)] \nu(0, x, a) 
- \frac{\pi_a(x)}{\hat{\pi}_a(x)} [1-\hat{\xi}(x, a)]\hat{\nu}(0, x, a) 
+ \hat{\nu}(0, x, a)[1-\hat{\xi}(x, a)]\\
& + t_{\mathrm{max}} \frac{\pi_a(x)}{\hat{\pi}_a(x)} \xi(x,a)
- t_{\mathrm{max}}  \frac{\pi_a(x)}{\hat{\pi}_a(x)} \hat{\xi}(x, a) 
+ t_{\mathrm{max}} \hat{\xi}(x, a)\\
=& \frac{\pi_a(x)}{\hat{\pi}_a(x)} [1-\xi(x, a)] \nu(0, x, a) 
- \frac{\pi_a(x)}{\hat{\pi}_a(x)} [1-{\xi}(x, a)]{\nu}(0, x, a) 
+ {\nu}(0, x, a)[1-{\xi}(x, a)]\\
& + t_{\mathrm{max}} \frac{\pi_a(x)}{\hat{\pi}_a(x)} \xi(x,a)
- t_{\mathrm{max}}  \frac{\pi_a(x)}{\hat{\pi}_a(x)} {\xi}(x, a) 
+ t_{\mathrm{max}} {\xi}(x, a)\\
=& \hat{\nu}(0, x, a)[1-\hat{\xi}(x, a)+ t_{\mathrm{max}} {\xi}(x, a)\\
=& {\mu}^{+}(x, a),
\end{aligned}
\end{equation}

(2) Under $\hat{\pi}_a(x) = \pi_a(x)$, this reduces to
\begin{equation}
\begin{aligned}
&\E[\hat{\phi}^{+}(x, a)\mid X = x]\\
=& \frac{\pi_a(x)}{\hat{\pi}_a(x)} [1-\xi(x, a)] \nu(0, x, a) 
- \frac{\pi_a(x)}{\hat{\pi}_a(x)} [1-\hat{\xi}(x, a)]\hat{\nu}(0, x, a) 
+ \hat{\nu}(0, x, a)[1-\hat{\xi}(x, a)]\\
& + t_{\mathrm{max}} \frac{\pi_a(x)}{\hat{\pi}_a(x)} \xi(x,a)
- t_{\mathrm{max}}  \frac{\pi_a(x)}{\hat{\pi}_a(x)} \hat{\xi}(x, a) 
+ t_{\mathrm{max}} \hat{\xi}(x, a)\\
=& \frac{\pi_a(x)}{{\pi}_a(x)} [1-\xi(x, a)] \nu(0, x, a) 
- \frac{\pi_a(x)}{{\pi}_a(x)} [1-\hat{\xi}(x, a)]\hat{\nu}(0, x, a) 
+ \hat{\nu}(0, x, a)[1-\hat{\xi}(x, a)]\\
& + t_{\mathrm{max}} \frac{\pi_a(x)}{{\pi}_a(x)} \xi(x,a)
- t_{\mathrm{max}}  \frac{\pi_a(x)}{{\pi}_a(x)} \hat{\xi}(x, a) 
+ t_{\mathrm{max}} \hat{\xi}(x, a)\\
=& {\nu}(0, x, a)[1-{\xi}(x, a)+ t_{\mathrm{max}} {\xi}(x, a)\\
=& {\mu}^{+}(x, a).
\end{aligned}
\end{equation}
This proves double robustness again. 
\end{proof}

\newpage
\subsection{Proof of Theorem~\ref{thm:oracle-efficiency}}
\label{app:proof_quasi-orcale_efficiency}

\begin{lemma}\label{lem:nuisance_function_error}
Consider the setting described in Theorem~\ref{thm:oracle-efficiency}. Then, we have
\begin{equation}
\begin{aligned}
&\E\left[ (\hat{\mu}^{+}(x, a) - \mu^{+}(x, a))^2 \mid X = x \right]\\
&\lesssim \mathcal{R}(x)
+ \E \left[ (\hat{\pi}_a(x) - \pi_a(x))^2 \right] \left(  \E \left[(\hat{\nu}(0, x, a) - \nu(0, x, a))^2\right]  + \E \left[(\hat{\nu}(1, x, a) - {\nu}(1, x, a))^2\right] \right.\\
& \left.+ \E \left[(\hat{\xi}(x, a) - {\xi}(x, a))^2\right] \right).\\
\end{aligned}
\end{equation}
\end{lemma}

\begin{proof}

Let $\mu^{+}(x, a)$ be the corresponding oracle to $\hat{\mu}^{+}(x, a)$. Further ,we define $\tilde{\mu}^{+}(x, a) = \hat{\E}_n[{\phi}^{+}(x, a)\mid X = x]$. Using the assumption, we can apply Proposition 1 of \citep{kennedy.2023} and obtain 
\begin{equation}
\begin{aligned}
\E\left[ (\hat{\mu}^{+}(x, a) - \mu^{+}(x, a))^2 \mid X = x \right]
\lesssim
\mathcal{R}(x,a) + \E \left[\hat{r}(x, a)^2 \right],
\end{aligned}
\end{equation}

where $\mathcal{R}(x,a) = \E \left[ ({\mu}^{+}(x, a) - \tilde{\mu}^{+}(x, a))^2 \right]$ is the oracle risk of the second stage regression. We can apply Eq.~(\ref{eq:pseudo-outcome}) in the proof of Theorem~\ref{thm:DR-property} Section~\ref{app:proof_DR_property} to obtain the following.

\textbf{Case}\circledgreen{1}: domain-knowledge upper bound.
\begin{equation}
\begin{aligned}
&\hat{r}(x, a)\\
=&\frac{\pi_a(x)}{\hat{\pi}_a(x)} [1-\xi(x, a)] \nu(0, x, a) - \frac{\pi_a(x)}{\hat{\pi}_a(x)} [1-\hat{\xi}(x, a)]\hat{\nu}(0, x, a) +\hat{\nu}(0, x, a)[1-\hat{\xi}(x, a)] \\
&+\frac{\pi_a(x)}{\hat{\pi}_a(x)} \xi(x, a) \nu(1, x, a) - \frac{\pi_a(x)}{\hat{\pi}_a(x)} \hat{\xi}(x, a)\hat{\nu}(0, x, a) +\hat{\nu}(1, x, a)\hat{\xi}(x, a) \\
&+ \gamma(x, a) \frac{\pi_a(x)}{\hat{\pi}_a(x)} \xi(x,a)
- \gamma(x, a)  \frac{\pi_a(x)}{\hat{\pi}_a(x)} \hat{\xi}(x, a) 
+ \gamma(x, a) \hat{\xi}(x, a) - \mu^{+}(x, a)\\
=& \left\{\frac{\pi_a(x)}{\hat{\pi}_a(x)} - 1 \right\}
\left\{\nu(0, x, a) \left[\hat{\xi}(x, a) - {\xi}(x, a)\right]+ [1-\hat\xi(x, a)] \left[\nu(0, x, a) -\hat{\nu}(0, x, a) \right]\right\}\\
&+\left\{\frac{\pi_a(x)}{\hat{\pi}_a(x)} - 1 \right\} 
\left\{\nu(1, x, a) \left[\xi(x, a) - \hat{\xi}(x, a)\right] +  \hat{\xi}(x, a) \left[\nu(1, x, a) - \hat{\nu}(1, x, a) \right] \right\}\\
&+\left\{\frac{\pi_a(x)}{\hat{\pi}_a(x)} - 1 \right\} \gamma(x, a) 
\left[ \xi(x,a) - \hat{\xi}(x, a)\right]\\
=&\left\{\frac{\pi_a(x)}{\hat{\pi}_a(x)} - 1 \right\} 
\left\{\left[\nu(1, x, a) - \gamma(x, a) - \nu(0, x, a)\right] \left[\hat\xi(x, a) - {\xi}(x, a)\right] \right.\\
&\left. + [1-\hat\xi(x, a)] \left[\nu(0, x, a) -\hat{\nu}(0, x, a) \right]
+ \hat{\xi}(x, a)[\hat{\nu}(1, x, a) - {\nu}(1, x, a)]\right\}.
\end{aligned}
\end{equation}

Applying the inequality $(a+b+c)^2 \leq 3(a^2+b^2+c^2)$ together with Assumption~\ref{assumption:boundedness} from and the fact that $\pi_a(x)\leq 1$ yields

{\small{\begin{equation}
\begin{aligned}
\hat{r}(x,a)^2
=& \left\{\frac{\pi_a(x)}{\hat{\pi}_a(x)} - 1 \right\}^2
\left\{\left[\nu(1, x, a) + \gamma(x, a) - \nu(0, x, a)\right] \left[\hat\xi(x, a) - {\xi}(x, a)\right]\right.\\
&\left. + [1-\hat\xi(x, a)] \left[\nu(0, x, a) -\hat{\nu}(0, x, a) \right]
+ \hat{\xi}(x, a)[\hat{\nu}(1, x, a) - {\nu}(1, x, a)]\right\}^2\\
\leq& 3
\left\{\frac{\pi_a(x)}{\hat{\pi}_a(x)} - 1 \right\}^2 
\left\{\left[\nu(1, x, a) + \gamma(x, a) - \nu(0, x, a)\right]^2 \left[\hat\xi(x, a) - {\xi}(x, a)\right]^2\right.\\
&\left.+ [1-\hat\xi(x, a)]^2 \left[\nu(0, x, a) -\hat{\nu}(0, x, a) \right]^2 
+ \hat{\xi}(x, a)^2[\hat{\nu}(1, x, a) - {\nu}(1, x, a)]^2\right\}\\
=& 3
\left\{\frac{\pi_a(x)}{\hat{\pi}_a(x)} - 1 \right\}^2 
\left[\nu(1, x, a) + \gamma(x, a) - \nu(0, x, a)\right]^2 \left[\hat\xi(x, a) - {\xi}(x, a)\right]^2\\
&+3
\left\{\frac{\pi_a(x)}{\hat{\pi}_a(x)} - 1 \right\}^2 
[1-\hat\xi(x, a)]^2 \left[\nu(0, x, a) -\hat{\nu}(0, x, a) \right]^2\\
&+3
\left\{\frac{\pi_a(x)}{\hat{\pi}_a(x)} - 1 \right\}^2 
\hat{\xi}(x, a)^2\left[\hat{\nu}(1, x, a) - {\nu}(1, x, a)\right]^2.
\end{aligned}
\end{equation}
}}

By using Assumption~\ref{assumption:boundedness}, which gives 
$\hat{\pi}_a(x) > \varepsilon$ and thus 
$\left\{\frac{\pi_a(x)}{\hat{\pi}_a(x)} - 1\right\}^2 
= \frac{(\pi_a(x) - \hat{\pi}_a(x))^2}{\hat{\pi}_a(x)^2} 
\leq \frac{1}{\varepsilon^2}(\pi_a(x) - \hat{\pi}_a(x))^2$, 
together with the boundedness of $\nu$, $\xi$, and $\gamma$, we obtain
\begin{equation}
\begin{aligned}
\hat{r}(x,a)^2
\leq& \frac{3(2C + t_{\mathrm{max}})^2}{\varepsilon^2}
\left[\hat{\pi}_a(x) - {\pi}_a(x)\right]^2
\left[\hat{\xi}(x, a) - {\xi}(x, a)\right]^2 \\
&+ \frac{3}{\varepsilon^2}\left[\hat{\pi}_a(x) - {\pi}_a(x)\right]^2
\left[\nu(0, x, a) - \hat{\nu}(0, x, a)\right]^2 \\
&+ \frac{3}{\varepsilon^2}\left[\hat{\pi}_a(x) - {\pi}_a(x)\right]^2
\left[\hat{\nu}(1, x, a) - {\nu}(1, x, a)\right]^2.
\end{aligned}
\end{equation}
Applying expectations on both sides yields the result, 
because $\hat{\pi}_a(x) \indep (\hat{\nu}(\delta, x, a), \hat{\xi}(x, a))$ due to sample splitting.

\end{proof}

\newpage
\section{Extensions}

\subsection{Setting with continuous treatments}
\label{app:continuous_setting}

In many applications, treatments are not discrete but continuous (e.g., dosage levels, duration of therapy). Below, we extend our \method to settings with continuous treatments where $A \in \gA \subseteq \mathbb{R}$. 

The key technical issue is the following. In the discrete case, we use the indicator $\mathbbm{1}(A = a)$ weighted by the propensity score $\pi_a(x)$. For continuous treatments, the treatment variable $A$ has a density with respect to the Lebesgue measure, so the point-mass indicator $\mathbbm{1}(A = a)$ has probability zero and dividing by the generalized propensity score $f_{A \mid X}(a \mid x)$ no longer yields a well-behaved influence function. Our solution is to replace $\mathbbm{1}(A = a)$ with a smooth, symmetric kernel function $K_h(A - a)$, following standard practice in semiparametric estimation for continuous treatment~\citep{ichimura.2021, colangelo.2023}. This kernel smoothing localizes the pseudo-outcome around the target treatment value $a$ and ensures stable estimation.

Our solution to overcome this is the following. We approximate the Dirac delta $\delta(\cdot)$ with a smooth, symmetric kernel function $K_h(t)$. Thereby, we follow standard practice in semiparametric estimation for continuous treatments \citep{ichimura.2021, colangelo.2023}. This kernel smoothing makes the localization well-behaved and ensures stable estimation.

With the above modification, our partial identification framework, including the domain-knowledge upper bounds, extends naturally to continuous treatments, as shown in Eq.~(\ref{eq:upper_gamma_bound_continuous}). Formally, the pseudo-outcome under a continuous treatment setting is defined via
{\small{\begin{equation}\label{eq:upper_gamma_bound_continuous}
\begin{aligned}
    &\hat{\phi}^{+}(x, a)\\
    =& \frac{1(\Delta = 0)K_h(A-a)}{\hat{f}_{A\mid X}( a \mid x)}\{\tilde{T} - \hat{\nu}(0, x, a)\} 
    + \frac{\hat{\nu}(0, x, a)K_h(A-a)}{\hat{f}_{A\mid X}( a \mid x)} \{1(\Delta = 0)- [1-\hat{\xi}(x, a)]\} + \hat{\nu}(0, x, a)[1-\hat{\xi}(x, a)]\\
    &+ \frac{1(\Delta = 1)K_h(A-a)}{\hat{f}_{A\mid X}( a \mid x)}\{\tilde{T} - \hat{\nu}(1, x, a)\} 
    + \frac{\hat{\nu}(1, x, a)K_h(A-a)}{\hat{f}_{A\mid X}( a \mid x)} \{1(\Delta = 1)- \hat{\xi}(x, a)\} + \hat{\nu}(1, x, a)\hat{\xi}(x, a)\\
    &+ \gamma(x, a) \frac{K_h(A-a)}{\hat{f}_{A\mid X}( a \mid x)} \left\{ 1(\Delta = 1) - \hat{\xi}(x, a)\right\} + \gamma(x, a) \hat{\xi}(x, a),
\end{aligned}
\end{equation}}}
where $ \hat{f}_{A\mid X}( a \mid x) $ is an estimator of $f_{A\mid X}( a \mid x) $, which is the generalized propensity score and the $K_h(A-a) = k(\frac{A-a}{h})/h$ is the kernel with standard second-order kernel function $k(\cdot)$ and bandwidth $h$. As a result, our \method meta-learner for the continuous setting continues to enjoy the desirable properties established in the discrete setting -- most notably, consistency and double robustness -- which ensures that the bounds are reliably estimated even under model misspecification.

\newpage
\subsection{Hidden confounding}
\label{app:hidden_confounding}

In this section, we provide an extension for observational data (e.g., patient registries, electronic health records) that can be subject to hidden confounding. To this end, we consider data with hidden confounding $U$, where the full population is $(X, A, T, C, U) \sim \sP$. For simplicity, we consider the binary treatment setting in the following derivation, i.e., $A \in \{0, 1\}$. Given the setting, we are interested in the lower bound in Eq.~(\ref{eq:lower_bound}) and the upper bound in Eq.~(\ref{eq:upper_gamma_bound}) of the form
\begin{equation}
\begin{aligned}
\gQ_{\mathrm{LB}}(x, a, \gM) &= \mathbb{E}_{\mathbb{P}(\cdot \mid x, \mathrm{do}(A=a))}[\tilde{T}],\\
\gQ_{\mathrm{UB}}(x, a, \gM) &= \mathbb{E}_{\mathbb{P}(\cdot \mid x, \mathrm{do}(A=a))}[\tilde{T}] 
+ \gamma(x, a)\cdot\mathbb{P}(\Delta = 1 \mid x, \mathrm{do}(A = a)).
\end{aligned}
\end{equation}

\begin{definition}(SCM)
Let $\mathcal{V} = \{V_1, \ldots, V_n\}$ denote a set of observed endogenous variables, 
let $\mathcal{U} \sim P_{\mathcal{U}}$ denote a set of unobserved exogenous variables, 
and let $\mathcal{F} = \{f_{V_1}, \ldots, f_{V_n}\}$ with $f_{V_i} : \mathrm{Pa}(V_i) \subseteq \mathcal{V} \cup \mathcal{U} \to V_i$. 
The tuple $(\mathcal{V}, \mathcal{U}, \mathcal{F}, P_{\mathcal{U}})$ is called a structural causal model (SCM).
\end{definition}

We assume a directed graph $\mathcal{G}_{\mathcal{C}}$ induced by SCM $\mathcal{C}$ to be acyclic.

\begin{definition} (Generalized marginal sensitivity model (GMSM)).
Let $\mathbf{V} = \{\mathbf{X}, \mathbf{A}, \mathbf{T}, \mathbf{C}\}$. Let $\mathbf{X}$, $\mathbf{A}$, $\mathbf{T}$, and $\mathbf{C}$ denote a set of observed endogenous variables, $\mathbf{U}$ a set of unobserved exogenous variables, and $\mathcal{G}$ a causal directed acyclic graph (DAG) on $\mathbf{V} \bigcup \mathbf{U}$. For an observational distribution $P_{\mathbf{V}}$ on $\mathbf{V}$ and a family $\mathcal{P}$ of joint probability distribution on  $\mathbf{V} \bigcup \mathbf{U}$ that satisfy
\begin{equation}
\begin{aligned}
\frac{1}{(1-\Gamma) \mathbb{P}(a\mid x) + \Gamma} \leq \frac{P(U = u \mid x, a)}{P(U = u \mid x, \mathrm{do}(A= a))} \leq \frac{1}{(1-\Gamma^{-1}) \mathbb{P}(a\mid x) + \Gamma^{-1}}.
\end{aligned}
\end{equation}
\end{definition}

Using the definition above, we aim to obtain bounds $\gQ^{-}(x, a, \gS) \leq \gQ^{+}(x, a, \gS)$ so that
\begin{equation}
\begin{aligned}
\gQ_{\mathrm{LB}}^{-}(x, a, \gS) &= \inf_{\gM\in \gC(\gS)} \gQ_{\mathrm{LB}}(x, a, \gM) \text{ and } \gQ_{\mathrm{LB}}^{+}(x, a, \gS) = \sup_{\gM\in \gC(\gS)} \gQ_{\mathrm{LB}}(x, a, \gM)\\
\gQ_{\mathrm{UB}}^{-}(x, a, \gS) &= \inf_{\gM\in \gC(\gS)} \gQ_{\mathrm{UB}}(x, a, \gM) \text{ and } \gQ_{\mathrm{UB}}^{+}(x, a, \gS) = \sup_{\gM\in \gC(\gS)} \gQ_{\mathrm{UB}}(x, a, \gM).
\end{aligned}
\end{equation}

\begin{assumption}
We assume the conditional unconfoundedness, which is $(T(a), C(a)) \indep A \mid X, U$, $\forall a$.
\end{assumption}

\begin{theorem}
\label{thm:gmsm}
Let $\mathcal{S}$ be a GMSM with bounds $s_W^{-} = \frac{1}{(1-\Gamma) \mathbb{P}(a\mid x) + \Gamma}$ and $s_W^{+} = \frac{1}{(1-\Gamma^{-1}) \mathbb{P}(a\mid x) + \Gamma^{-1}}$, for $W \in \{T, C\}$. We define $c_W^{+} = \frac{(1-s_W^{-})s_W^{+}}{s_W^{+} - s_W^{-}} $. Since $W \in \{T, C\} \in \mathbb{R}$ is continuous, we define the probability density function
\begin{equation}
\begin{aligned}
\mathbb{P}_{+}(w\mid x, a) = 
\begin{cases}
\dfrac{1}{s^{+}_{W}}\;\mathbb{P}(w \mid x, a), & \text{if } F(w)\le c^{+}_{W},\\
\dfrac{1}{s^{-}_{W}}\;\mathbb{P}(w \mid x, a), & \text{if } F(w)> c^{+}_{W}.
\end{cases}
\end{aligned}
\end{equation}
where $F_{+}(\cdot)$ denote the conditional CDF corresponding to $\mathbb{P}_{+}(w\mid x, a)$. Further let $F_{-}(\cdot)$ be the conditional CDF corresponding to $\mathbb{P}_{-}(w\mid x, a)$, which is defined by swapping signs (in $c_W$ and $s_W$). For any SCM $\mathcal{M}$, we denote the CDF corresponding to $\mathbb{P}_{\mathcal{M}}(\cdot \mid x, \mathrm{do}(A = a))$ by $F_{\mathcal{M}}(\cdot)$. Then, for all $w$, we have
\begin{equation}
\label{eq:msm_bound_cdf}
\begin{aligned}
F_{+}(w) \leq \inf_{\gM\in \gC(\gS)} F_{\gM}(w) \text{ and } F_{-}(w) \geq \sup_{\gM\in \gC(\gS)} F_{\gM}(w).
\end{aligned}
\end{equation}
Assume now that $\mathbb{P}(u \mid x, \mathrm{do}(A = a)) = \mathbb{P}(u \mid x)$. Then, the bounds are sharp (i.e., equality holds in Eq.~(\ref{eq:msm_bound_cdf}) whenever $A$ is discrete and it holds that $1/s_W^+ \geq \mathbb{P}(a\mid x)$.
\end{theorem}

\begin{proof}
We adopt the proof of Theorem 1 from \citet{frauen.2023a}.
\end{proof}

We now leverage Theorem~\ref{thm:gmsm} to obtain explicit solutions of our lower bounds and upper bounds.
\begin{corollary}(Bounds with hidden confounding)
\label{cor:bounds_gmsm}
When the decision function $\mathcal{D}$ induced by the GMSM $\mathcal{S}$ 
is monotone, we obtain sharp bounds, i.e.,
\begin{equation}
\begin{aligned}
\gQ_{\mathrm{LB}}^{-} &= \mathbb{E}_{\mathbb{P}_{-}}[\tilde{T} \mid x, a],\\
\gQ_{\mathrm{UB}}^{+} &= \mathbb{E}_{\mathbb{P}_{+}}[\tilde{T} \mid x, a] 
+ \gamma(x, a)\cdot \mathbb{P}_{+}(\Delta = 1 \mid x, a).
\end{aligned}
\end{equation}
\end{corollary}

\begin{proof}
We adopt the proof of Corollary 1 from \citet{frauen.2023a}.
\end{proof}

\section{Details for the synthetic datasets}
\label{app:data_generation}

\textbf{Data-generating process:} For our synthetic dataset, we simulate 
an observed confounder $X \sim \mathrm{Uniform}[10, 100]$ to mimic the 
patients' age, and a latent frailty variable $U \sim \mathcal{N}(0, 1)$, 
which represents unobserved disease severity and is not available to the 
learner. We define the propensity score in the randomized control trial 
as $\pi_a(x) = \mathbb{P}(A = 1 \mid X =x) = 0.5$, and in observational 
dataset experiments as $\pi_a(x) = \sigma\bigl(\frac{x-45}{45}\bigr)$. 
Next, we specify the treatment effect function $\tau(x, a)$:
\begin{itemize}
\item Exponential function: $\tau(x, a) = 20 \cdot \exp(a + 0.01x) + \varepsilon$
\item Sin function: $\tau(x, a) = \bigl(\sin\bigl(\tfrac{x-10}{90}\cdot 2\pi\bigr)+1.2\bigr)\cdot 10\cdot a + x + \varepsilon$
\item Logistic-sin function: $\tau(x, a) = \frac{30}{1+\exp(-0.1(x-50))} + 5\sin(0.2x) + 10 + \varepsilon$
\end{itemize}
where $\varepsilon \sim \mathcal{N}(0, 0.1)$. The survival time is 
\begin{equation}
T = \tau(x,a)\cdot a + \tfrac{1}{3}\bigl(\sin(12x)+x\bigr) 
+ \tfrac{1}{60}\cos(20x) - \beta_T \cdot U + \varepsilon,
\end{equation}
where $\beta_T = 1.0$, so that sicker patients (larger $U$) have shorter 
survival times. The censoring time is generated as
\begin{equation}
C = c_0 \cdot \exp\bigl(-\beta_C \cdot U + \eta\bigr), 
\quad \eta \sim \mathcal{N}(0, 0.1),
\end{equation}
where $\beta_C = 1.0$ and $c_0$ is chosen to achieve the target censoring rate $\xi \in \{0.2, 0.4, 0.6\}$. The observed outcome and censoring indicator are then
\begin{equation}
\tilde{T} = \min\{T, C\}, \qquad \Delta = \mathbbm{1}(C \leq T).
\end{equation}
By construction, $U$ simultaneously shortens $T$ and $C$, inducing 
$C \not\perp T \mid X, A$, which constitutes genuine informative 
censoring. We generate $2000$ samples for each setting and random seed.

\newpage
\section{Additional experimental results}

\subsection{Additional experiments with point estimators}
\label{app:drcut_comparison}

We implement the method from \citet{rubin.2007}, which extends the doubly robust censoring unbiased transformation to CATE estimation via a second-stage regression on the corresponding pseudo-outcomes. This baseline shares the same target estimand as our framework but assumes non-informative censoring, i.e., $C \perp T \mid X, A$, which is violated in our setting. We therefore use it as a point-estimation baseline to empirically demonstrate the bias induced by ignoring informative censoring. 

\begin{figure}[htbp]
    \includegraphics[width=\linewidth]{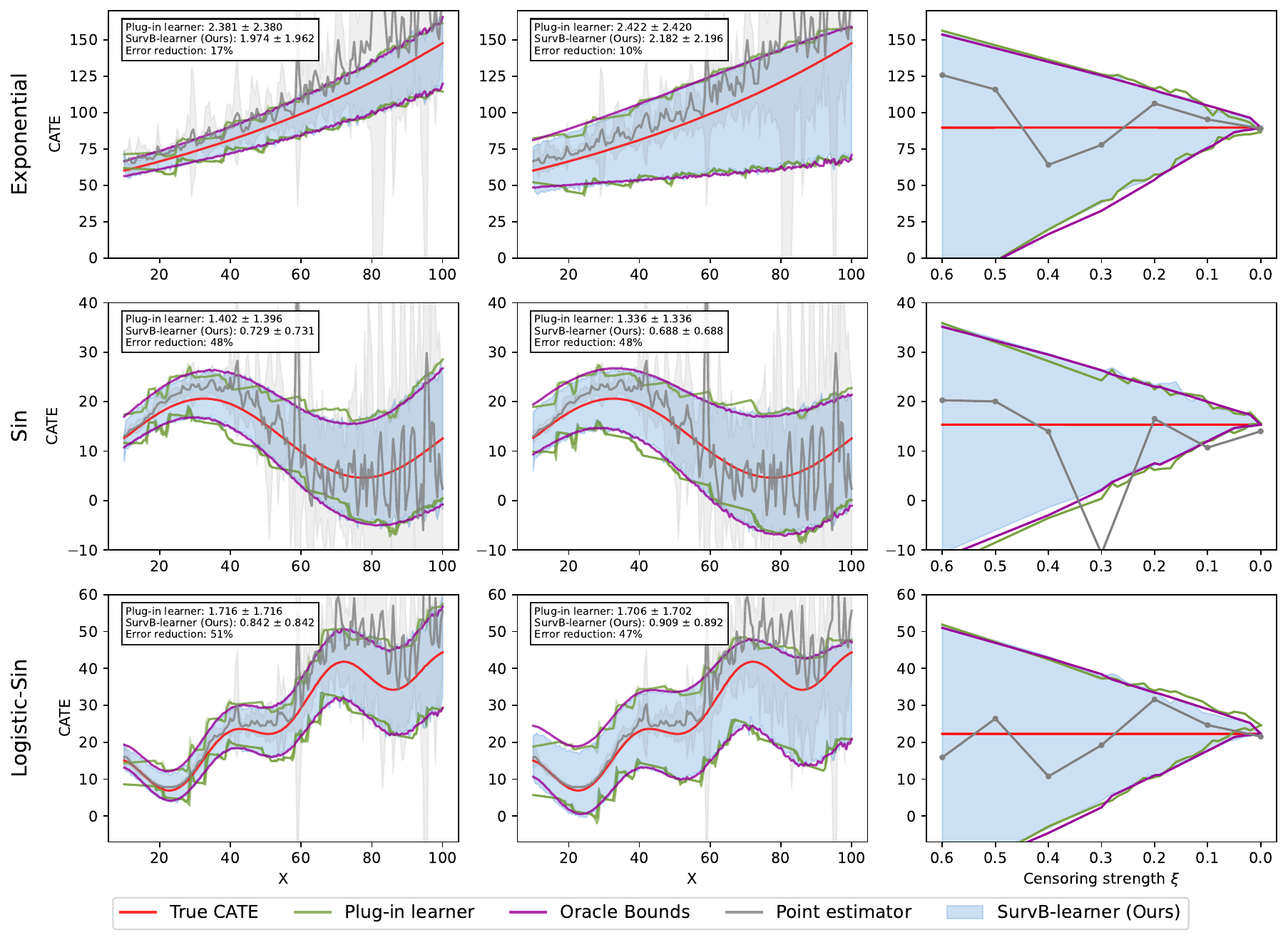}
    \caption{\textbf{Results for experiments with synthetic data including the point estimation baseline}.}
    \label{fig:drcut}
\end{figure}

Results are shown in Fig.~\ref{fig:drcut}. The comparison exposes a fundamental limitation of point estimation under informative censoring. As $\xi(x,a)$ grows (decided by the DGP in appendix~\ref{app:data_generation}), the point estimator increasingly deviates from the true CATE and eventually falls entirely outside our bounds, confirming that ignoring informative censoring produces invalid estimates in high-censoring regimes.

When $\xi(x,a)$ is small, the point estimator performs comparably to our \method, consistent with our theoretical result that the width of CAPO bounds $\gamma(x,a)\xi(x,a) \to 0$ as $\xi \to 0$, and our bounds recover the point estimation. And in this regime, even conservative bounds without domain knowledge are informative. Together, these observations confirm that our \method is strictly more general: it degrades to point estimation when censoring is weak, while remaining valid when censoring is strong and point estimation fails.
\subsection{Empirical coverage and smoothness of estimated bounds}
\label{app:coverage_tv}

We further evaluate our \method against the plug-in learner using empirical coverage and total variation (TV). Coverage measures the proportion of evaluation points where the estimated bounds contain the true CATE:
\begin{equation}
\widehat{\mathrm{Coverage}}
=
\frac{1}{RM}
\sum_{r=1}^{R}
\sum_{j=1}^{M}
\mathbbm{1}
\left\{
\hat{\tau}^{-}_{r}(x_j)
\leq \tau_r(x_j)
\leq \hat{\tau}^{+}_{r}(x_j)
\right\},
\end{equation}
where $R$ is the number of simulation seeds and $M$ is the number of evaluation points.

\begin{table}[htbp]
\centering
\small
\begin{tabular}{llcc}
\toprule
Bounds & Function & SurvB-learner & Plug-in learner \\
\midrule
\multirow{3}{*}{Domain knowledge}
& exponential  & \textbf{100.00\%} & 97.45\% \\
& sin          & \textbf{100.00\%} & 99.75\% \\
& logistic-sin & \textbf{100.00\%} & 97.15\% \\
\midrule
\multirow{3}{*}{Conservative}
& exponential  & \textbf{100.00\%} & 100.00\% \\
& sin          & \textbf{100.00\%} & 99.95\% \\
& logistic-sin & \textbf{100.00\%} & 99.95\% \\
\bottomrule
\end{tabular}
\vspace{0.1cm}
\caption{\textbf{Empirical coverage of estimated CATE bounds.} Higher is better.}
\label{tab:coverage_results}
\end{table}

Table~\ref{tab:coverage_results} shows that our \method achieves 100\% empirical coverage across all settings, whereas the plug-in learner fails to achieve complete coverage in several cases.

We further measure the smoothness of estimated bounds using total variation:
\begin{equation}
\mathrm{TV}(f)
=
\sum_{j=1}^{M-1}
|f(x_{j+1})-f(x_j)|,
\end{equation}
where $\{x_1,\ldots,x_M\}$ is a fixed ordered evaluation grid. For each simulation seed, we sum the TV of the lower and upper bound functions and report the mean and standard deviation across seeds. Values closer to the oracle TV indicate better recovery of the underlying variation.

\begin{table}[htbp]
\centering
\small
\begin{tabular}{llccc}
\toprule
Bounds & Function & Oracle TV & SurvB-learner TV & Plug-in TV \\
\midrule
\multirow{3}{*}{Domain knowledge}
& exponential
& $78.2 \pm 0.5$
& $\mathbf{84.0 \pm 2.7}$
& $89.4 \pm 4.0$ \\
& sin
& $31.2 \pm 0.1$
& $\mathbf{33.1 \pm 0.7}$
& $39.2 \pm 1.6$ \\
& logistic-sin
& $60.2 \pm 0.4$
& $\mathbf{61.0 \pm 1.8}$
& $63.8 \pm 2.1$ \\
\midrule
\multirow{3}{*}{Conservative}
& exponential
& $49.4 \pm 0.6$
& $\mathbf{65.1 \pm 2.6}$
& $80.8 \pm 5.5$ \\
& sin
& $26.7 \pm 0.1$
& $\mathbf{28.1 \pm 0.7}$
& $37.6 \pm 2.5$ \\
& logistic-sin
& $50.7 \pm 0.2$
& $\mathbf{52.7 \pm 1.2}$
& $57.1 \pm 1.6$ \\
\bottomrule
\end{tabular}
\vspace{0.1cm}
\caption{\textbf{Total variation of estimated CATE bounds.} Mean $\pm$ standard deviation across simulation seeds. Values closer to the oracle TV are preferred.}
\label{tab:tv_results}
\end{table}

As shown in Table~\ref{tab:tv_results}, our \method consistently achieves TV closer to the oracle bounds and lower variability across simulation seeds than the plug-in learner. Together, these results demonstrate better empirical coverage and more stable bound estimation, consistent with Fig.~\ref{fig:syn_exp_results_fig}.

\subsection{Sensitivity analysis on the specification of the post-dropout survival function}
\label{app:sensitivity_gamma}

We investigate the sensitivity of our bounds to the specification of $\gamma(x,a)$. In our synthetic DGPs, the true expected post-dropout survival time is available as
$\gamma^{\star}(x,a)=\E[T-\tilde{T}\mid X=x,A=a,\Delta=1]$.
We set $\gamma_{\rho}(x,a)=\rho\gamma^{\star}(x,a)$ for both treatment arms and vary $\rho$, including values that underestimate ($\rho<1$) or overestimate ($\rho>1$) the true post-dropout survival. We additionally consider the conservative specification. Oracle nuisance functions are used to isolate the effect of misspecifying $\gamma(x,a)$.

\begin{table}[htbp]
\centering
\small
\resizebox{\linewidth}{!}{%
\begin{tabular}{lccccccc}
\toprule
Metric
& $\rho=0.5$
& $\rho=0.8$
& $\rho=1.0$
& $\rho=1.2$
& $\rho=1.5$
& $\rho=1.7$
& Conservative \\
\midrule
Coverage
& 0.5874
& 0.8252
& 1.0000
& 1.0000
& 1.0000
& 1.0000
& 1.0000 \\
Average width
& 7.700
& 12.320
& 15.400
& 18.480
& 23.100
& 26.180
& 35.344 \\
\bottomrule
\end{tabular}
}
\vspace{0.1cm}
\caption{\textbf{Sensitivity analysis of $\gamma(x,a)$ on the Sin DGP.} Coverage and average CATE bound width for different scaling factors $\rho$. Values $\rho<1$ violate Eq.~(\ref{eq:range_of_gamma}).}
\label{tab:gamma_sensitivity}
\end{table}

Table~\ref{tab:gamma_sensitivity} shows that underestimating $\gamma(x,a)$ produces narrower bounds but reduces coverage, while $\rho\geq1$ achieves full coverage in our experiments at a linear cost in bound width. This follows from Proposition~\ref{prop:width_of_bounds}, where the CATE bound width is the sum of $\gamma(x,a)\xi(x,a)$ across both arms. Hence, $\gamma(x,a)$ should be chosen as a plausible upper bound rather than a best estimate of post-dropout survival. When domain knowledge is unavailable, the conservative specification provides a valid alternative.

\subsection{Informativeness of conservative bounds under varying censoring strengths}
\label{app:conservative_width}

We examine whether conservative bounds remain informative by stratifying evaluation points in the Sin DGP according to the arm-average censoring strength
$\bar{\xi}(x)=[\xi(x,a_1)+\xi(x,a_2)]/2$.
We report the average conservative bound width and empirical coverage within each stratum.

\begin{table}[htbp]
\centering
\small
\begin{tabular}{lcccc}
\toprule
Metric
& $\xi < 0.05$
& $0.05 \le \xi < 0.2$
& $0.2 \le \xi < 0.3$
& $\xi > 0.3$\\
\midrule
Conservative width
& 7.156
& 27.141
& 44.428
& 52.233 \\
Conservative coverage
& 1.0000
& 1.0000
& 1.0000
& 1.0000 \\
\bottomrule
\end{tabular}
\vspace{0.1cm}
\caption{Conservative CATE bound width and coverage across censoring-strength strata.}
\label{tab:conservative_censoring_width}
\end{table}

Table~\ref{tab:conservative_censoring_width} shows that conservative bounds maintain full coverage while their width increases with censoring strength. This follows from
\begin{equation}
\begin{aligned}
&\tau_{a_1,a_2}^{+}(x)-\tau_{a_1,a_2}^{-}(x) =\sum_{a\in\{a_1,a_2\}}
[t_{\mathrm{max}}-\nu(1,x,a)]\xi(x,a).
\end{aligned}
\end{equation}
As censoring vanishes in both arms, the bound width approaches zero. Thus, conservative bounds remain relatively narrow in low-censoring subgroups, while wider bounds reflect greater identification uncertainty in heavily censored subgroups.

\subsection{Comparison with existing causal survival estimators under informative censoring}
\label{app:comparison_existing_methods}

We compare our \method with three point-estimation methods from SurvHTE-Bench: the doubly robust (DR) point estimator, integrated causal survival forest (Integrated-CSF), and hazard-EIF meta-learner. Unlike our \method, which estimates CATE bounds under informative censoring, these baselines target point estimates under non-informative censoring. We evaluate all methods on three synthetic DGPs with informative censoring and report ATE estimates with 95\% intervals for the baselines and domain-knowledge ATE bounds for our \method and the plug-in learner.

\begin{table}[htbp]
\centering
\small
\resizebox{\linewidth}{!}{%
\begin{tabular}{lccc}
\toprule
Method
& sin
& logistic-sin
& exponential \\
\midrule
True ATE
& $12.60$
& $24.10$
& $97.52$ \\
\midrule
DR point estimator
& $11.64 \pm 0.09$
& $39.50 \pm 6.12$
& $101.18 \pm 3.12$ \\
Integrated-CSF
& $12.29 \pm 0.02$
& $27.38 \pm 0.04$
& $99.20 \pm 0.08$ \\
Hazard-EIF meta-learner
& $12.20 \pm 0.02$
& $27.40 \pm 0.50$
& $99.43 \pm 0.34$ \\
\midrule
\method (LB, UB)
& $[5.94 \pm 0.36,\;21.03 \pm 0.38]$
& $[18.97 \pm 0.37,\;34.07 \pm 0.39]$
& $[83.20 \pm 1.50,\;108.35 \pm 1.66]$ \\
Plug-in learner (LB, UB)
& $[6.00 \pm 0.33,\;21.09 \pm 0.32]$
& $[19.06 \pm 0.36,\;34.15 \pm 0.30]$
& $[83.49 \pm 1.42,\;108.64 \pm 1.38]$ \\
\bottomrule
\end{tabular}
}
\vspace{0.1cm}
\caption{\textbf{Comparison with existing causal survival estimators under informative censoring.} Point estimates are reported with 95\% intervals, while bounds are reported with uncertainty margins for each endpoint.}
\label{tab:comparison_survival_baselines}
\end{table}
\vspace{-0.2cm}

Table~\ref{tab:comparison_survival_baselines} shows that the reported 95\% intervals of all three point estimation baselines fail to contain the true ATE across the DGPs, whereas both our \method and the plug-in learner produce bounds containing the truth. This highlights the distinction between sampling uncertainty and identification uncertainty: narrow confidence intervals around point estimates do not account for potential bias from informative censoring, whereas our bounds explicitly account for this uncertainty through $\gamma(x,a)$.

\newpage
\section{Implementation details}
\label{app:implementation_detail}
\subsection{Experiments with synthetic data}

For the synthetic experiments, we build on the implementation framework commonly used to evaluate meta-learners for point-identified CATE estimation~\citep{frauen.2025}. We employ the random forests~\citep{breiman.2001} for nuisance estimation in both the plug-in learner and our \method. This ensures that performance differences arise solely from the meta-learners rather than from model choice. All learners were implemented in Python 3.11.12 using scikit-learn 1.5.2. 

Random forests were fitted with the default hyperparameter settings and without additional tuning (i.e, $\text{n\_estimators} = 100$, $\text{max\_depth} = \text{None}$, and $\text{min\_samples\_leaf} = \text{2}$). For the implementation of the plug-in learner, we also use a random forest with default settings to estimate the nuisance functions, namely, $\nu(\delta, x, a)$ and $\xi(x, a)$. Then, we follow the Eq.~(\ref{eq:lower_bound}), Eq.~(\ref{eq:upper_gamma_bound}), and Eq.~(\ref{eq:upper_non_informative_bound}) in Section~\ref{sec:propose_bounds} directly to compute the bounds. For the implementation of our \method, we apply three-fold cross-fitting, and also apply the random forests using default settings in the second stage regression.

For Domain knowledge bounds, we set $\gamma(X, a)$ is uniformaly over covariates and treatments, which make sense, since for synthetic data we are comapre the accuracy of SurvB-learner, instead of complexity. For exponential function, $\gamma = 50$; sin function, $\gamma = 30$, logistic-sin function, $\gamma = 30$.

\textbf{Regarding the RMSE metric:} We define the RMSE in Table~\ref{tab:syn_exp_results} as the sum of the upper and lower 
bounds against their respective ground-truth values $\mu^{+}$ and $\mu^{-}$:
\begin{equation}
    \text{RMSE} = \sqrt{\frac{1}{N}\sum_{i=1}^{N}\left(\hat{\mu}^{\pm}_i 
    - \mu^{\pm}_i\right)^2}.
\end{equation}
That is, the reported RMSE measures how accurately the estimated bounds recover the true bounds, not the deviation from a single point estimate. 

\subsection{ADJUVANT dataset}

The ADJUVANT trial was designed to enroll patients with EGFR-mutant stage II--IIIA non-small cell lung cancer (NSCLC). However, in the available dataset, one patient was recorded as clinical stage I, and we therefore excluded this patient from our analysis.

Given the limited sample size of the ADJUVANT cohort ($N=170$), we adopt three-fold cross-fitting. In the first stage, we estimate nuisance functions using random forests, which are nonparametric ensembles of decision trees~\citep{breiman.2001}. In the second stage, we regress pseudo-outcomes on covariates and treatment with random forest regressors using the default settings to obtain the bound estimates. Following the ADJUVANT trial~\citep{zhong.2018, liu.2021}, we set $t_\mathrm{max} = 50$ months for DFS and OS time.

For the data-driven analysis in Figure~\ref{fig:results_adjuvant_1}~\textbf{(C)}, we adapt the idea of regression tree partitioning~\citep{zhang.2010}, where we replace the standard squared-error splitting rule by a customized criterion that maximizes the percentage of lower bounds (LBs) above zero within each leaf state in Eq.~\ref{eq:tree_criterion}. Hence, we have
\begin{equation}
\begin{aligned}
\label{eq:tree_criterion}
&p_L = \frac{1}{|M_L|} \sum_{i \in M_L} \mathbf{1}\{ \hat{\tau}^{-}_{\text{gefitinib, chemo}}(x_i) > 0 \}, \\
\qquad 
&p_R = \frac{1}{|M_R|} \sum_{i \in M_R} \mathbf{1}\{ \hat{\tau}^{-}_{\text{gefitinib, chemo}}(x_i) > 0 \},\\
&\text{Gain}(M_L, M_R) = \max \{ p_L, \; p_R \},
\end{aligned}
\end{equation}
where $M_L$, $M_R$ are the set of samples that fall into the left/right leaf node under a candidate split. We set the maximum tree depth to 2 and require at least 2 samples per leaf node.

\newpage
\section{Real-world dataset}
\label{app:adjuvant_results}
\begin{figure}[htbp]
\centering
    \includegraphics[width=1\linewidth]{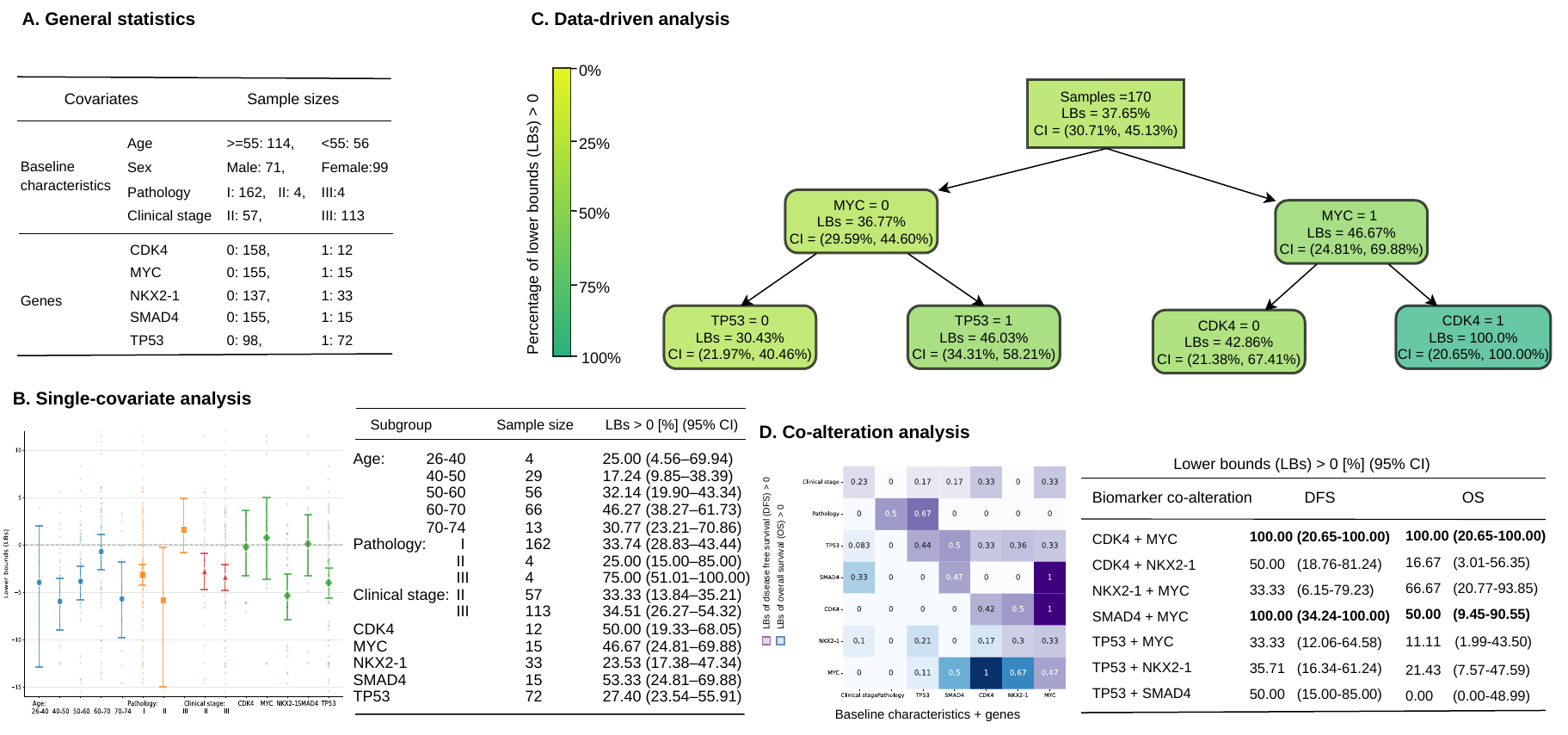}
    \caption{
    \footnotesize\textbf{Findings for the ADJUVANT trial.} The ADJUVANT trial~\citep{zhong.2018, liu.2021} assesses the effect of adjuvant gefitinib in $N=170$ patients with EGFR-mutant stage II--IIIA non-small cell lung cancer (NSCLC), yet the original study was subject to large censoring (37.06\% for DFS and 63.53\% for OS). Hence, our aim is to identify subgroups of patients that have a robust treatment benefit despite potential informative censoring. To do so, we report the proportion of patients where the lower bounds (LBs) of the CATE are strictly greater than zero (i.e., ``LBs $\ge$ 0''), expressed as percentages (\%). So, a value close to 100\% indicates that nearly all patients in this subgroup benefit from treatment, with a CATE above zero despite potential informative censoring.
    \textbf{(A)}~Patient characteristics: The table on the left reports the baseline characteristics and genetic alterations of the patient population.
    \textbf{(B)}~Single-covariate analysis: The forest plots show the mean and standard deviation of the estimated lower bounds of the CATE within subgroups defined by baseline characteristics or genetic covariates by bootstrap. For each subgroup on the x-axis, we performed $2000$ bootstrap replications. The background scatter points represent the CATE lower bound estimates without bootstrapping. 
    \textbf{(C)}~Data-driven analysis: Using a tailored recursive partition algorithm, we select subgroups with the highest percentage of LBs above zero. Method details are in Appendix~\ref{app:implementation_detail}. 
    \textbf{(D)}~Co-alteration analysis: The heatmap (left) reports the probabilities that the conservative lower bounds are above zero (``LBs $\ge$ 0'') for pairs of genetic biomarkers, while the table (right, sorted alphabetically) summarizes these probabilities for both disease-free survival (DFS) and overall survival (OS) for selected genetic co-alterations. Hence, our framework provides robust evidence of treatment benefit for certain subgroups, even for the efficacy outcome (OS), despite the original study not being able to confirm this, and in the presence of informative censoring that could otherwise bias the results. 
    }
    \label{fig:results_adjuvant_1}
\end{figure}

\end{document}